\documentclass{article} 
\usepackage{iclr2027_conference, times}

\makeatletter
\def\@notice{}

\def\ps@headings{%
  \def\@oddhead{}%
  \def\@evenhead{}%
  \def\@oddfoot{\hfil\thepage\hfil}
  \def\@evenfoot{\hfil\thepage\hfil}%
}
\makeatother

\usepackage{amsmath,amsfonts,bm}

\def\eqref#1{equation~\ref{#1}}

\def\1{\bm{1}}

\DeclareMathAlphabet{\mathsfit}{\encodingdefault}{\sfdefault}{m}{sl}
\SetMathAlphabet{\mathsfit}{bold}{\encodingdefault}{\sfdefault}{bx}{n}

\usepackage{url}
\usepackage[table]{xcolor}
\usepackage{dsfont}
\usepackage{mathtools}
\usepackage{amsmath}
\usepackage{amsthm}
\usepackage{amsthm}
\usepackage{centernot}
\usepackage{lipsum}
\usepackage{amssymb}
\usepackage{booktabs}
\usepackage{multirow}
\usepackage{tabularx}
\usepackage{graphicx}
\usepackage{xcolor}
\theoremstyle{plain}

\newtheorem{prop}{Proposition}
\newtheorem{cor}{Corollary}

\title{Is invariance all you need for algorithmic fairness? Removing Demographic Information Can Create New Bias}

\author{
\textbf{Aditya Parikh}$^{1,\dagger}$,\quad \textbf{Eike Petersen}$^{2,\dagger}$,\quad \textbf{Stella Frank}$^{1}$, \quad
\textbf{Enzo Ferrante}$^{3}$,\\
\textbf{Melanie Ganz}$^{4}$,\quad \textbf{Aasa Feragen}$^{1}$ \\
\\
{\small $^{1}$Technical University of Denmark, Kgs.~Lyngby, Denmark} \\
{\small $^{2}$Fraunhofer Institute for Digital Medicine MEVIS, Bremen, Germany} \\
{\small $^{3}$CONICET -- Universidad de Buenos Aires, ICC, Buenos Aires, Argentina} \\
{\small $^{4}$University of Copenhagen, Copenhagen, Denmark} \\
\\
{\small $^{\dagger}$Equal contribution. \qquad Corresponding author: \texttt{adipa@dtu.dk}}
}

\iclrfinalcopy 
\begin{document}

\maketitle

\begin{abstract}

Encoded demographic information in internal model representations is a commonly assumed risk factor for algorithmic bias, with demographic representation invariance often being touted as the ideal state. However, while demographic shortcut learning is a genuine threat, some degree of encoding is necessary when demographics correlate with target labels. Here, we show, mathematically and empirically, that enforcing demographic invariance can actually \emph{hamper} bias mitigation and even create new biases. We distinguish marginal from class-conditional representation invariance, and show that they imply the standard group fairness notions of demographic parity and equalized odds, respectively. We evaluate the effects on predictive performance and fairness of enforcing both invariance types, both theoretically and empirically across five tabular and two chest X-ray imaging datasets. Our findings support our mathematical argument that demographic representation invariance is neither \textit{desirable} nor \textit{sufficient} for fairness. \footnote{An earlier version is available as a non-peer reviewed preprint~\citep{petersen2023demographically}. This submission supersedes it with consent of all authors; all experiments are new.}

Full code, configs, and random seeds are made available anonymously.\footnote{Available at: https://anonymous.4open.science/r/invariance-project-931C/README.md}

\end{abstract}

\section{Introduction}
\label{sec:intro}

\textit{Demographic information} can be recovered from latent representations produced by medical imaging models, even when these models are neither explicitly trained to predict demographic attributes, nor trained using them~\citep{glocker2023algorithmic,yang2024limits,duffy2022confounders,gichoya2022ai,yi2021radiology,lotter2024acquisition,quarta2026hidden}. This raises the valid concern that a patient's group membership, once encoded in the representation, may be exploited by the model, e.g.~as shortcuts used instead of task-relevant features~\citep{stanley2025and, lin2024shortcut,brown2023detecting,yang2024limits}, or to reproduce historical biases encoded in systematically biased labels~\citep{akintande2025medicine,parikh2026towards}.

In this paper, we argue that this intuitive concern is a misunderstanding. We show, mathematically and empirically, that \textit{demographic decodability}: the ability to recover group membership from representation, by itself, does not explain unfairness or performance disparities between groups~\citep{jones2023role,stanley2025exploring,jones2025rethinking}. It is entirely possible for a representation to only encode task-specific features, yet still have high demographic decodability, simply because the task-specific features correlate with group membership, e.g., when disease symptoms or prevalence differ between groups. Removing all demographic information therefore risks removing task-specific features associated with group membership and thus harming performance. In the context of different true prevalences, enforcing demographic invariance results in systematic differences in error rates between the groups, harming equalized odds fairness.

These mathematical arguments stand in contrast to previous studies reporting positive empirical effects of applying alignment methods such as adversarial learning~\citep{edwards2015censoring,zhang2018mitigating} or distribution matching~\citep{louizos2015variational} to minimize internal encoding of demographic information~\citep{zemel2013learning,creager2019flexibly,zhao2022inherent,wang2024advancing,rashid2026stride,sadri2025mutual,sarridis2024flac}. The theoretical guarantees and fundamental limitations of these methods have, however, received little attention~\citep{zhao2022fundamental}. Our experiments support~\citet{cerrato202410yearsfairrepresentations}~showing that removing demographic information is surprisingly hard, and that most well-known methods fail at accurately removing demographic information from internal representations. This suggests that the regularizing effect of demographic decoding found in previous work is likely caused by failure to actually remove demographic information as intended.

We show that marginal representation invariance (i.e., the representation having the same distribution in every demographic group), which is the intuitive and common target of such studies, implies \emph{demographic parity}, also known as the independence fairness criterion. This is important, as demographic parity is well known to be incompatible with other fairness criteria often deemed more appropriate in real-life applications~\citep{barocas2023fairness}. We distinguish marginal representation invariance, in practice the most commonly considered form of invariance, from \emph{class-conditional} representation invariance (i.e., the representation is required to have the same distribution only among individuals with the same label), which we show implies the well-known \emph{equalized odds}, or \emph{separation}, fairness criterion~\citep{barocas2023fairness}.

Our experiments compare a range of different methods for enforcing representation invariance, and we use the one that actually \emph{does} work. We clearly demonstrate how marginal invariance leads to failure in satisfying equalized odds fairness. Our experiments \emph{also} show that while class-conditional invariance enforces equalized odds well, it also comes with the known risk of leveling performance down~\citep{mittelstadt2023unfairness} for the best groups in order to obtain fairness.


\begin{figure}[t]
  \centering
  \includegraphics[width=0.8\linewidth]{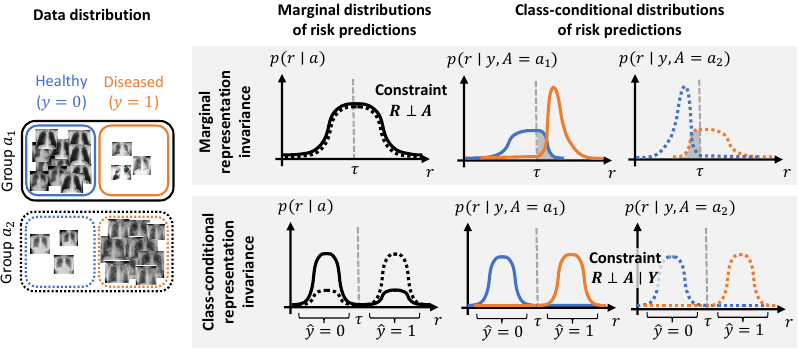}
  \caption{\textbf{Two forms of representation invariance and the effects of enforcing those} in a case study predicting disease for two groups ($a_1$ and $a_2$) with different prevalence. \textbf{Top:} $Z\perp A$ gives identical marginal risk distributions across groups. Under unequal prevalence, the class-conditional risk distributions must differ, giving different error rates across groups. \textbf{Bottom:} $Z\perp A\mid Y$ gives equal error rates, but still has drawbacks: it is incompatible with calibration by groups and requires equalizing away potentially important differences between unknown disease subtypes.} 
  \label{fig:schematic}
\end{figure}

\section{Problem Setting}
\label{sec:problem_settings}

We denote random variables by uppercase letters $(X, Y, Z)$ and their realizations by lowercase letters $(x, y,z)$. Consider a supervised prediction task with input $X\in\mathcal{X}$, outcome $R \in \mathcal{R}$ sometimes thresholded to a categorical $Y\in\mathcal{Y}$, and sensitive attribute $A\in\mathcal{A}=\{a_1,\ldots,a_m\}$ with $P(A=a)>0$ for all $a$. The input $X$ can be tabular features, an image, or any other data type. The outcome can be binary ($Y \in \{0, 1\}$), multiclass ($Y \in \{0, 1 \dots, K-1\}$), a regressed score $R \in \mathcal{R}$, or even a more complex output such as a segmentation mask. The attributes $A$ identify groups such as sex, ethnicity, or age group, and may be multi-valued or intersectional.

A deterministic feature extractor $g\colon\mathcal{X}\to\mathcal{Z}$ produces a representation $Z=g(X)$, which is mapped to predictions $R \in \mathcal{R}$ by a second deterministic component (e.g., a classifier head) $h \colon \mathcal{Z}\to\mathcal{R}$, producing a risk/score vector $R = h(Z)$. If required, $R$ may be discretized to obtain binary or categorical outcome labels $\hat{Y}$. Neither $g$, $h$, nor $\tau$ take $A$ as input.


We consider a model to \textit{encode}\footnote{We use the term \textit{encoding} for consistency with the literature. Above-chance prediction of $A$ from $Z$ does not necessarily mean the model explicitly represents demographic identity or uses it as a shortcut.} group membership $A$ in its latent representation $Z$ if group membership is \textit{decodable} from $Z$ above chance (i.e.\ $Z\not\perp A$). We use \textit{decodability} as a measurement (or a proxy) of invariance throughout. Notice that decodability depends on the evaluation dataset $\mathcal{D}$ and is not a property of the model alone.

\section{What does Representation Invariance Guarantee?}
\label{sec:invariance}

While the propositions in this section are not new individually, the way we connect them with two forms of representation invariance (illustrated in Fig.~\ref{fig:schematic}) and derive their consequences for fairness and performance is new. This connection is important to avoid common misunderstandings that may increase bias and decrease performance. We cite prior work where proofs are available. Sec.~\ref{sec:experiments} then introduces an experimental setup that allows us to empirically test these propositions in Sec.~\ref{sec:results}.












\subsection{Marginal Representation Invariance}
\label{sec:marginal_invariance}

Marginal representation invariance asks that the distribution of $Z$ be identical across groups:
\begin{equation}
\label{eq:mrg_inv}
    p(z \mid a) = p(z) \quad \forall a \qquad\Longleftrightarrow\qquad Z \perp A
\end{equation}
Under this condition, the learned representation $Z$ is statistically independent of $A$, denoted by $Z \perp A$: the model does not \textit{encode} group membership on this dataset.

\begin{prop} \label{prop:marginv}
Marginal representation invariance implies demographic parity/independence.
\end{prop}

\begin{proof}
If the risk prediction $R$ and the decision $\hat Y=\mathds{1}[R\ge\tau]$ depend only on \(Z\), then both predictions inherit this independence of $A$ from $Z$:
\begin{equation}
\label{eq:marginal_invariance_dp}
    Z\perp A
    \;\Longleftrightarrow\;
    p(z\mid a)=p(z) \ \forall a
    \;\Longrightarrow\;
    p(r\mid a)=p(r),
    \quad
    p(\hat y\mid a)=p(\hat y),
    \quad \forall a.
\end{equation}
That is, if $Z \perp A$, then we get \({R}\perp A\) and \(\hat{Y}\perp A\): the distributions of predicted risk scores and predicted decisions are identical across sensitive groups, which implies the fairness criterion known as \textit{demographic parity}, or \textit{independence}. 
\end{proof}

For a binary attribute $A \in \{0,1 \}$ this gives $P(\hat Y=1\mid A=0)=P(\hat Y=1\mid A=1)$, whose violation is measured by the demographic parity gap
\begin{equation}
\label{eq:dp_gap}
    \Delta_{\mathrm{DP}}
    =
    \left|
        P(\hat Y=1\mid A=0)
        -
        P(\hat Y=1\mid A=1)
    \right|.
\end{equation}
Therefore, exact marginal representation invariance implies $\Delta_{\mathrm{DP}} = 0$ if the classification head and the decision threshold do not use $A$~\citep{zemel2013learning, zhao2022fundamental,madras2018learningadversariallyfairtransferable}. Its finite-sample estimate may differ from zero $(\Delta_{\mathrm{DP}} \approx 0)$ due to sampling variability.

\subsubsection*{Unequal prevalence across groups causes problems for marginal invariance} Outcomes $Y$ often vary across demographic groups: Income differs across age groups or sex; crime rates differ across populations due to historical and institutional factors; disease prevalence depends on e.g.~sex and age, with breast cancer, which affects some but very few men, as an extreme example. 

\begin{prop}
Under unequal group prevalences, marginal invariance creates an upper bound on predictive performance.
\end{prop}

\begin{proof}
While marginal invariance requires 
    $p(\hat y\mid A=0)=p(\hat y\mid A=1)$, an accurate classifier should approximately reproduce the group-specific outcome distributions 
    $p(\hat y\mid a)\approx p(y\mid a)$. However,  when 
    $p(Y\mid A=0)\neq p(Y\mid A=1)$,
these requirements, i.e., demographic parity and perfect (desired) prediction, cannot both hold. Marginal representation invariance must thus introduce prediction errors in at least one group.
\end{proof}

For a binary outcome, the group-conditional representation distribution is
\begin{equation}
\label{eq:marginal_mixture}
\begin{split}
    p(z\mid a)
    &=
    \sum_{y\in\{0,1\}}
    p(z\mid y,a)\,p(y\mid a)\\
    &=
    \pi_a\,p(z\mid Y=1,A=a)
    +(1-\pi_a)\,p(z\mid Y=0,A=a),
\end{split}
\end{equation}
where
$\pi_a=P(Y=1\mid A=a)$
denotes the positive-label prevalence in group \(a\). If the within-class distributions are shared across groups, $p(z\mid y,a)=p(z\mid y)$, and $Z\not\perp Y$: the difference in group-wise $p(z \mid a)$ originates in $\pi_a$.

For a binary prediction task, let $e_a=P(\hat Y\neq Y\mid A=a)$ denote the prediction error for group $a$. Any predictor satisfying demographic parity obeys
$e_0+e_1\geq|\pi_0-\pi_1|$~\citep{mcnamara2019costs,zhao2022fundamental}. For multi-class tasks, an iterative division into binary one-versus-rest prediction problems can be used to derive a similar bound. The bound constrains the combined error across groups; the entire cost may fall on one group.




\subsection{Class-Conditional Representation Invariance}
\label{sec:class_conditional_invariance}

\begin{figure}[t]
  \centering
  \includegraphics[width=0.9\linewidth, trim=17 5pt 17 5pt, clip]{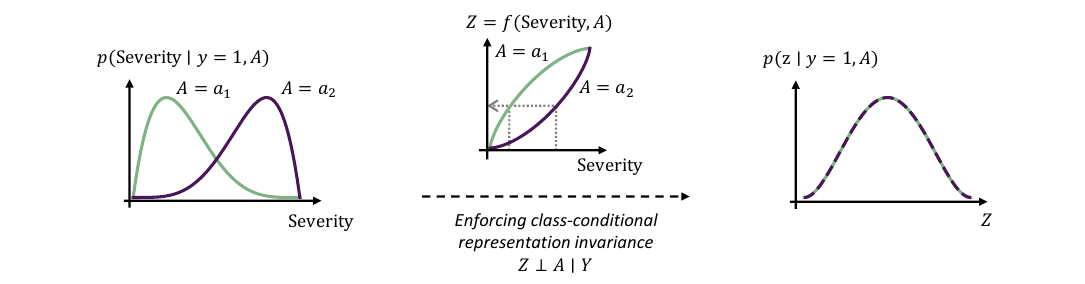}
  \caption{\textbf{Effects of enforcing Class-conditional representation invariance in the presence of intra-class variation.} When the within-class distribution of a clinically relevant factor (here disease severity among $Y{=}1$ patients) differs between groups, enforcing $Z\perp A\mid Y$ requires mapping patients with differing disease severities to the same latent representation, depending on which group a patient belongs to (center). The aligned representation then hides a real difference in disease presentation, and the same $z$ no longer means the same severity in both groups (see also~\cite{lohaus2022two}).}
  \label{fig:intra-class}
\end{figure}

Class-conditional representation invariance instead aligns the representation distributions separately within each outcome class:
\begin{equation}
\label{eq:conditional_invariance}
    p(z\mid y,a)=p(z\mid y)
    \quad \forall a,y
    \quad\Longleftrightarrow\quad
    Z\perp A\mid Y.
\end{equation}
This condition allows the marginal representation distributions \(p(z\mid a)\) to differ when the groups have different outcome prevalences. Thus, class-conditional representation invariance does not imply marginal representation invariance ($Z\perp A\mid Y
    \;\not\Longrightarrow\;
    Z\perp A$).

\begin{prop} \label{prop:condinv}
Class-conditional representation invariance implies the separation/equalized odds fairness criterion.
\end{prop}

\begin{proof}
If the risk prediction $R$ and thresholded prediction $\hat{Y}$ depend only on $Z$, then they inherit the conditional independence of $Z$ implying $R\perp A\mid Y
    ~\text{and}~
    \hat Y\perp A\mid Y.$
This property is known as \emph{separation}~\citep{hardt2016equality,madras2018learningadversariallyfairtransferable}; separation on the classification level ($\hat{Y}$) in the case of binary classification is also known as \emph{equalized odds}.
\end{proof}

We note that when \(Y=1\) we get equal $\mathrm{TPRs}$ (\emph{equal opportunity} fairness criterion), whereas with the additional
\(Y=0\) we also get equal $\mathrm{FPRs}$.
Therefore, class-conditional representation invariance implies $\Delta_{\mathrm{EOpp}}=0
    ~\text{and}~
    \Delta_{\mathrm{EO}}=0$.

\subsection{Invariance-fairness incompatibilities}

Combining Props.~\ref{prop:marginv} and \ref{prop:condinv} with known fairness incompatibility theorems, we obtain the following invariance-fairness incomparability results for classification:

\begin{cor}
Assume that $R$ and $Y$ are not independent of $A$. In this case: (i) Assume that $Y$ is binary; then marginal representation invariance is incompatible with the separation fairness criterion. If group-wise prevalences differ, representation invariance is incompatible with separation even for multiple classes. (ii) Marginal representation invariance is incompatible with the sufficiency fairness criterion, as well as with calibration by groups. (iii) Class-conditional representation invariance is incompatible with the independence fairness criterion. (iv) Assume that all events in the joint distribution of $(A, R, Y)$ have positive probability. Then, class-conditional representation invariance is incompatible with the sufficiency fairness criterion, as well as with calibration by groups.
\end{cor}

\begin{proof}
(i) If marginal invariance holds, the independence criterion holds by Prop.~\ref{prop:marginv}. For binary $Y$, (i) follows from \cite[Ch.~3, Prop.~3]{barocas2023fairness}. If prevalences depend on $A$ and marginal independence holds, then constant predictive rates lead to differences in either TPR, FPR, or both, across those groups where prevalences differ. As a consequence, separation does not hold.

(ii) If marginal invariance holds, the independence criterion holds by Prop.~\ref{prop:marginv}, and thus sufficiency cannot hold by~\cite[Ch.~3, Prop.~2]{barocas2023fairness}. If calibration by groups were to hold, then sufficiency would hold as well, which it cannot; this proves the claim for calibration by groups.

(iii) If class-conditional invariance holds, then equalized odds holds by Prop.~\ref{prop:condinv}. Thus, independence cannot hold by \cite[Ch.~3, Prop.~3]{barocas2023fairness}.

(iv) If class-conditional invariance holds, then equalized odds holds by Prop.~\ref{prop:condinv}. Thus, sufficiency cannot hold by \cite[Ch.~3, Prop.~4]{barocas2023fairness}. If calibration by groups were to hold, then sufficiency would hold as well, which it cannot; this proves the claim for calibration by groups.
\end{proof}

These results all follow from known incompatibility results as stated in~\citep{barocas2023fairness}; we note that even stronger or more detailed fairness incompatibility results from the literature will generalize into invariance-fairness incompatibility results following the same line of reasoning. 

We also note that groupwise prevalence differences, which are used as a running challenging example in this paper, are a special case of $A$ and $Y$ not being independent, which is an assumption for all these incompatibility results. Fig.~\ref{fig:intra-class} shows a further concern about class-conditional invariance conflicting with within-class differences in disease severity across groups.

\subsection{What happens under prevalence shifts?}

Consider a fixed model evaluated on two populations $P$ and $Q$, with unchanged within-class distribution (i.e., $P(X\mid Y,A)=Q(X\mid Y,A)$) but a change in group-specific prevalence ($P(Y\mid A)\neq Q(Y\mid A)$). This could, e.g., happen if a disease over time becomes more prevalent in one group due to, e.g., updated diagnostic criteria that fit the group better, or behavioral changes that affect one group but not the other (e.g., a new contraceptive pill entering the market affecting women, but not men). It also commonly occurs when intentionally up- or downsampling certain classes in training and evaluation datasets. In such settings, we show:

\begin{prop}
 Class-conditional representation invariance is preserved under group-specific prevalence shifts, while marginal representation invariance is not.
\end{prop}

\begin{proof}
As the model is fixed, its conditional representation distribution remains unchanged ($P(Z\mid Y,A)=Q(Z\mid Y,A)$).
Therefore, if representations across groups for each outcome $Y$ are aligned for $P$, they also remain aligned for $Q$~\citep{asiaee2026fairness}:
\[
Z\perp_P A\mid Y
\quad\Longrightarrow\quad
Z\perp_Q A\mid Y.
\]
For binary outcomes, let $\pi_a^Q=Q(Y=1\mid A=a)$. The representation distribution for $Q$ becomes
\[
Q(Z\mid A=a)
=
\pi_a^Q P(Z\mid Y=1,A=a)
+
(1-\pi_a^Q)P(Z\mid Y=0,A=a).
\]
Changing $\pi$ also changes the proportions of positive and negative representations in each group. Their overall distributions can therefore become different, even when they were aligned under $P$:
\[
Z\perp_P A
\quad\centernot\Longrightarrow\quad
Z\perp_Q A.
\]
This finalizes our proof of both claims.
\end{proof}
These results show that \textbf{marginal representation invariance depends on the population it is evaluated on}~\citep{zhao2022inherent}. Class-conditional representation alignment, on the other hand, is preserved under group-specific prevalence shift. 

\begin{figure*}[t]
  \centering
  \begin{minipage}[c]{0.61\textwidth}
    \centering
    \includegraphics[width=\linewidth]{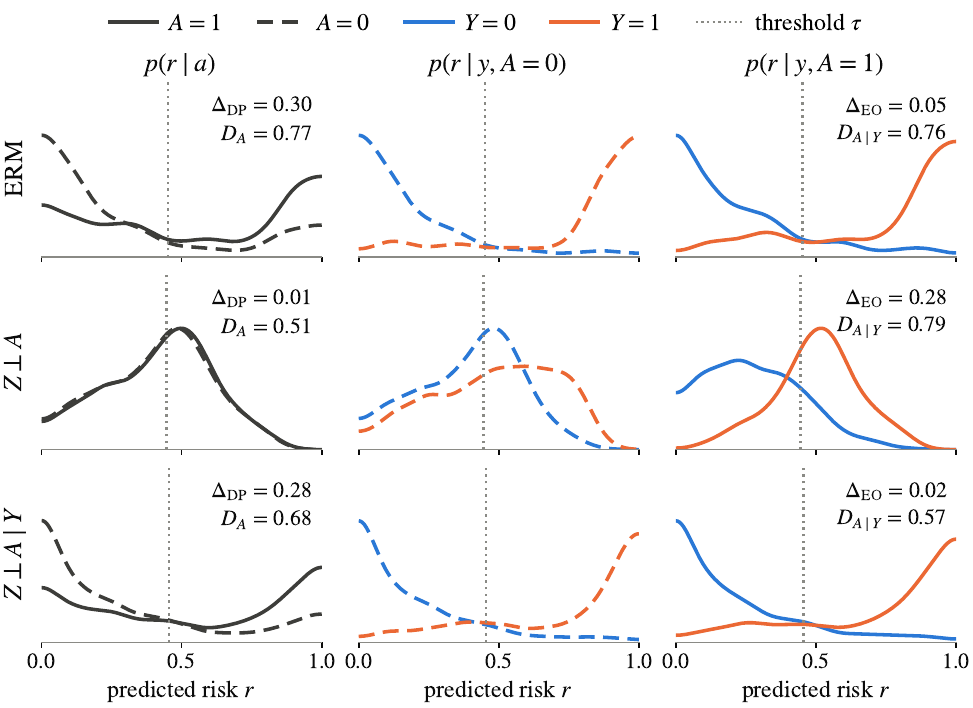}
  \end{minipage}%
  \hfill
  \begin{minipage}[c]{0.38\textwidth}
    \caption{\textbf{Fig.~\ref{fig:schematic} estimated from data} (Dutch, shift-50, seed 42). Rows: ERM, $Z\perp A$, $Z\perp A\mid Y$. Left: marginal risk densities $p(r\mid a)$ $A \in \{0,1\}$, with $\Delta_{\mathrm{DP}}$ and $D_A$. Middle and right: class-conditional densities $p(r\mid y,a)$ within each group, with $\Delta_{\mathrm{EO}}$ and $D_{A\mid Y}$. Dotted: Youden threshold. Marginal alignment makes the left column coincide by pulling the within-class curves apart across groups ($\mathrm{P}_1$); class-conditional alignment makes the middle and right columns coincide and leaves the left column at the base-rate gap ($\mathrm{P}_3$).}
    \label{fig:densities}
  \end{minipage}
\end{figure*}

\section{Experimental Design}
\label{sec:experiments}

\textbf{Datasets and prevalence conditions:}
We examine five tabular benchmarks (Adult, Dutch, Communities \& Crime, Taiwan credit, COMPAS)~\citep{kohavi1996scaling,calders2010three,redmond2002data,yeh2009comparisons} and two CheXpert tasks (edema, cardiomegaly)~\citep{irvin2019chexpert}; depending on the dataset, sensitive attributes are sex, race, or racial composition. Tabular datasets are used to experiment with different group prevalence disparities: natural, equalized, and removal of 50\% or 75\% of positive examples from the lower-prevalence group (\textit{shift-50, shift-75}), resulting in substantial prevalence disparities. In CheXpert, we simply remove 50\% of female-positive cases and compare with natural prevalence. Each condition is a fixed subsample, so every train/test split shares the same prevalence under the given condition (See App.~\ref{tab:app-datasets}).

\textbf{Methods:}
All methods are compared against empirical risk minimization ($\mathrm{ERM}$) baselines. On tabular data, we use categorical Fair Normalizing Flows (FNFs)~\citep{balunović2022fairnormalizingflows} sweeping over $\gamma \in \{0, .02, .1, .2, .5, .9, 1\}$, two MADE hidden widths $\{50,100\}$, and two feature sets, across four prevalence conditions, and three seeds $\{42,123,456\}$ to yield $5 \times 4 \times 3 \times (7 \times 2 \times 2 + 2) = 1{,}800$ runs. For CheXpert, we compare the invariant methods (FNFs, MMD regularizer, adversarial) across two tasks, with the same three seeds, $\gamma \in \{.05, 0.50, 0.95\}$, yielding $2 \times 3 \times (3 \times 3 + 1) = 60$ runs. Data splits are 60/10/10/20\% (tabular) and 60/5/5/10/20\% (CheXpert) for training, validation, probe-validation, and test partitions; thresholds are selected on the validation set.

\textbf{Evaluation:}
\textit{Decodability} $D_A$ is measured with independent probes (logistic regression, gradient-boosted tree, 3 MLPs) trained on held-out probe-validation representations. We evaluate predictive performance using the AUROC score. For fairness evaluation, we report group differences: $\Delta_\mathrm{DP}$, $\Delta_\mathrm{EO}$ and $\Delta_\mathrm{EOpp}$, along with $\Delta \mathrm{PPV}$ and $\Delta \mathrm{NPV}$. \textit{Calibration} is measured by calibration bias ($\Delta_{\mathrm{bias}}=|(\bar y_1-\bar p_1)-(\bar y_0-\bar p_0)|$) and group-wise recalibration intercepts~$\Delta_{\mathrm{cal}}=|a_1-a_0|$~\citep{cox1958two}~(unlike $\mathrm{ECE}$, both stay unbiased under unequal group sizes~\citep{petersen2023assessingfairnessriskscore}). All results are reported as mean $\pm$ standard deviation over training seeds. Paired bootstrapping uses the same samples for each method and is matched to the ERM baseline: 1,000 resamples for tabular, and 2,000 resamples for CheXpert. Additional results for robustness, including 95\% percentile intervals, Bonferroni-adjusted confidence intervals for probe-AUROC, sensitivity analysis over $\gamma$, and calibration-bin tests, are in Appendixes~\ref{app:tabular} and~\ref{app:chexpert}.

\section{Results}
\label{sec:results}

\begin{table*}[t]
\centering
\footnotesize
\setlength{\tabcolsep}{4pt}
\renewcommand{\arraystretch}{1.05}
\begin{tabular}{clccccccc}
\toprule
 & & Utility & \multicolumn{2}{c}{Decodability of $A$} & Indep. & Separation & Sufficiency & Group calibration \\
\cmidrule(lr){3-3}\cmidrule(lr){4-5}\cmidrule(lr){6-6}\cmidrule(lr){7-7}\cmidrule(lr){8-8}\cmidrule(lr){9-9}
 & Method & AUROC & $D_A$ & $D_{A\mid Y}$ & $\Delta_{\mathrm{DP}}$ & $\Delta_{\mathrm{EO}}$ & $\Delta_{\mathrm{PPV}}\,/\,\Delta_{\mathrm{NPV}}$ & $\Delta_{\mathrm{bias}}\,/\,\Delta_{\mathrm{cal}}$ \\
\midrule
\multirow{3}{*}{\rotatebox[origin=c]{90}{Adult}}
 & ERM                   & $0.89$ & $0.93$ & $0.99$ & $0.40$ & $0.28$ & $0.18\,/\,0.05$ & $0.02\,/\,0.32$ \\
 & $Z\perp A$            & $0.79$ & $0.58$ & $0.85$ & $\mathbf{0.04}\,\checkmark$ & $0.40$ & $0.49\,/\,0.04$ & $0.23\,/\,2.12$ \\
 & $Z\perp A\mid Y$      & $0.86$ & $0.62$ & $0.55$ & $0.17$ & $\mathbf{0.08}\,\checkmark$ & $0.40\,/\,0.07$ & $0.15\,/\,1.80$ \\
\midrule
\multirow{3}{*}{\rotatebox[origin=c]{90}{Dutch}}
 & ERM                   & $0.90$ & $0.77$ & $0.77$ & $0.30$ & $0.05$ & $0.27\,/\,0.28$ & $0.22\,/\,2.00$ \\
 & $Z\perp A$            & $0.71$ & $0.52$ & $0.79$ & $\mathbf{0.01}\,\checkmark$ & $0.31$ & $0.63\,/\,0.21$ & $0.44\,/\,2.32$ \\
 & $Z\perp A\mid Y$      & $0.89$ & $0.68$ & $0.56$ & $0.27$ & $\mathbf{0.03}\,\checkmark$ & $0.34\,/\,0.27$ & $0.23\,/\,2.05$ \\
\midrule
\multirow{3}{*}{\rotatebox[origin=c]{90}{C\,\&\,C}}
 & ERM                   & $0.85$ & $0.79$ & $0.66$ & $0.43$ & $0.20$ & $0.50\,/\,0.30$ & $0.23\,/\,2.10$ \\
 & $Z\perp A$            & $0.67$ & $0.55$ & $0.73$ & $\mathbf{0.05}\,\checkmark$ & $0.28$ & $0.71\,/\,0.31$ & $0.55\,/\,2.87$ \\
 & $Z\perp A\mid Y$      & $0.85$ & $0.74$ & $0.57$ & $0.40$ & $\mathbf{0.28}$\, & $0.53\,/\,0.37$ & $0.29\,/\,2.34$ \\
\midrule
\multirow{3}{*}{\rotatebox[origin=c]{90}{Taiwan}}
 & ERM                   & $0.75$ & $0.63$ & $0.62$ & $0.07$ & $0.08$ & $0.19\,/\,0.06$ & $0.08\,/\,0.68$ \\
 & $Z\perp A$            & $0.68$ & $0.52$ & $0.69$ & $\mathbf{0.01}\,\checkmark$ & $0.14$ & $0.28\,/\,0.04$ & $0.11\,/\,0.86$ \\
 & $Z\perp A\mid Y$      & $0.75$ & $0.54$ & $0.53$ & $0.04$ & $\mathbf{0.01}\,\checkmark$ & $0.18\,/\,0.07$ & $0.09\,/\,0.79$ \\
\midrule
\multirow{3}{*}{\rotatebox[origin=c]{90}{COMP.}}
 & ERM                   & $0.71$ & $0.68$ & $0.65$ & $0.21$ & $0.15$ & $0.23\,/\,0.15$ & $0.13\,/\,0.78$ \\
 & $Z\perp A$            & $0.64$ & $0.56$ & $0.63$ & $\mathbf{0.02}\,\checkmark$ & $0.11$ & $0.33\,/\,0.18$ & $0.24\,/\,1.19$ \\
 & $Z\perp A\mid Y$      & $0.69$ & $0.56$ & $0.54$ & $0.10$ & $\mathbf{0.04}\,\checkmark$ & $0.27\,/\,0.18$ & $0.20\,/\,1.09$ \\
\midrule\midrule
\multirow{3}{*}{\rotatebox[origin=c]{90}{Edema}}
 & ERM                   & $0.84$ & $0.98$ & $0.98$ & $0.15$ & $0.10$ & $0.08\,/\,0.02$ & $0.01\,/\,0.09$ \\
 & $Z\perp A$            & $0.83$ & $0.62$ & $0.63$ & $\mathbf{0.01}\,\checkmark$ & $0.05$ & $0.16\,/\,0.06$ & $0.10\,/\,0.80$ \\
 & $Z\perp A\mid Y$      & $0.83$ & $0.60$ & $0.59$ & $0.05$ & $\mathbf{0.02}\,\checkmark$ & $0.13\,/\,0.05$ & $0.07\,/\,0.58$ \\
\midrule
\multirow{3}{*}{\rotatebox[origin=c]{90}{Cardio.}}
 & ERM                   & $0.85$ & $0.98$ & $0.98$ & $0.13$ & $0.09$ & $0.09\,/\,0.02$ & $0.01\,/\,0.19$ \\
 & $Z\perp A$            & $0.84$ & $0.64$ & $0.65$ & $\mathbf{0.01}\,\checkmark$ & $0.05$ & $0.17\,/\,0.03$ & $0.07\,/\,0.96$ \\
 & $Z\perp A\mid Y$      & $0.84$ & $0.62$ & $0.62$ & $0.04$ & $\mathbf{0.01}\,\checkmark$ & $0.15\,/\,0.03$ & $0.05\,/\,0.74$ \\
\bottomrule
\end{tabular}
\caption{\textbf{Fairness criteria under the \textit{shift-50}} on test split, mean over $3$ seeds. $Z\perp A$ is marginal FNF at $\gamma=1$ (CheXpert: $\gamma=0.95$; strongest setting), $Z\perp A\mid Y$ class-conditional FNF at the same $\gamma$, both with the feature set and width selected by the lowest probe AUROC. $D_A$ and $D_{A\mid Y}$ are the maximum probe AUROC for recovering $A$ from $Z$ marginally and within class. $\checkmark$ denotes the criterion each invariance target guarantees ($\mathrm{P}_{1,3}$); every other $\Delta$ is not-stable relative to ERM. $\pm$std are omitted for space and reported, together with the full $\gamma$ and feature-set sweep and the remaining CheXpert configurations, in Appendix~\ref{app:seeds} and~\ref{app:ablation_hparams}.}
\label{tab:fnf-results}
\end{table*}

Table~\ref{tab:fnf-results} summarizes the \textit{shift-50} case; the rest, including full-hyperparameter sweeps and per-seed scores, are in the Appendix~\ref{app:seeds}. We find that ${D_A}$ drops for every alignment method, relative to $\mathrm{ERM}$, but no method achieves exact statistical invariance: the best methods, i.e., with the lowest $D_A$ values, reach $\mathrm{AUROC}$ $0.52$--$0.58$, while the within-class probes reach $0.53$--$0.62$. The results reflect approximate invariance, not exact invariance; therefore, we advise reading fairness conclusions with residual dependence in mind. In the following, we refer to propositions from Section~\ref{sec:invariance} using $\mathrm{P}_n$.

\paragraph{Marginal alignment achieves \textit{independence} but not \textit{separation} or \textit{sufficiency} ($\mathrm{P}_{1,2}$).}
Marginal alignment methods reduce $\mathrm{\Delta DP}$ to $0.01-0.05$ across all datasets ($\mathrm{P}_1$). However, due to inherent incompatibility, it degrades separation ($\mathrm{\Delta \mathrm{EO}}$). On the Dutch dataset, probe-AUC reduces $0.77 \rightarrow 0.52$ and the $\mathrm{\Delta DP}$ $0.30 \rightarrow 0.01$, while $\mathrm{\Delta EO}$ increases $0.05 \rightarrow 0.31$. The predictive performance cost reflects the $e_0+e_1\geq\vert{}\pi_0-\pi_1\vert{}$ bound, resulting in task-$\mathrm{AUROC}$ drop of $0.07$--$0.19$ points across datasets ($\mathrm{P}_2$)~(Detailed in Appendix~\ref{app:base_rate}). Finally, forcing identical score distributions across groups breaks \textit{sufficiency} and \textit{calibration}:$\Delta{\mathrm{PPV}}$ increases in four of five datasets (Dutch $0.27 \to 0.63$; Adult $0.18 \to 0.49$; C\&C $0.50 \to 0.71$; Credit $0.19 \to 0.28$), and the group calibration-bias gap $\Delta_{\mathrm{bias}}$ nearly doubles (Dutch $0.22 \to 0.44$; C\&C $0.23 \to 0.55$). 

\paragraph{Class-conditional alignment implies \textit{separation} ($\mathrm{P}_{3}$).}
Class-condition representation invariance produces the smallest $\mathrm{\Delta \mathrm{EO}}$ (Dutch $0.05\!\to\!0.03$, Credit $0.08\!\to\!0.01$, COMPAS $0.15\!\to\!0.04$, Adult $0.28\!\to\!0.08$; the exception is C\&C whose stratum for $A=0$ has $<80$ rows), implying \textit{separation} ($\mathrm{P}_3$). Unlike marginal alignment, task-$\mathrm{AUROC}$ is maintained within $0.01-0.03$ of $\mathrm{ERM}$. Calibration and \textit{sufficiency} metrics slightly degrade (Dutch $\Delta{\mathrm{PPV}}$: $0.27 \to 0.34$; $\Delta_{\mathrm{bias}}$: $0.22 \to 0.23$), showing conflict between score-level separation and calibration. Notice this only applies to \textit{risk score separation} ($R \perp A \mid Y$); calibrated scores can still yield equal error rates under thresholded decisions~\citep{reich2020possibility}.

\paragraph{Prevalence shifts break marginal but not conditional alignment ($\mathrm{P}_4$).} 
Fig.~\ref{fig:tradeoff-plot} demonstrates evaluation stability under change in population shifts. When base-rate $P(Y \mid A)$ changes while $P(X \mid A, Y)$ remains fixed: marginal alignment methods see increase in $\mathrm{\Delta _\mathrm{DP}}$ ($0.02~(\text{equalized}) \rightarrow 0.19~(\text{shift-75})$), hence \textit{independence} is lost. Under the same shift, conditional alignment maintains a stable $\mathrm{\Delta_\mathrm{EO}}$ ($0.05 \rightarrow 0.03$), but $\mathrm{\Delta_\mathrm{DP}}$ increases ($0.01 \rightarrow$ 0.32); confirming $\mathrm{P}_4$.  


\begin{figure*}[t]
  \centering
  \begin{minipage}[c]{0.68\textwidth}
    \centering
    \includegraphics[width=\linewidth]{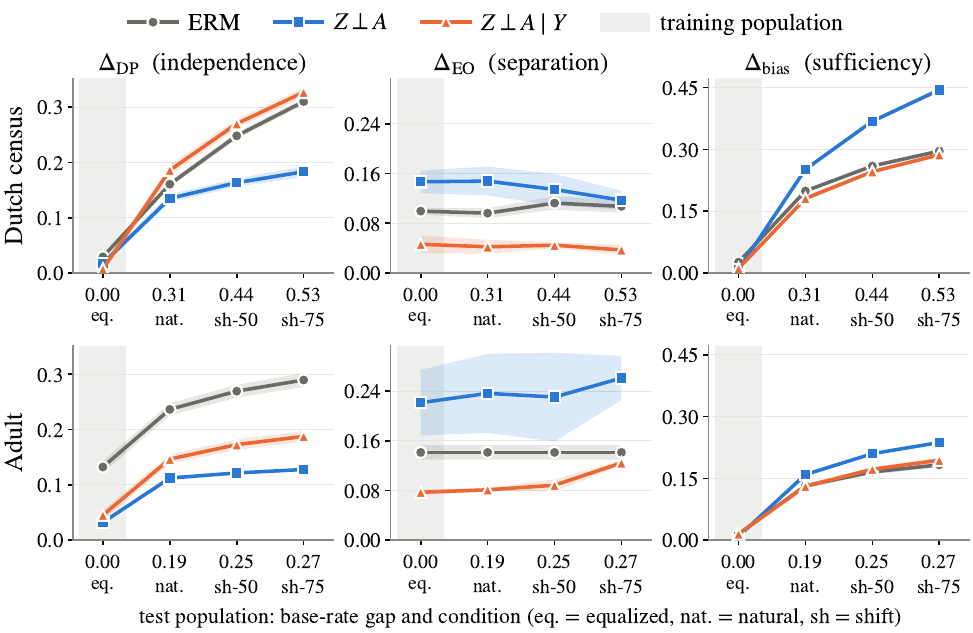}
  \end{minipage}%
  \hfill
  \begin{minipage}[c]{0.30\textwidth}
    \caption{\textbf{Guarantees under a change of evaluation population} ($\mathrm{P}_4$). Trained on the equalized condition (shaded) and evaluated on increasing base-rate gap: $P(Y \mid A)$ changes, $P(X \mid Y,A)$ fixed. $Z \perp A$ loses parity as the gap grows; $Z \perp A \mid Y$ keeps $\Delta_{\mathrm{EO}}$ flat and loses parity by exactly the base-rate gap. $\Delta_{\text{bias}}$ grows. Mean $\pm$ std over three seeds.}
    \label{fig:tradeoff-plot}
  \end{minipage}
  \label{fig:deployment}
\end{figure*}

\section{Discussion}

\paragraph{Removing demographic information as a fairness objective?}
The motivating question for this work is whether models should be trained to \textit{not} \textit{encode} group membership in their latent space. Our analysis shows that representation invariance is neither desirable nor sufficient for fairness. Enforcing strict marginal representation invariance while ensuring demographic parity also removes genuinely useful information and hurts model performance when outcome prevalence differs~\citep{barocas2023fairness,dwork2012fairness,zhao2022inherent,zhao2022fundamental}. This goes beyond the well-known \textit{leveling down} effect~\citep{mittelstadt2023unfairness,zietlow2022leveling}, as we do not just harm performance for the advantaged group; marginal invariance is incompatible with commonly desired fairness criteria. In short, removing information via marginal alignment methods is harmful to the fairness objective: enforcing parity at the cost of performance, separation, and sufficiency.

\paragraph{Removing unnecessarily strong demographic encoding is desirable.} As illustrated in Table~\ref{tab:fnf-results}, class-conditional representation invariance reduces decodability without costing $\mathrm{AUROC}$ and improves $\mathrm{EO}$ relative to $\mathrm{ERM}$. This suggests that some (excess) demographic information learned by $\mathrm{ERM}$ can be erased without costing fairness and performance. This presents a different trade-off, removing information about intra-class variations~(e.g.~comorbidities or disease phenotypes, see~\cite{castro2020causality,li2020rethinking,mahajan2021domain,oakden2020hidden}), and compromises calibration by groups~\citep{barocas2023fairness,hedden2021statistical}, although recent work questions whether this is truly problematic~\citep{kina2026prevalencecalibrationshortcutmitigation}. Every trade-off we discuss stays aligned with the well-established incompatibility between \textit{independence}, \textit{separation} and \textit{sufficiency}~\citep{kleinberg2016inherent,barocas2023fairness}. Even with the reported cost of class-conditional invariance, there is no reason to prefer marginal representation invariance over it.

\paragraph{The strength of constraints matters.} Representations, scores, and thresholded decisions impose different requirements. Marginal representation invariance makes the score distribution identical across all groups. However, at a chosen threshold, achieving parity only requires the same positive prediction rate. Similarly, aligning $Z$ within each outcome class aligns the class-conditional score distribution, whereas achieving separation at a threshold requires equal $\mathrm{TPRs}$ and $\mathrm{FPRs}$. This shows that representation invariance is \textit{sufficient} but not \emph{necessary} for satisfying decision-level criteria~\citep{madras2018learningadversariallyfairtransferable}. This difference also matters for calibration: under unequal prevalences, class-conditional alignment of risk scores generally conflicts with calibration by groups, even though equality of error rates at a chosen threshold does not impose the same restrictions~\citep{reich2020possibility,petersen2023assessingfairnessriskscore}.
Further relaxations are conceivable and potentially desirable, such as distribution \emph{support} matching~\citep{Tong2022}, or matching only the first moment of the conditional score distributions~(\cite{kina2026prevalencecalibrationshortcutmitigation}; also see Appendix~\ref{app:score-level}).



\paragraph{When can representation invariance help?} Our results do not imply that invariance-based methods are never useful, and the drawbacks discussed depend on what form is applied. Marginal representation invariance should generally be avoided when group prevalences ($\pi_a$) are not near equal. Invariance might also be used as a less restrictive mitigation technique~\citep{wu2022fairprune,chiu2023toward} by using weaker regularization strength (in our experiments, $\gamma < 1$ achieves a relatively better performance-parity tradeoff) or in unsupervised domain adaptation where labels $y$ are unavailable for the target domain~\citep{moyer2020scanner,li2018domain}. Here, representation invariance should be applied strategically, acknowledging the associated risks, and favoring conditional approaches for supervised tasks or optimizing post-hoc thresholds if strict latent alignment is unnecessary~\citep{hardt2016equality,reich2020possibility}.

\paragraph{Statistical invariance vs. model invariance.}
We discuss a case where a model treats two patients with similar physiology differently due to their group membership. A \emph{model} is invariant to $A$ if any change in $A$ while keeping everything else fixed leaves $\hat{y}$ unchanged. This still allows $\hat{Y}$ and $A$ to be correlated when the labels $Y$ correlate with $A$.
By contrast, a marginally invariant \emph{representation} requires $Z, \hat{Y} \perp A$ even when $Y$ and $A$ are correlated. Enforcing this can require disparate treatment: if disease presentations differ between groups, the model still maps them to identical distributions by treating similar patients from different groups differently~\citep{lohaus2022two}. 

To address this, an alternative conceptual goal is to learn invariant mechanisms that are stable across all environments~\citep{arjovsky2019invariant,makar2022causally,puli2021out,scholkopf2021toward,subbaswamy2020development}. One approach is \textit{counterfactual invariance}, where the attribute $A$ or a specific intervention should not change the model's prediction~\citep{veitch2021counterfactual}. When $A$ is a demographic attribute, this corresponds to \textit{counterfactual fairness}~\citep{kusner2017counterfactual}. However, defining these counterfactuals is challenging in medical imaging. Changing $A$ may also change anatomy, physiology, overlapping diseases, or even disease presentation. Therefore, specifying which factors change and which remain fixed can affect the predictions. This makes any fairness guarantees dependent on these modeling choices, their (unknown or undiscovered) confounding, and also the validity of the causal model~\citep{chiappa2018pathspecificcounterfactualfairness,kilbertus2020sensitivity}.


\paragraph{Limitations of our work.}
We advise reading our findings with caution. Our models and methods achieved approximate rather than exact statistical invariance, leaving demographic traces in representations. However, this does not change our primary conclusion. Our empirical analysis is limited to classification tasks and simulated prevalence differences between groups, which may not reflect real-world extremes. Fairness gaps and calibration costs are sensitive to operating thresholds, compression, or model selection. \textbf{As future work}, we plan to study why \textit{decodability} behaves differently across data domains and methods, and expand our analysis to complex tasks like segmentation.

\section{Conclusion}
\label{sec:conc}

We have shown that representation invariance, in spite of its intuitive appeal for fair AI, is incompatible with most popular fairness criteria when $A$ and $Y$ are not independent. We show, moreover, that previous work showing the benefit of representation independence might have been a result of failure to completely remove demographic information, suggesting that while invariance is often not desirable, reducing the encoding of demographic information can have a regularizing effect on both performance and fairness.
In particular, we implore the community to avoid the often strongly highlighted concerns regarding encoded demographic information. While encoded demographics may sound deeply alarming to users and even developers not familiar with the invariance-fairness-performance tradeoff, these encoded demographics are, in many practical applications, necessary to obtain the most equitable and high-performing models.
There \emph{is} a real risk of demographic shortcut learning in the case where models encode demographic information more strongly than is necessary for accurate prediction, as was the case for all $\mathrm{ERM}$ models in our experiments.
In such cases, the impetus must not be to aim for \emph{no} demographic encoding, however, but rather to limit encoding to the smallest level that is sufficient for maintaining accurate, robust, and fair prediction.







\subsection*{Generative AI Disclosure}

Following fair and responsible use of AI for research, we list the use of generative AI for the following sub-tasks:

\begin{itemize}
    \item \textbf{Major use:} Writing boilerplate code for the experimental pipeline, automated logging and tracking runs across thousands of experiments, formatting \LaTeX{} tables for the main paper and appendix (entries were verified).
    \item \textbf{Minor use:} Writing scripts to visualize metrics from raw logs, and mapping different notations across the literature (all mathematical proofs and steps are verified), minor checks for grammar, phrasing, flow, and typo fixes.
    \item \textbf{Not used for:} Ideation, core methodology, deriving theoretical proofs, main manuscript, literature referencing and the bibliography, or structuring the paper.
\end{itemize}

All AI-assisted outputs were manually verified by the authors. The authors take full responsibility for the final output and claims in this work.



\bibliography{iclr2027_conference}

@article{quarta2026hidden,
  title={Hidden in Plain Sight: Vector Embeddings give away Demographic Information},
  author={Quarta, Alessandro and Santos, Filipe and Marzullo, Aldo and Sousa, Joao MC and Vieira, Susana M and Celi, Leo Anthony and Seyyed-Kalantari, Laleh and Calimeri, Francesco},
  journal={IEEE Journal of Biomedical and Health Informatics},
  year={2026},
  publisher={IEEE}
}

@article{glocker2023algorithmic,
  title={Algorithmic encoding of protected characteristics in chest X-ray disease detection models},
  author={Glocker, Ben and Jones, Charles and Bernhardt, M{\'e}lanie and Winzeck, Stefan},
  journal={EBioMedicine},
  volume={89},
  year={2023},
  publisher={Elsevier}
}

@article{yang2024limits,
  title={The limits of fair medical imaging {AI} in real-world generalization},
  author={Yang, Yuzhe and Zhang, Haoran and Gichoya, Judy W and Katabi, Dina and Ghassemi, Marzyeh},
  journal={Nature medicine},
  volume={30},
  number={10},
  pages={2838--2848},
  year={2024},
  publisher={Nature Publishing Group US New York}
}

@article{duffy2022confounders,
  title={Confounders mediate {AI} prediction of demographics in medical imaging},
  author={Duffy, Grant and Clarke, Shoa L and Christensen, Matthew and He, Bryan and Yuan, Neal and Cheng, Susan and Ouyang, David},
  journal={NPJ digital medicine},
  volume={5},
  number={1},
  pages={188},
  year={2022},
  publisher={Nature Publishing Group UK London}
}

@article{gichoya2022ai,
  title={{AI} recognition of patient race in medical imaging: a modelling study},
  author={Gichoya, Judy Wawira and Banerjee, Imon and Bhimireddy, Ananth Reddy and Burns, John L and Celi, Leo Anthony and Chen, Li-Ching and Correa, Ramon and Dullerud, Natalie and Ghassemi, Marzyeh and Huang, Shih-Cheng and others},
  journal={The Lancet Digital Health},
  volume={4},
  number={6},
  pages={e406--e414},
  year={2022},
  publisher={Elsevier}
}

@article{yi2021radiology,
  title={Radiology “forensics”: determination of age and sex from chest radiographs using deep learning},
  author={Yi, Paul H and Wei, Jinchi and Kim, Tae Kyung and Shin, Jiwon and Sair, Haris I and Hui, Ferdinand K and Hager, Gregory D and Lin, Cheng Ting},
  journal={Emergency Radiology},
  volume={28},
  number={5},
  pages={949--954},
  year={2021},
  publisher={Springer}
}

@article{lotter2024acquisition,
  title={Acquisition parameters influence AI recognition of race in chest x-rays and mitigating these factors reduces underdiagnosis bias},
  author={Lotter, William},
  journal={Nature Communications},
  volume={15},
  number={1},
  pages={7465},
  year={2024},
  publisher={Nature Publishing Group UK London}
}

@article{brown2023detecting,
  title={Detecting shortcut learning for fair medical AI using shortcut testing},
  author={Brown, Alexander and Tomasev, Nenad and Freyberg, Jan and Liu, Yuan and Karthikesalingam, Alan and Schrouff, Jessica},
  journal={Nature communications},
  volume={14},
  number={1},
  pages={4314},
  year={2023},
  publisher={Nature Publishing Group UK London}
}

@article{stanley2025and,
  title={Where, why, and how is bias learned in medical image analysis models? A study of bias encoding within convolutional networks using synthetic data},
  author={Stanley, Emma AM and Souza, Raissa and Wilms, Matthias and Forkert, Nils D},
  journal={EBioMedicine},
  volume={111},
  year={2025},
  publisher={Elsevier}
}

@inproceedings{lin2024shortcut,
  title={Shortcut learning in medical image segmentation},
  author={Lin, Manxi and Weng, Nina and Mikolaj, Kamil and Bashir, Zahra and Svendsen, Morten BS and Tolsgaard, Martin G and Christensen, Anders N and Feragen, Aasa},
  booktitle={International Conference on Medical Image Computing and Computer-Assisted Intervention},
  pages={623--633},
  year={2024},
  organization={Springer}
}

@inproceedings{jones2023role,
  title={The role of subgroup separability in group-fair medical image classification},
  author={Jones, Charles and Roschewitz, M{\'e}lanie and Glocker, Ben},
  booktitle={International Conference on Medical Image Computing and Computer-Assisted Intervention},
  pages={179--188},
  year={2023},
  organization={Springer}
}

@inproceedings{akintande2025medicine,
  title={Medicine after death: {XAI} and algorithmic fairness under label bias},
  author={Akintande, Olalekan Joseph and Bigdeli, Siavash Arjomand and Feragen, Aasa},
  booktitle={Fourth European Workshop on Algorithmic Fairness},
  pages={171--186},
  year={2025},
  organization={ML Research Press}
}

@article{parikh2026towards,
  title={Towards Fairness under Label Bias in Image Segmentation: Impact, Measurement and Mitigation},
  author={Parikh, Aditya and Frank, Stella and Das, Sneha and Feragen, Aasa},
  journal={arXiv preprint arXiv:2605.06891},
  year={2026}
}

@inproceedings{stanley2025exploring,
  title = {Exploring the Interplay of Label Bias with Subgroup Size and Separability: A Case Study in Mammographic Density},
  author = {Stanley, Emma AM and Mehta, Raghav and Roschewitz, M{\'e}lanie},
  booktitle = {Fairness of AI in Medical Imaging: Third International Workshop, FAIMI 2025, Held in Conjunction with MICCAI 2025, Daejeon, South Korea, September 23, 2025, Proceedings},
  pages = {74},
  year = {2025},
  organization = {Springer Nature},
}

@inproceedings{jones2025rethinking,
  title={Rethinking fair representation learning for performance-sensitive tasks},
  author={Jones, Charles and de Sousa Ribeiro, Fabio and Roschewitz, M{\'e}lanie and Castro, Daniel and Glocker, Ben},
  booktitle={International Conference on Learning Representations},
  volume={2025},
  pages={29269--29288},
  year={2025}
}

@article{mittelstadt2023unfairness,
  title={The unfairness of fair machine learning: Leveling down and strict egalitarianism by default},
  author={Mittelstadt, Brent and Wachter, Sandra and Russell, Chris},
  journal={Mich. Tech. L. Rev.},
  volume={30},
  pages={1},
  year={2023},
  publisher={HeinOnline}
}

@inproceedings{edwards2015censoring,
  title     = {{Censoring Representations with an Adversary}},
  author    = {Edwards, Harrison and Storkey, Amos J.},
  booktitle = {International Conference on Learning Representations},
  year      = {2016},
}

@inproceedings{louizos2015variational,
  author       = {Christos Louizos and
                  Kevin Swersky and
                  Yujia Li and
                  Max Welling and
                  Richard S. Zemel},
  editor       = {Yoshua Bengio and
                  Yann LeCun},
  title        = {The Variational Fair Autoencoder},
  booktitle    = {4th International Conference on Learning Representations, {ICLR} 2016,
                  San Juan, Puerto Rico, May 2-4, 2016, Conference Track Proceedings},
  year         = {2016},
  url          = {http://arxiv.org/abs/1511.00830},
  bibsource    = {dblp computer science bibliography, https://dblp.org}
}

@inproceedings{zhang2018mitigating,
  title={Mitigating unwanted biases with adversarial learning},
  author={Zhang, Brian Hu and Lemoine, Blake and Mitchell, Margaret},
  booktitle={Proceedings of the 2018 AAAI/ACM Conference on AI, Ethics, and Society},
  pages={335--340},
  year={2018}
}

@article{zhao2022inherent,
  title={Inherent tradeoffs in learning fair representations},
  author={Zhao, Han and Gordon, Geoffrey J},
  journal={Journal of Machine Learning Research},
  volume={23},
  number={57},
  pages={1--26},
  year={2022}
}

@inproceedings{creager2019flexibly,
  title={Flexibly fair representation learning by disentanglement},
  author={Creager, Elliot and Madras, David and Jacobsen, J{\"o}rn-Henrik and Weis, Marissa and Swersky, Kevin and Pitassi, Toniann and Zemel, Richard},
  booktitle={International conference on machine learning},
  pages={1436--1445},
  year={2019},
  organization={PMLR}
}

@inproceedings{zemel2013learning,
  title={Learning fair representations},
  author={Zemel, Rich and Wu, Yu and Swersky, Kevin and Pitassi, Toni and Dwork, Cynthia},
  booktitle={International conference on machine learning},
  pages={325--333},
  year={2013},
  organization={PMLR}
}

@article{rashid2026stride,
  title={Stride-Net: Fairness-Aware Disentangled Representation Learning for Chest X-Ray Diagnosis},
  author={Rashid, Darakshan and Imam, Raza and Mahapatra, Dwarikanath and Lall, Brejesh},
  journal={arXiv preprint arXiv:2602.10875},
  year={2026}
}

@inproceedings{wang2024advancing,
  title={Advancing graph counterfactual fairness through fair representation learning},
  author={Wang, Zichong and Chu, Zhibo and Blanco, Ronald and Chen, Zhong and Chen, Shu-Ching and Zhang, Wenbin},
  booktitle={Joint European Conference on Machine Learning and Knowledge Discovery in Databases},
  pages={40--58},
  year={2024},
  organization={Springer}
}

@InProceedings{Tong2022,
  author    = {Shangyuan Tong and Timur Garipov and Yang Zhang and Shiyu Chang and Tommi S. Jaakkola},
  booktitle = {International Conference on Learning Representations},
  title     = {Adversarial Support Alignment},
  year      = {2022},
  url       = {https://openreview.net/forum?id=26gKg6x-ie},
}

@inproceedings{sadri2025mutual,
  title={Mutual Information Regularization for Fairness-Aware Deep Imaging Representations},
  author={Sadri, Amir Reza and DeSilvio, Thomas and Viswanath, Satish E},
  booktitle={International Conference on Medical Image Computing and Computer-Assisted Intervention},
  pages={399--409},
  year={2025},
  organization={Springer}
}

@article{sarridis2024flac,
  title={Flac: Fairness-aware representation learning by suppressing attribute-class associations},
  author={Sarridis, Ioannis and Koutlis, Christos and Papadopoulos, Symeon and Diou, Christos},
  journal={IEEE Transactions on Pattern Analysis and Machine Intelligence},
  volume={47},
  number={2},
  pages={1148--1160},
  year={2024},
  publisher={IEEE}
}

@article{hardt2016equality,
  title = {Equality of opportunity in supervised learning},
  author = {Hardt, Moritz and Price, Eric and Srebro, Nati},
  journal = {Advances in neural information processing systems},
  volume = {29},
  year = {2016},
}

@book{barocas2023fairness,
  title = {Fairness and machine learning: Limitations and opportunities},
  author = {Barocas, Solon and Hardt, Moritz and Narayanan, Arvind},
  year = {2023},
  publisher = {MIT press},
}

@InProceedings{kleinberg2016inherent,
  author =	{Kleinberg, Jon and Mullainathan, Sendhil and Raghavan, Manish},
  title =	{{Inherent Trade-Offs in the Fair Determination of Risk Scores}},
  booktitle =	{8th Innovations in Theoretical Computer Science Conference (ITCS 2017)},
  pages =	{43:1--43:23},
  year =	{2017},
  volume =	{67},
  doi =		{10.4230/LIPIcs.ITCS.2017.43}
}

@InProceedings{reich2020possibility,
  author =	{Lazar Reich, Claire and Vijaykumar, Suhas},
  title =	{{A Possibility in Algorithmic Fairness: Can Calibration and Equal Error Rates Be Reconciled?}},
  booktitle =	{2nd Symposium on Foundations of Responsible Computing (FORC 2021)},
  pages =	{4:1--4:21},
  series =	{Leibniz International Proceedings in Informatics (LIPIcs)},
  year =	{2021},
  volume =	{192},
  doi =		{10.4230/LIPIcs.FORC.2021.4}
}

@inproceedings{mcnamara2019costs,
  title={Costs and benefits of fair representation learning},
  author={McNamara, Daniel and Ong, Cheng Soon and Williamson, Robert C},
  booktitle={Proceedings of the 2019 AAAI/ACM Conference on AI, Ethics, and Society},
  pages={263--270},
  year={2019}
}

@article{zhao2022fundamental,
  title={Fundamental limits and tradeoffs in invariant representation learning},
  author={Zhao, Han and Dan, Chen and Aragam, Bryon and Jaakkola, Tommi S and Gordon, Geoffrey J and Ravikumar, Pradeep},
  journal={Journal of machine learning research},
  volume={23},
  number={340},
  pages={1--49},
  year={2022}
}

@article{li2020rethinking,
  title={Rethinking distributional matching based domain adaptation},
  author={Li, Bo and Wang, Yezhen and Che, Tong and Zhang, Shanghang and Zhao, Sicheng and Xu, Pengfei and Zhou, Wei and Bengio, Yoshua and Keutzer, Kurt},
  journal={arXiv preprint arXiv:2006.13352},
  year={2020}
}

@inproceedings{zietlow2022leveling,
  title={Leveling down in computer vision: Pareto inefficiencies in fair deep classifiers},
  author={Zietlow, Dominik and Lohaus, Michael and Balakrishnan, Guha and Kleindessner, Matth{\"a}us and Locatello, Francesco and Sch{\"o}lkopf, Bernhard and Russell, Chris},
  booktitle={2022 IEEE/CVF Conference on Computer Vision and Pattern Recognition (CVPR)},
  pages={10400--10411},
  year={2022},
  organization={IEEE}
}

@misc{madras2018learningadversariallyfairtransferable,
      title={Learning Adversarially Fair and Transferable Representations}, 
      author={David Madras and Elliot Creager and Toniann Pitassi and Richard Zemel},
      year={2018},
      eprint={1802.06309},
      archivePrefix={arXiv},
      primaryClass={cs.LG},
      url={https://arxiv.org/abs/1802.06309}, 
}

@article{arjovsky2019invariant,
  title={Invariant risk minimization},
  author={Arjovsky, Martin and Bottou, L{\'e}on and Gulrajani, Ishaan and Lopez-Paz, David},
  journal={arXiv preprint arXiv:1907.02893},
  year={2019}
}

@inproceedings{makar2022causally,
  title={Causally motivated shortcut removal using auxiliary labels},
  author={Makar, Maggie and Packer, Ben and Moldovan, Dan and Blalock, Davis and Halpern, Yoni and D’Amour, Alexander},
  booktitle={International Conference on Artificial Intelligence and Statistics},
  pages={739--766},
  year={2022},
  organization={PMLR}
}

@inproceedings{puli2021out,
  title={Out-of-distribution generalization in the presence of nuisance-induced spurious correlations},
  author    = {Puli, Aahlad Manas and Zhang, Lily H and Oermann, Eric Karl and Ranganath, Rajesh},
  booktitle = {International Conference on Learning Representations},
  year      = {2022},
}

@article{scholkopf2021toward,
  title={Toward causal representation learning},
  author={Sch{\"o}lkopf, Bernhard and Locatello, Francesco and Bauer, Stefan and Ke, Nan Rosemary and Kalchbrenner, Nal and Goyal, Anirudh and Bengio, Yoshua},
  journal={Proceedings of the IEEE},
  volume={109},
  number={5},
  pages={612--634},
  year={2021},
  publisher={IEEE}
}

@article{subbaswamy2020development,
  title={From development to deployment: dataset shift, causality, and shift-stable models in health {AI}},
  author={Subbaswamy, Adarsh and Saria, Suchi},
  journal={Biostatistics},
  volume={21},
  number={2},
  pages={345--352},
  year={2020},
  publisher={Oxford University Press}
}

@inproceedings{veitch2021counterfactual,
author = {Veitch, Victor and D'Amour, Alexander and Yadlowsky, Steve and Eisenstein, Jacob},
title = {Counterfactual invariance to spurious correlations: why and how to pass stress tests},
year = {2021},
booktitle = {Proceedings of the 35th International Conference on Neural Information Processing Systems},
articleno = {1239},
numpages = {13},
series = {NIPS '21}
}

@article{kusner2017counterfactual,
  title={Counterfactual fairness},
  author={Kusner, Matt J and Loftus, Joshua and Russell, Chris and Silva, Ricardo},
  journal={Advances in neural information processing systems},
  volume={30},
  year={2017}
}

@article{lohaus2022two,
  title={Are two heads the same as one? Identifying disparate treatment in fair neural networks},
  author={Lohaus, Michael and Kleindessner, Matth{\"a}us and Kenthapadi, Krishnaram and Locatello, Francesco and Russell, Chris},
  journal={Advances in Neural Information Processing Systems},
  volume={35},
  pages={16548--16562},
  year={2022}
}

@misc{chiappa2018pathspecificcounterfactualfairness,
      title={Path-Specific Counterfactual Fairness}, 
      author={Silvia Chiappa and Thomas P. S. Gillam},
      year={2018},
      eprint={1802.08139},
      archivePrefix={arXiv},
      primaryClass={stat.ML},
      url={https://arxiv.org/abs/1802.08139}, 
}

@inproceedings{kilbertus2020sensitivity,
  title={The sensitivity of counterfactual fairness to unmeasured confounding},
  author={Kilbertus, Niki and Ball, Philip J and Kusner, Matt J and Weller, Adrian and Silva, Ricardo},
  booktitle={Uncertainty in artificial intelligence},
  pages={616--626},
  year={2020},
  organization={PMLR}
}

@inproceedings{balunović2022fairnormalizingflows,
  title     = {{Fair Normalizing Flows}},
    author={Mislav Balunović and Anian Ruoss and Martin Vechev},
  booktitle = {International Conference on Learning Representations},
  year      = {2022},
}

@inproceedings{wu2022fairprune,
  title={Fairprune: Achieving fairness through pruning for dermatological disease diagnosis},
  author={Wu, Yawen and Zeng, Dewen and Xu, Xiaowei and Shi, Yiyu and Hu, Jingtong},
  booktitle={International Conference on Medical Image Computing and Computer-Assisted Intervention},
  pages={743--753},
  year={2022},
  organization={Springer}
}

@inproceedings{chiu2023toward,
  title={Toward fairness through fair multi-exit framework for dermatological disease diagnosis},
  author={Chiu, Ching-Hao and Chung, Hao-Wei and Chen, Yu-Jen and Shi, Yiyu and Ho, Tsung-Yi},
  booktitle={International Conference on Medical Image Computing and Computer-Assisted Intervention},
  pages={97--107},
  year={2023},
  organization={Springer}
}

@inproceedings{li2018domain,
  title={Domain generalization via conditional invariant representations},
  author={Li, Ya and Gong, Mingming and Tian, Xinmei and Liu, Tongliang and Tao, Dacheng},
  booktitle={Proceedings of the AAAI conference on artificial intelligence},
  volume={32},
  year={2018}
}

@article{moyer2020scanner,
  title={Scanner invariant representations for diffusion MRI harmonization},
  author={Moyer, Daniel and Ver Steeg, Greg and Tax, Chantal MW and Thompson, Paul M},
  journal={Magnetic resonance in medicine},
  volume={84},
  number={4},
  pages={2174--2189},
  year={2020},
  publisher={Wiley Online Library}
}

@inproceedings{dwork2012fairness,
  title = {Fairness through awareness},
  author = {Dwork, Cynthia and Hardt, Moritz and Pitassi, Toniann and Reingold, Omer and Zemel, Richard},
  booktitle = {Proceedings of the 3rd innovations in theoretical computer science conference},
  pages = {214--226},
  year = {2012},
}

@article{hedden2021statistical,
  title={On statistical criteria of algorithmic fairness},
  author={Hedden, Brian},
  journal={Philosophy and Public Affairs},
  volume={49},
  pages={209},
  year={2021},
  publisher={HeinOnline}
}

@article{castro2020causality,
  title={Causality matters in medical imaging},
  author={Castro, Daniel C and Walker, Ian and Glocker, Ben},
  journal={Nature Communications},
  volume={11},
  number={1},
  pages={3673},
  year={2020},
  publisher={Nature Publishing Group UK London}
}

@inproceedings{mahajan2021domain,
  title={Domain generalization using causal matching},
  author={Mahajan, Divyat and Tople, Shruti and Sharma, Amit},
  booktitle={International conference on machine learning},
  pages={7313--7324},
  year={2021},
  organization={PMLR}
}

@inproceedings{oakden2020hidden,
  title = {Hidden stratification causes clinically meaningful failures in machine learning for medical imaging},
  author = {Oakden-Rayner, Luke and Dunnmon, Jared and Carneiro, Gustavo and R{\'e}, Christopher},
  booktitle = {Proceedings of the ACM conference on health, inference, and learning},
  pages = {151--159},
  year = {2020},
}

@misc{asiaee2026fairness,
      title={Fairness Under Group-Conditional Prior Probability Shift: Invariance, Drift, and Target-Aware Post-Processing}, 
      author={Amir Asiaee and Kaveh Aryan},
      year={2026},
      eprint={2602.05144},
      archivePrefix={arXiv},
      primaryClass={cs.LG},
      url={https://arxiv.org/abs/2602.05144}, 
}

@article{cox1958two,
  title={Two further applications of a model for binary regression},
  author={Cox, David R},
  journal={Biometrika},
  volume={45},
  number={3/4},
  pages={562--565},
  year={1958},
  publisher={JSTOR}
}

@misc{petersen2023assessingfairnessriskscore,
      title={On (assessing) the fairness of risk score models}, 
      author={Eike Petersen and Melanie Ganz and Sune Hannibal Holm and Aasa Feragen},
      year={2023},
      eprint={2302.08851},
      archivePrefix={arXiv},
      primaryClass={cs.LG},
      url={https://arxiv.org/abs/2302.08851}, 
}

@InProceedings{kina2026prevalencecalibrationshortcutmitigation,
  author    = {Mohamed Amine Kina and Eike Petersen},
  booktitle = {17th International Workshop on Machine Learning in Medical Imaging (MLMI 2026) at MICCAI 2026},
  title     = {Prevalence calibration as shortcut mitigation},
  year      = {2026},
}

@article{ganin2016domain,
  title = {Domain-adversarial training of neural networks},
  author = {Ganin, Yaroslav and Ustinova, Evgeniya and Ajakan, Hana and Germain, Pascal and Larochelle, Hugo and Laviolette, Fran{\c{c}}ois and March, Mario and Lempitsky, Victor},
  journal = {Journal of machine learning research},
  volume = {17},
  number = {59},
  pages = {1--35},
  year = {2016},
}

@article{long2018conditional,
  title={Conditional adversarial domain adaptation},
  author={Long, Mingsheng and Cao, Zhangjie and Wang, Jianmin and Jordan, Michael},
  journal={Advances in neural information processing systems},
  volume={31},
  year={2018}
}

@inproceedings{kohavi1996scaling,
author = {Kohavi, Ron},
title = {Scaling up the accuracy of {Naive-Bayes} classifiers: a decision-tree hybrid},
year = {1996},
publisher = {AAAI Press},
booktitle = {Proceedings of the Second International Conference on Knowledge Discovery and Data Mining},
pages = {202--207},
numpages = {6},
location = {Portland, Oregon},
series = {KDD'96}
}

@article{calders2010three,
  title={Three naive {B}ayes approaches for discrimination-free classification},
  author={Calders, Toon and Verwer, Sicco},
  journal={Data mining and knowledge discovery},
  volume={21},
  number={2},
  pages={277--292},
  year={2010},
  publisher={Springer}
}

@article{redmond2002data,
  title={A data-driven software tool for enabling cooperative information sharing among police departments},
  author={Redmond, Michael and Baveja, Alok},
  journal={European Journal of Operational Research},
  volume={141},
  number={3},
  pages={660--678},
  year={2002},
  publisher={Elsevier}
}

@article{yeh2009comparisons,
  title={The comparisons of data mining techniques for the predictive accuracy of probability of default of credit card clients},
  author={Yeh, I-Cheng and Lien, Che-hui},
  journal={Expert systems with applications},
  volume={36},
  number={2},
  pages={2473--2480},
  year={2009},
  publisher={Elsevier}
}

@inproceedings{irvin2019chexpert,
  title={Chexpert: A large chest radiograph dataset with uncertainty labels and expert comparison},
  author    = {Irvin, Jeremy and Rajpurkar, Pranav and Ko, Michael and Yu, Yifan and Ciurea-Ilcus, Silviana and Chute, Christopher and Marklund, Henrik and Haghgoo, Behzad and Ball, Robyn L. and Shpanskaya, Katie S. and Seekins, Jayne and Mong, David A. and Halabi, Safwan S. and Sandberg, Jesse K. and Jones, Ricky and Larson, David B. and Langlotz, Curtis P. and Patel, Bhavik N. and Lungren, Matthew P. and Ng, Andrew Y.},
  booktitle={Proceedings of the AAAI conference on artificial intelligence},
  volume={33},
  pages={590--597},
  year={2019}
}

@article{gretton2006kernel,
  title = {A kernel method for the two-sample-problem},
  author = {Gretton, Arthur and Borgwardt, Karsten and Rasch, Malte and Sch{\"o}lkopf, Bernhard and Smola, Alex},
  journal = {Advances in neural information processing systems},
  volume = {19},
  year = {2006},
}

@inproceedings{perez2017fair,
  title={Fair kernel learning},
  author={P{\'e}rez-Suay, Adri{\'a}n and Laparra, Valero and Mateo-Garc{\'\i}a, Gonzalo and Mu{\~n}oz-Mar{\'\i}, Jordi and G{\'o}mez-Chova, Luis and Camps-Valls, Gustau},
  booktitle={Joint European Conference on Machine Learning and Knowledge Discovery in Databases},
  pages={339--355},
  year={2017},
  organization={Springer}
}

@article{belrose2023leace,
  title={Leace: Perfect linear concept erasure in closed form},
  author={Belrose, Nora and Schneider-Joseph, David and Ravfogel, Shauli and Cotterell, Ryan and Raff, Edward and Biderman, Stella},
  journal={Advances in Neural Information Processing Systems},
  volume={36},
  pages={66044--66063},
  year={2023}
}

@inproceedings{ravfogel2022adversarial,
  title={Adversarial concept erasure in kernel space},
  author={Ravfogel, Shauli and Vargas, Francisco and Goldberg, Yoav and Cotterell, Ryan},
  booktitle={Proceedings of the 2022 Conference on Empirical Methods in Natural Language Processing},
  pages={6034--6055},
  year={2022}
}

@inproceedings{shakibania2026obliviate,
  title={Obliviate: Erasing Concepts from Autoregressive Image Generation Models},
  author={Shakibania, Hossein and Grebe, Jonas Henry and Braun, Tobias and Aktemur, Ege and Aslani, Saleh and Yi{\u{g}}it, Mehmet G and Rohrbach, Marcus},
  booktitle={European Conference on Computer Vision},
  pages={665--683},
  year={2026},
  organization={Springer}
}

@article{petersen2023demographically,
  title={Are demographically invariant models and representations in medical imaging fair?},
  author={Petersen, Eike and Ferrante, Enzo and Ganz, Melanie and Feragen, Aasa},
  journal={arXiv preprint arXiv:2305.01397},
  year={2023}
}

@article{rahimi2007random,
  title={Random features for large-scale kernel machines},
  author={Rahimi, Ali and Recht, Benjamin},
  journal={Advances in neural information processing systems},
  volume={20},
  year={2007}
}

@misc{cerrato202410yearsfairrepresentations,
      title={10 Years of Fair Representations: Challenges and Opportunities}, 
      author={Mattia Cerrato and Marius Köppel and Philipp Wolf and Stefan Kramer},
      year={2024},
      eprint={2407.03834},
      archivePrefix={arXiv},
      primaryClass={cs.LG},
      url={https://arxiv.org/abs/2407.03834}, 
}
\bibliographystyle{iclr2027_conference}

\clearpage

\appendix
\section{Appendix 1: Tabular Datasets}
\label{app:tabular}

\subsection{Additional Details: Fairness under increasing base-rate gap and prevalence shift conditions}
\label{app:base_rate}

Refer to Fig.~\ref{fig:shift-all}, Fig.~\ref{fig:prevalence-gap} and Table~\ref{tab:app-shift} for results of all evaluation metrics against the training base-rate gap for the four prevalence conditions. Class-conditional alignment is similar to $\mathrm{ERM}$'s $\mathrm{AUROC}$ at every gap, while the marginal flow manages to reduce decodability, and its summed error $e_0+e_1$ rises with the gap. $D_A$ for $\mathrm{ERM}$ follows the gap where $A$ becomes decodable from $Y$. $\Delta_{\mathrm{DP}}$ for $\mathrm{ERM}$ and class-conditional alignment follows the default base-rate; $\Delta_{\mathrm{EO}}$ for the marginal method increases with the gap; $\Delta_{\mathrm{PPV}}$ and $\Delta_{\mathrm{bias}}$ grow for every method, as anticipated from our propositions. Fig.~\ref{fig:prop2} plots the summed error of all $2{,}040$ evaluations against the bound~\citep{mcnamara2019costs}; none falls below it.

\begin{figure}[!htbp]
  \centering
  \includegraphics[width=\linewidth]{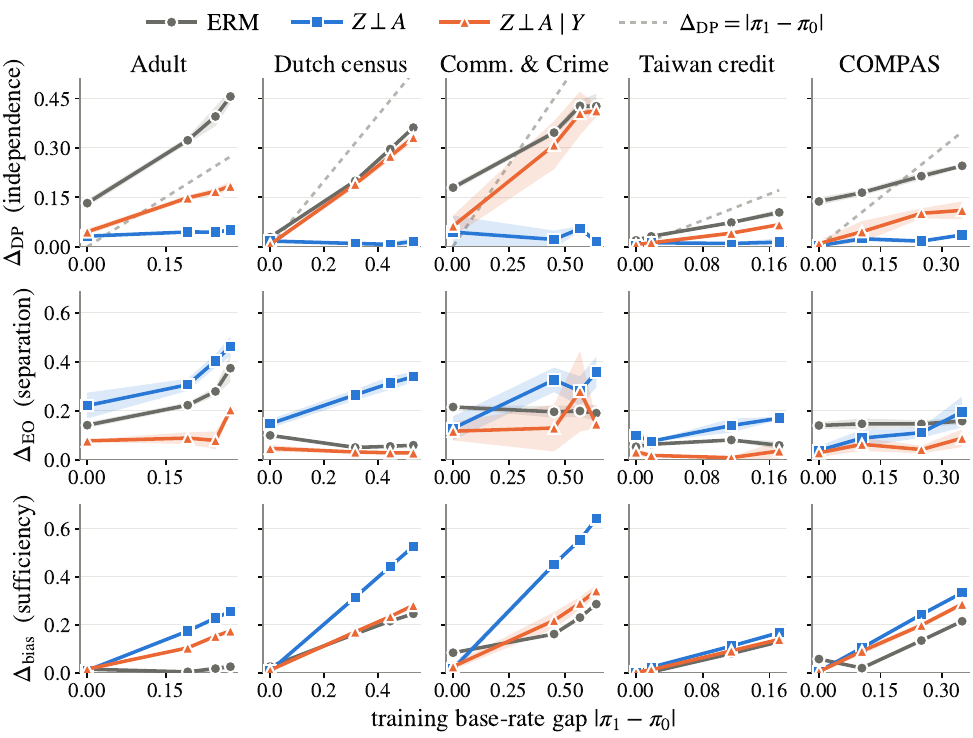}
  \caption{\textbf{Fairness gaps increase with the base-rate gap.} Top: $\Delta_{\mathrm{DP}}$; Middle: $\Delta_{\mathrm{EO}}$; Bottom: $\Delta_{\mathrm{bias}}$ against the training base-rate gap $|\pi_1-\pi_0|$ over the four prevalence conditions, per dataset and method. Marginal representation invariance harms \textit{separation} for parity at a rate set by the gap ($\mathrm{P}_1$); class-conditional representation invariance holds $\Delta_{\mathrm{EO}}$ ($\mathrm{P}_3$) but not $\Delta_{\mathrm{bias}}$ ($\mathrm{P}_3$). Mean $\pm$ std over three seeds reported.}
  \label{fig:prevalence-gap}
\end{figure}

\begin{figure}[!htbp]
  \centering
  \includegraphics[width=\linewidth]{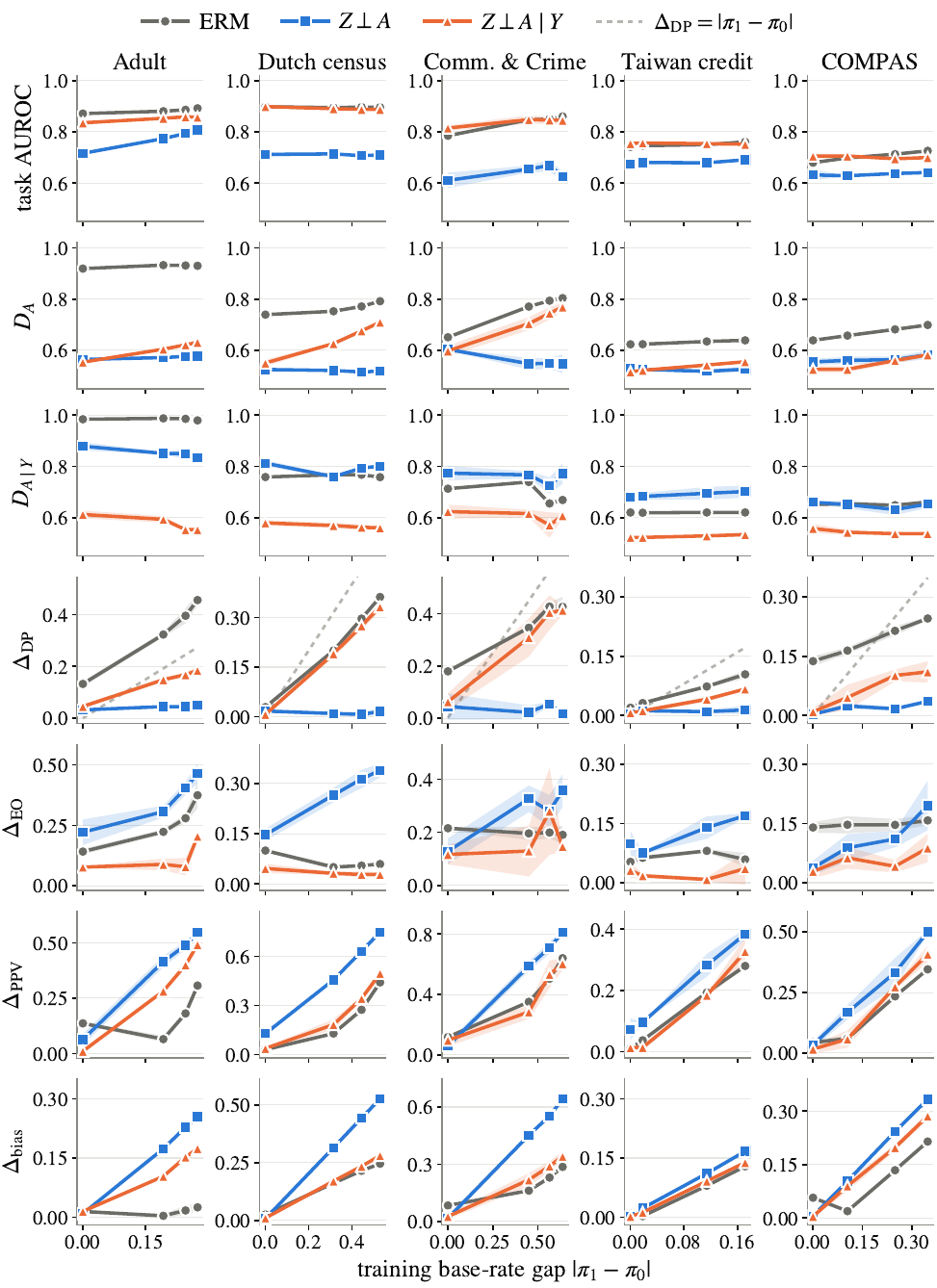}
  \caption{Every metric against the training base-rate gap $|\pi_1-\pi_0|$ (equalized, natural, shift-50, shift-75; matched test). Dotted line in the $\Delta_{\mathrm{DP}}$ row: $\Delta_{\mathrm{DP}}=|\pi_1-\pi_0|$. Mean $\pm$ sd over seeds.}
  \label{fig:shift-all}
\end{figure}

\begin{figure}[!htbp]
  \centering
  \includegraphics[width=.62\linewidth]{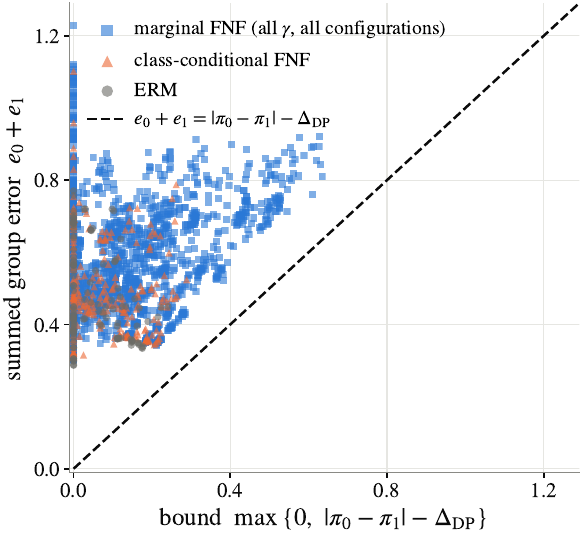}
  \caption{Summed group error $e_0+e_1$ against the bound $\max\{0,|\pi_0-\pi_1|-\Delta_{\mathrm{DP}}\}$ (mentioned in $\mathrm{P}_2$) for every model$\times$prevalence condition.}
  \label{fig:prop2}
\end{figure}

\begin{table*}[!htbp]
\centering
\caption{\textbf{Evaluation metrics by training condition across datasets} (reported as mean $\pm$ std over all seeds). $\Delta\pi = |\pi_1-\pi_0|$ is the base-rate gap. \textbf{Conditions}: Eq (equalized), Nat (natural), S50 (shift-50), S75 (shift-75). \textbf{Methods}: ERM, \textbf{M} (marginal alignment, $Z\perp A$), and \textbf{C} (class-conditional alignment, $Z\perp A\mid Y$). Datasets: Adult, Dutch census (Dutch), Communities \& Crime (Crime), Taiwan credit (Credit), and COMPAS.}
\label{tab:app-shift}
\resizebox{\textwidth}{!}{%
\begin{tabular}{lll r rrrrrrrr}
\toprule
\textbf{Dataset} & \textbf{Cond.} & \textbf{Method} & $\boldsymbol{\Delta\pi}$ & \textbf{AUROC} & $\boldsymbol{D_A}$ & $\boldsymbol{D_{A\mid Y}}$ & $\boldsymbol{\Delta_{\mathrm{DP}}}$ & $\boldsymbol{\Delta_{\mathrm{EO}}}$ & $\boldsymbol{\Delta_{\mathrm{PPV}}}$ & $\boldsymbol{\Delta_{\mathrm{bias}}}$ & $\boldsymbol{e_0+e_1}$ \\
\midrule
\multirow{12}{*}{Adult} 
      & Eq  & ERM & 0.00 & 0.87 $\pm$ 0.00 & 0.92 $\pm$ 0.00 & 0.98 $\pm$ 0.00 & 0.13 $\pm$ 0.01 & 0.14 $\pm$ 0.01 & 0.14 $\pm$ 0.01 & 0.01 $\pm$ 0.00 & 0.42 $\pm$ 0.02 \\
      & Eq  & M   & 0.00 & 0.72 $\pm$ 0.01 & 0.56 $\pm$ 0.00 & 0.88 $\pm$ 0.02 & 0.03 $\pm$ 0.00 & 0.22 $\pm$ 0.05 & 0.06 $\pm$ 0.03 & 0.01 $\pm$ 0.00 & 0.57 $\pm$ 0.06 \\
      & Eq  & C   & 0.00 & 0.84 $\pm$ 0.01 & 0.55 $\pm$ 0.00 & 0.61 $\pm$ 0.02 & 0.05 $\pm$ 0.01 & 0.08 $\pm$ 0.00 & 0.01 $\pm$ 0.00 & 0.01 $\pm$ 0.00 & 0.54 $\pm$ 0.05 \\
      \addlinespace[2pt]
      & Nat & ERM & 0.19 & 0.88 $\pm$ 0.00 & 0.93 $\pm$ 0.00 & 0.99 $\pm$ 0.00 & 0.32 $\pm$ 0.01 & 0.22 $\pm$ 0.01 & 0.06 $\pm$ 0.01 & 0.00 $\pm$ 0.00 & 0.38 $\pm$ 0.01 \\
      & Nat & M   & 0.19 & 0.77 $\pm$ 0.00 & 0.57 $\pm$ 0.01 & 0.85 $\pm$ 0.01 & 0.04 $\pm$ 0.00 & 0.31 $\pm$ 0.02 & 0.42 $\pm$ 0.03 & 0.17 $\pm$ 0.00 & 0.65 $\pm$ 0.07 \\
      & Nat & C   & 0.19 & 0.85 $\pm$ 0.00 & 0.60 $\pm$ 0.00 & 0.59 $\pm$ 0.01 & 0.15 $\pm$ 0.01 & 0.09 $\pm$ 0.02 & 0.28 $\pm$ 0.00 & 0.10 $\pm$ 0.00 & 0.49 $\pm$ 0.04 \\
      \addlinespace[2pt]
      & S50 & ERM & 0.25 & 0.89 $\pm$ 0.00 & 0.93 $\pm$ 0.00 & 0.99 $\pm$ 0.00 & 0.40 $\pm$ 0.03 & 0.28 $\pm$ 0.02 & 0.18 $\pm$ 0.01 & 0.02 $\pm$ 0.01 & 0.37 $\pm$ 0.03 \\
      & S50 & M   & 0.25 & 0.79 $\pm$ 0.00 & 0.58 $\pm$ 0.01 & 0.85 $\pm$ 0.02 & 0.04 $\pm$ 0.00 & 0.40 $\pm$ 0.03 & 0.49 $\pm$ 0.02 & 0.23 $\pm$ 0.00 & 0.68 $\pm$ 0.04 \\
      & S50 & C   & 0.25 & 0.86 $\pm$ 0.00 & 0.62 $\pm$ 0.01 & 0.55 $\pm$ 0.00 & 0.17 $\pm$ 0.01 & 0.08 $\pm$ 0.04 & 0.40 $\pm$ 0.01 & 0.15 $\pm$ 0.00 & 0.51 $\pm$ 0.05 \\
      \addlinespace[2pt]
      & S75 & ERM & 0.27 & 0.89 $\pm$ 0.00 & 0.93 $\pm$ 0.00 & 0.98 $\pm$ 0.00 & 0.46 $\pm$ 0.02 & 0.37 $\pm$ 0.06 & 0.31 $\pm$ 0.01 & 0.03 $\pm$ 0.01 & 0.38 $\pm$ 0.04 \\
      & S75 & M   & 0.27 & 0.81 $\pm$ 0.00 & 0.58 $\pm$ 0.01 & 0.83 $\pm$ 0.01 & 0.05 $\pm$ 0.01 & 0.46 $\pm$ 0.04 & 0.55 $\pm$ 0.01 & 0.25 $\pm$ 0.00 & 0.63 $\pm$ 0.02 \\
      & S75 & C   & 0.27 & 0.86 $\pm$ 0.00 & 0.63 $\pm$ 0.01 & 0.55 $\pm$ 0.00 & 0.18 $\pm$ 0.01 & 0.20 $\pm$ 0.00 & 0.49 $\pm$ 0.01 & 0.17 $\pm$ 0.01 & 0.50 $\pm$ 0.02 \\
\midrule
\multirow{12}{*}{Dutch} 
      & Eq  & ERM & 0.00 & 0.90 $\pm$ 0.00 & 0.74 $\pm$ 0.00 & 0.76 $\pm$ 0.00 & 0.03 $\pm$ 0.00 & 0.10 $\pm$ 0.01 & 0.03 $\pm$ 0.00 & 0.03 $\pm$ 0.00 & 0.30 $\pm$ 0.00 \\
      & Eq  & M   & 0.00 & 0.71 $\pm$ 0.01 & 0.53 $\pm$ 0.00 & 0.81 $\pm$ 0.01 & 0.02 $\pm$ 0.01 & 0.15 $\pm$ 0.02 & 0.13 $\pm$ 0.01 & 0.01 $\pm$ 0.00 & 0.64 $\pm$ 0.03 \\
      & Eq  & C   & 0.00 & 0.90 $\pm$ 0.00 & 0.55 $\pm$ 0.01 & 0.58 $\pm$ 0.01 & 0.01 $\pm$ 0.01 & 0.05 $\pm$ 0.01 & 0.04 $\pm$ 0.01 & 0.01 $\pm$ 0.01 & 0.36 $\pm$ 0.01 \\
      \addlinespace[2pt]
      & Nat & ERM & 0.31 & 0.89 $\pm$ 0.00 & 0.75 $\pm$ 0.00 & 0.77 $\pm$ 0.00 & 0.20 $\pm$ 0.01 & 0.05 $\pm$ 0.01 & 0.13 $\pm$ 0.00 & 0.16 $\pm$ 0.00 & 0.36 $\pm$ 0.00 \\
      & Nat & M   & 0.31 & 0.71 $\pm$ 0.01 & 0.52 $\pm$ 0.00 & 0.76 $\pm$ 0.01 & 0.01 $\pm$ 0.00 & 0.27 $\pm$ 0.02 & 0.46 $\pm$ 0.01 & 0.31 $\pm$ 0.00 & 0.70 $\pm$ 0.01 \\
      & Nat & C   & 0.31 & 0.89 $\pm$ 0.00 & 0.63 $\pm$ 0.00 & 0.57 $\pm$ 0.01 & 0.19 $\pm$ 0.01 & 0.03 $\pm$ 0.00 & 0.18 $\pm$ 0.03 & 0.17 $\pm$ 0.01 & 0.39 $\pm$ 0.02 \\
      \addlinespace[2pt]
      & S50 & ERM & 0.44 & 0.90 $\pm$ 0.00 & 0.77 $\pm$ 0.00 & 0.77 $\pm$ 0.00 & 0.30 $\pm$ 0.00 & 0.05 $\pm$ 0.00 & 0.27 $\pm$ 0.01 & 0.22 $\pm$ 0.00 & 0.35 $\pm$ 0.01 \\
      & S50 & M   & 0.44 & 0.71 $\pm$ 0.00 & 0.52 $\pm$ 0.00 & 0.79 $\pm$ 0.01 & 0.01 $\pm$ 0.01 & 0.31 $\pm$ 0.03 & 0.63 $\pm$ 0.00 & 0.44 $\pm$ 0.00 & 0.73 $\pm$ 0.01 \\
      & S50 & C   & 0.44 & 0.89 $\pm$ 0.00 & 0.68 $\pm$ 0.01 & 0.56 $\pm$ 0.01 & 0.27 $\pm$ 0.01 & 0.03 $\pm$ 0.01 & 0.34 $\pm$ 0.01 & 0.23 $\pm$ 0.01 & 0.38 $\pm$ 0.01 \\
      \addlinespace[2pt]
      & S75 & ERM & 0.53 & 0.90 $\pm$ 0.00 & 0.79 $\pm$ 0.00 & 0.76 $\pm$ 0.00 & 0.36 $\pm$ 0.01 & 0.06 $\pm$ 0.01 & 0.44 $\pm$ 0.02 & 0.25 $\pm$ 0.00 & 0.34 $\pm$ 0.01 \\
      & S75 & M   & 0.53 & 0.71 $\pm$ 0.00 & 0.52 $\pm$ 0.00 & 0.80 $\pm$ 0.01 & 0.02 $\pm$ 0.01 & 0.34 $\pm$ 0.02 & 0.75 $\pm$ 0.01 & 0.53 $\pm$ 0.00 & 0.75 $\pm$ 0.01 \\
      & S75 & C   & 0.53 & 0.89 $\pm$ 0.00 & 0.71 $\pm$ 0.00 & 0.56 $\pm$ 0.00 & 0.33 $\pm$ 0.01 & 0.03 $\pm$ 0.01 & 0.50 $\pm$ 0.02 & 0.28 $\pm$ 0.00 & 0.36 $\pm$ 0.02 \\
\midrule
\multirow{12}{*}{Crime} 
      & Eq  & ERM & 0.00 & 0.78 $\pm$ 0.00 & 0.65 $\pm$ 0.00 & 0.71 $\pm$ 0.00 & 0.18 $\pm$ 0.01 & 0.22 $\pm$ 0.01 & 0.12 $\pm$ 0.00 & 0.08 $\pm$ 0.00 & 0.65 $\pm$ 0.01 \\
      & Eq  & M   & 0.00 & 0.61 $\pm$ 0.03 & 0.60 $\pm$ 0.00 & 0.77 $\pm$ 0.03 & 0.04 $\pm$ 0.05 & 0.13 $\pm$ 0.05 & 0.06 $\pm$ 0.03 & 0.02 $\pm$ 0.02 & 0.94 $\pm$ 0.18 \\
      & Eq  & C   & 0.00 & 0.81 $\pm$ 0.01 & 0.60 $\pm$ 0.02 & 0.62 $\pm$ 0.03 & 0.06 $\pm$ 0.03 & 0.12 $\pm$ 0.04 & 0.10 $\pm$ 0.04 & 0.03 $\pm$ 0.01 & 0.67 $\pm$ 0.08 \\
      \addlinespace[2pt]
      & Nat & ERM & 0.45 & 0.85 $\pm$ 0.00 & 0.77 $\pm$ 0.00 & 0.74 $\pm$ 0.00 & 0.35 $\pm$ 0.01 & 0.20 $\pm$ 0.02 & 0.35 $\pm$ 0.01 & 0.16 $\pm$ 0.00 & 0.47 $\pm$ 0.03 \\
      & Nat & M   & 0.45 & 0.65 $\pm$ 0.01 & 0.55 $\pm$ 0.02 & 0.77 $\pm$ 0.01 & 0.02 $\pm$ 0.03 & 0.33 $\pm$ 0.05 & 0.59 $\pm$ 0.03 & 0.45 $\pm$ 0.01 & 0.82 $\pm$ 0.01 \\
      & Nat & C   & 0.45 & 0.85 $\pm$ 0.02 & 0.70 $\pm$ 0.03 & 0.62 $\pm$ 0.01 & 0.31 $\pm$ 0.07 & 0.13 $\pm$ 0.10 & 0.28 $\pm$ 0.05 & 0.22 $\pm$ 0.03 & 0.47 $\pm$ 0.06 \\
      \addlinespace[2pt]
      & S50 & ERM & 0.57 & 0.85 $\pm$ 0.00 & 0.79 $\pm$ 0.00 & 0.66 $\pm$ 0.00 & 0.43 $\pm$ 0.02 & 0.20 $\pm$ 0.02 & 0.50 $\pm$ 0.01 & 0.23 $\pm$ 0.00 & 0.46 $\pm$ 0.01 \\
      & S50 & M   & 0.57 & 0.67 $\pm$ 0.02 & 0.55 $\pm$ 0.02 & 0.73 $\pm$ 0.03 & 0.05 $\pm$ 0.01 & 0.28 $\pm$ 0.06 & 0.71 $\pm$ 0.04 & 0.55 $\pm$ 0.01 & 0.80 $\pm$ 0.05 \\
      & S50 & C   & 0.57 & 0.85 $\pm$ 0.01 & 0.74 $\pm$ 0.02 & 0.57 $\pm$ 0.05 & 0.40 $\pm$ 0.07 & 0.28 $\pm$ 0.16 & 0.53 $\pm$ 0.09 & 0.29 $\pm$ 0.03 & 0.49 $\pm$ 0.06 \\
      \addlinespace[2pt]
      & S75 & ERM & 0.64 & 0.86 $\pm$ 0.00 & 0.80 $\pm$ 0.00 & 0.67 $\pm$ 0.02 & 0.43 $\pm$ 0.04 & 0.19 $\pm$ 0.03 & 0.64 $\pm$ 0.03 & 0.29 $\pm$ 0.00 & 0.45 $\pm$ 0.00 \\
      & S75 & M   & 0.64 & 0.63 $\pm$ 0.02 & 0.55 $\pm$ 0.04 & 0.77 $\pm$ 0.04 & 0.02 $\pm$ 0.01 & 0.36 $\pm$ 0.06 & 0.81 $\pm$ 0.01 & 0.64 $\pm$ 0.01 & 0.86 $\pm$ 0.03 \\
      & S75 & C   & 0.64 & 0.84 $\pm$ 0.02 & 0.77 $\pm$ 0.02 & 0.61 $\pm$ 0.01 & 0.41 $\pm$ 0.02 & 0.15 $\pm$ 0.02 & 0.60 $\pm$ 0.04 & 0.34 $\pm$ 0.02 & 0.45 $\pm$ 0.03 \\
\midrule
\multirow{12}{*}{Credit} 
      & Eq  & ERM & 0.00 & 0.74 $\pm$ 0.00 & 0.62 $\pm$ 0.00 & 0.62 $\pm$ 0.00 & 0.02 $\pm$ 0.01 & 0.05 $\pm$ 0.01 & 0.01 $\pm$ 0.00 & 0.01 $\pm$ 0.00 & 0.57 $\pm$ 0.01 \\
      & Eq  & M   & 0.00 & 0.67 $\pm$ 0.00 & 0.53 $\pm$ 0.00 & 0.68 $\pm$ 0.01 & 0.00 $\pm$ 0.00 & 0.10 $\pm$ 0.03 & 0.07 $\pm$ 0.04 & 0.00 $\pm$ 0.00 & 0.62 $\pm$ 0.07 \\
      & Eq  & C   & 0.00 & 0.75 $\pm$ 0.01 & 0.52 $\pm$ 0.00 & 0.52 $\pm$ 0.01 & 0.01 $\pm$ 0.01 & 0.03 $\pm$ 0.02 & 0.01 $\pm$ 0.01 & 0.00 $\pm$ 0.00 & 0.52 $\pm$ 0.05 \\
      \addlinespace[2pt]
      & Nat & ERM & 0.02 & 0.75 $\pm$ 0.00 & 0.62 $\pm$ 0.00 & 0.62 $\pm$ 0.00 & 0.03 $\pm$ 0.00 & 0.06 $\pm$ 0.01 & 0.04 $\pm$ 0.00 & 0.00 $\pm$ 0.00 & 0.53 $\pm$ 0.02 \\
      & Nat & M   & 0.02 & 0.68 $\pm$ 0.00 & 0.53 $\pm$ 0.00 & 0.68 $\pm$ 0.01 & 0.01 $\pm$ 0.00 & 0.08 $\pm$ 0.00 & 0.10 $\pm$ 0.00 & 0.02 $\pm$ 0.00 & 0.61 $\pm$ 0.03 \\
      & Nat & C   & 0.02 & 0.76 $\pm$ 0.00 & 0.52 $\pm$ 0.00 & 0.52 $\pm$ 0.01 & 0.01 $\pm$ 0.01 & 0.02 $\pm$ 0.01 & 0.01 $\pm$ 0.01 & 0.01 $\pm$ 0.00 & 0.49 $\pm$ 0.03 \\
      \addlinespace[2pt]
      & S50 & ERM & 0.11 & 0.75 $\pm$ 0.00 & 0.63 $\pm$ 0.00 & 0.62 $\pm$ 0.00 & 0.07 $\pm$ 0.00 & 0.08 $\pm$ 0.00 & 0.19 $\pm$ 0.01 & 0.08 $\pm$ 0.00 & 0.49 $\pm$ 0.03 \\
      & S50 & M   & 0.11 & 0.68 $\pm$ 0.01 & 0.52 $\pm$ 0.01 & 0.69 $\pm$ 0.02 & 0.01 $\pm$ 0.00 & 0.14 $\pm$ 0.03 & 0.28 $\pm$ 0.04 & 0.11 $\pm$ 0.00 & 0.55 $\pm$ 0.06 \\
      & S50 & C   & 0.11 & 0.75 $\pm$ 0.00 & 0.54 $\pm$ 0.01 & 0.53 $\pm$ 0.00 & 0.04 $\pm$ 0.01 & 0.01 $\pm$ 0.01 & 0.18 $\pm$ 0.01 & 0.09 $\pm$ 0.00 & 0.54 $\pm$ 0.09 \\
      \addlinespace[2pt]
      & S75 & ERM & 0.17 & 0.76 $\pm$ 0.00 & 0.64 $\pm$ 0.00 & 0.62 $\pm$ 0.00 & 0.10 $\pm$ 0.01 & 0.06 $\pm$ 0.02 & 0.28 $\pm$ 0.01 & 0.13 $\pm$ 0.00 & 0.50 $\pm$ 0.09 \\
      & S75 & M   & 0.17 & 0.69 $\pm$ 0.00 & 0.53 $\pm$ 0.01 & 0.70 $\pm$ 0.02 & 0.01 $\pm$ 0.01 & 0.17 $\pm$ 0.01 & 0.38 $\pm$ 0.02 & 0.17 $\pm$ 0.00 & 0.53 $\pm$ 0.05 \\
      & S75 & C   & 0.17 & 0.75 $\pm$ 0.00 & 0.56 $\pm$ 0.00 & 0.53 $\pm$ 0.01 & 0.07 $\pm$ 0.00 & 0.04 $\pm$ 0.04 & 0.33 $\pm$ 0.05 & 0.14 $\pm$ 0.00 & 0.41 $\pm$ 0.10 \\
\midrule
\multirow{12}{*}{COMPAS} 
      & Eq  & ERM & 0.00 & 0.68 $\pm$ 0.00 & 0.64 $\pm$ 0.00 & 0.65 $\pm$ 0.01 & 0.14 $\pm$ 0.01 & 0.14 $\pm$ 0.01 & 0.04 $\pm$ 0.01 & 0.06 $\pm$ 0.00 & 0.74 $\pm$ 0.04 \\
      & Eq  & M   & 0.00 & 0.63 $\pm$ 0.02 & 0.56 $\pm$ 0.01 & 0.66 $\pm$ 0.00 & 0.00 $\pm$ 0.00 & 0.04 $\pm$ 0.02 & 0.03 $\pm$ 0.01 & 0.01 $\pm$ 0.01 & 0.79 $\pm$ 0.05 \\
      & Eq  & C   & 0.00 & 0.71 $\pm$ 0.01 & 0.53 $\pm$ 0.01 & 0.56 $\pm$ 0.02 & 0.01 $\pm$ 0.01 & 0.03 $\pm$ 0.02 & 0.02 $\pm$ 0.02 & 0.00 $\pm$ 0.00 & 0.66 $\pm$ 0.02 \\
      \addlinespace[2pt]
      & Nat & ERM & 0.11 & 0.70 $\pm$ 0.00 & 0.66 $\pm$ 0.00 & 0.65 $\pm$ 0.00 & 0.16 $\pm$ 0.01 & 0.15 $\pm$ 0.02 & 0.06 $\pm$ 0.01 & 0.02 $\pm$ 0.00 & 0.68 $\pm$ 0.00 \\
      & Nat & M   & 0.11 & 0.63 $\pm$ 0.00 & 0.56 $\pm$ 0.01 & 0.65 $\pm$ 0.02 & 0.02 $\pm$ 0.01 & 0.09 $\pm$ 0.03 & 0.17 $\pm$ 0.03 & 0.10 $\pm$ 0.01 & 0.84 $\pm$ 0.03 \\
      & Nat & C   & 0.11 & 0.71 $\pm$ 0.01 & 0.53 $\pm$ 0.01 & 0.54 $\pm$ 0.01 & 0.05 $\pm$ 0.03 & 0.06 $\pm$ 0.03 & 0.06 $\pm$ 0.04 & 0.09 $\pm$ 0.01 & 0.66 $\pm$ 0.01 \\
      \addlinespace[2pt]
      & S50 & ERM & 0.25 & 0.71 $\pm$ 0.00 & 0.68 $\pm$ 0.01 & 0.65 $\pm$ 0.00 & 0.21 $\pm$ 0.01 & 0.15 $\pm$ 0.01 & 0.23 $\pm$ 0.01 & 0.13 $\pm$ 0.00 & 0.70 $\pm$ 0.03 \\
      & S50 & M   & 0.25 & 0.64 $\pm$ 0.00 & 0.56 $\pm$ 0.01 & 0.63 $\pm$ 0.02 & 0.02 $\pm$ 0.01 & 0.11 $\pm$ 0.02 & 0.33 $\pm$ 0.06 & 0.24 $\pm$ 0.00 & 0.78 $\pm$ 0.04 \\
      & S50 & C   & 0.25 & 0.69 $\pm$ 0.00 & 0.56 $\pm$ 0.01 & 0.54 $\pm$ 0.00 & 0.10 $\pm$ 0.02 & 0.04 $\pm$ 0.02 & 0.27 $\pm$ 0.01 & 0.20 $\pm$ 0.01 & 0.68 $\pm$ 0.03 \\
      \addlinespace[2pt]
      & S75 & ERM & 0.35 & 0.73 $\pm$ 0.00 & 0.70 $\pm$ 0.00 & 0.66 $\pm$ 0.00 & 0.25 $\pm$ 0.01 & 0.16 $\pm$ 0.01 & 0.35 $\pm$ 0.01 & 0.21 $\pm$ 0.00 & 0.67 $\pm$ 0.08 \\
      & S75 & M   & 0.35 & 0.64 $\pm$ 0.01 & 0.58 $\pm$ 0.02 & 0.65 $\pm$ 0.02 & 0.04 $\pm$ 0.00 & 0.20 $\pm$ 0.06 & 0.50 $\pm$ 0.01 & 0.33 $\pm$ 0.01 & 0.83 $\pm$ 0.03 \\
      & S75 & C   & 0.35 & 0.70 $\pm$ 0.01 & 0.58 $\pm$ 0.01 & 0.54 $\pm$ 0.00 & 0.11 $\pm$ 0.03 & 0.09 $\pm$ 0.04 & 0.41 $\pm$ 0.03 & 0.29 $\pm$ 0.01 & 0.68 $\pm$ 0.10 \\
\bottomrule
\end{tabular}%
}
\end{table*}

\subsection{Ablations: Probe Ablation}
\label{app:probes}

Refer to Fig.~\ref{fig:probes} and Table~\ref{tab:app-probes} for details on probe evaluation. For the marginal alignment methods, the empirical probe $\mathrm{AUROC}$ lies below the bound obtained from the exact per-group MADE densities, while the within-class probe stays high or rises above $\mathrm{ERM}$. For the class-conditional alignment methods, the within-class probe stays within the class bound, and the marginal probe rises to the base-rate floor of $\mathrm{P}_2$.

\begin{figure}[!htbp]
  \centering
  \includegraphics[width=\linewidth]{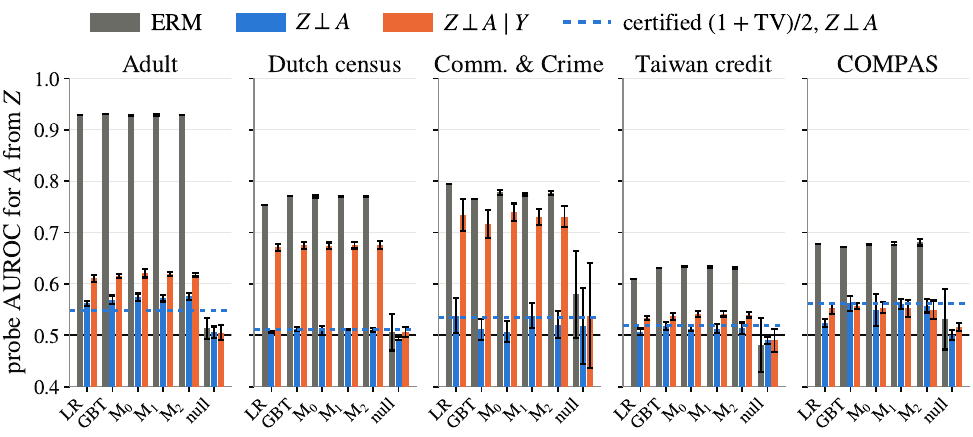}
  \caption{Probe $\mathrm{AUROC}$ to \textit{decode} for $A$ from $Z$ under six probe families (\textit{shift-50}, reported mean $\pm$ std over all seeds). Dotted: $(1+\mathrm{TV})/2$ for the marginal flow.}
  \label{fig:probes}
\end{figure}

\begin{table*}[!htbp]
\centering
\caption{\textbf{Probe ablation at shift 50.} Methods: ERM, Marginal (\textbf{M}: $Z\perp A$), and Conditional (\textbf{C}: $Z\perp A\mid Y$). Columns 3--8 report AUC of probe families predicting $A$ from $Z$: \textbf{LR} (logistic regression), \textbf{GB} (gradient-boosted trees), \textbf{MLP-$k$} ($k$ hidden layers), and \textbf{LR$-A$} (logistic control baseline excluding $A$). Right columns report the original FNF adversary balanced accuracy (\textbf{Adv BA}), the certified balanced accuracy bound $(1+\mathrm{TV})/2$ (\textbf{Bound}), and the maximum within-class probe AUC (\textbf{Max $W$}). Reported as mean $\pm$ sd over three seeds.}
\label{tab:app-probes}
\resizebox{\textwidth}{!}{%
\begin{tabular}{ll cccccc ccc}
\toprule
& & \multicolumn{6}{c}{\textbf{Probe AUC ($A$ from $Z$)}} & & & \\
\cmidrule(lr){3-8}
\textbf{Dataset} & \textbf{Method} & \textbf{LR} & \textbf{GB} & \textbf{MLP-0} & \textbf{MLP-1} & \textbf{MLP-2} & \textbf{LR$-A$} & \textbf{Adv BA} & \textbf{Bound} & \textbf{Max $W$} \\
\midrule
Adult & ERM & 0.93 $\pm$ 0.00 & 0.93 $\pm$ 0.00 & 0.93 $\pm$ 0.00 & 0.93 $\pm$ 0.00 & 0.93 $\pm$ 0.00 & 0.51 $\pm$ 0.02 & 0.86 $\pm$ 0.00 & -- & 0.99 $\pm$ 0.00 \\
      & M   & 0.56 $\pm$ 0.00 & 0.57 $\pm$ 0.01 & 0.57 $\pm$ 0.01 & 0.57 $\pm$ 0.01 & 0.58 $\pm$ 0.01 & 0.51 $\pm$ 0.01 & 0.55 $\pm$ 0.01 & 0.55 $\pm$ 0.00 & 0.85 $\pm$ 0.02 \\
      & C   & 0.61 $\pm$ 0.01 & 0.62 $\pm$ 0.00 & 0.62 $\pm$ 0.01 & 0.62 $\pm$ 0.00 & 0.62 $\pm$ 0.00 & 0.51 $\pm$ 0.01 & 0.58 $\pm$ 0.01 & 0.53 $\pm$ 0.00 & 0.56 $\pm$ 0.00 \\
\addlinespace[3pt]
Dutch census & ERM & 0.75 $\pm$ 0.00 & 0.77 $\pm$ 0.00 & 0.77 $\pm$ 0.00 & 0.77 $\pm$ 0.00 & 0.77 $\pm$ 0.00 & 0.51 $\pm$ 0.04 & 0.68 $\pm$ 0.00 & -- & 0.77 $\pm$ 0.00 \\
             & M   & 0.51 $\pm$ 0.00 & 0.51 $\pm$ 0.00 & 0.51 $\pm$ 0.01 & 0.51 $\pm$ 0.00 & 0.51 $\pm$ 0.00 & 0.49 $\pm$ 0.00 & 0.50 $\pm$ 0.00 & 0.51 $\pm$ 0.00 & 0.79 $\pm$ 0.01 \\
             & C   & 0.67 $\pm$ 0.01 & 0.67 $\pm$ 0.01 & 0.67 $\pm$ 0.01 & 0.68 $\pm$ 0.01 & 0.68 $\pm$ 0.01 & 0.51 $\pm$ 0.01 & 0.60 $\pm$ 0.00 & 0.56 $\pm$ 0.00 & 0.56 $\pm$ 0.01 \\
\addlinespace[3pt]
Comm. \& Crime & ERM & 0.79 $\pm$ 0.00 & 0.76 $\pm$ 0.00 & 0.78 $\pm$ 0.00 & 0.77 $\pm$ 0.00 & 0.78 $\pm$ 0.00 & 0.58 $\pm$ 0.08 & 0.69 $\pm$ 0.01 & -- & 0.66 $\pm$ 0.00 \\
               & M   & 0.54 $\pm$ 0.03 & 0.51 $\pm$ 0.02 & 0.51 $\pm$ 0.02 & 0.54 $\pm$ 0.02 & 0.52 $\pm$ 0.03 & 0.52 $\pm$ 0.07 & 0.51 $\pm$ 0.00 & 0.54 $\pm$ 0.01 & 0.73 $\pm$ 0.03 \\
               & C   & 0.73 $\pm$ 0.03 & 0.72 $\pm$ 0.03 & 0.74 $\pm$ 0.02 & 0.73 $\pm$ 0.02 & 0.73 $\pm$ 0.02 & 0.54 $\pm$ 0.10 & 0.64 $\pm$ 0.01 & 0.55 $\pm$ 0.01 & 0.58 $\pm$ 0.04 \\
\addlinespace[3pt]
Taiwan credit & ERM & 0.61 $\pm$ 0.00 & 0.63 $\pm$ 0.00 & 0.63 $\pm$ 0.00 & 0.63 $\pm$ 0.00 & 0.63 $\pm$ 0.00 & 0.48 $\pm$ 0.05 & 0.58 $\pm$ 0.00 & -- & 0.62 $\pm$ 0.00 \\
              & M   & 0.51 $\pm$ 0.01 & 0.52 $\pm$ 0.01 & 0.51 $\pm$ 0.00 & 0.51 $\pm$ 0.01 & 0.51 $\pm$ 0.01 & 0.49 $\pm$ 0.01 & 0.50 $\pm$ 0.00 & 0.52 $\pm$ 0.00 & 0.69 $\pm$ 0.02 \\
              & C   & 0.53 $\pm$ 0.00 & 0.54 $\pm$ 0.01 & 0.54 $\pm$ 0.01 & 0.54 $\pm$ 0.00 & 0.54 $\pm$ 0.01 & 0.49 $\pm$ 0.02 & 0.50 $\pm$ 0.00 & 0.52 $\pm$ 0.00 & 0.53 $\pm$ 0.00 \\
\addlinespace[3pt]
COMPAS & ERM & 0.68 $\pm$ 0.00 & 0.67 $\pm$ 0.00 & 0.68 $\pm$ 0.00 & 0.68 $\pm$ 0.00 & 0.68 $\pm$ 0.01 & 0.53 $\pm$ 0.06 & 0.60 $\pm$ 0.00 & -- & 0.65 $\pm$ 0.00 \\
       & M   & 0.52 $\pm$ 0.01 & 0.56 $\pm$ 0.01 & 0.55 $\pm$ 0.03 & 0.56 $\pm$ 0.01 & 0.56 $\pm$ 0.01 & 0.50 $\pm$ 0.01 & 0.52 $\pm$ 0.01 & 0.56 $\pm$ 0.00 & 0.63 $\pm$ 0.02 \\
       & C   & 0.55 $\pm$ 0.01 & 0.56 $\pm$ 0.01 & 0.55 $\pm$ 0.01 & 0.55 $\pm$ 0.02 & 0.55 $\pm$ 0.02 & 0.52 $\pm$ 0.01 & 0.52 $\pm$ 0.01 & 0.60 $\pm$ 0.01 & 0.54 $\pm$ 0.01 \\
\bottomrule
\end{tabular}%
}
\end{table*}

\subsection{Ablation: Strength of representation invariance, hyperparameters and thresholds}
\label{app:ablation_hparams}

Refer to Fig.~\ref{fig:gamma} for sweeps $\gamma$ for the marginal alignment method. $D_A$, $\Delta_{\mathrm{DP}}$ and AUROC both decrease with $\gamma$; $\Delta_{\mathrm{EO}}$ do not. On Dutch it's par with $\mathrm{ERM}$ level for $\gamma\le0.1$ and rises to $0.31$ at $\gamma=1$ as $D_{A\mid Y}$ increases from $0.61$ to $0.79$; on Adult it is above $\mathrm{ERM}$ for all $\gamma$; on COMPAS it stays within seed std of $\mathrm{ERM}$. Refer to Table~\ref{tab:app-hparams}, which demonstrates that changing the feature set does not change the conclusion: $\Delta_{\mathrm{DP}}\le0.10$ in every config while $\Delta_{\mathrm{EO}}$ stays at par or above $\mathrm{ERM}$. Table~\ref{tab:app-threshold} compares the Youden threshold with a default $0.5$: at default, the marginal alignment method predicts few positives under low prevalence on Adult, Credit and COMPAS, so both gaps shrink counterintuitively; on Dutch and Crime it holds for both thresholds.

\begin{figure}[!htbp]
  \centering
  \includegraphics[width=\linewidth]{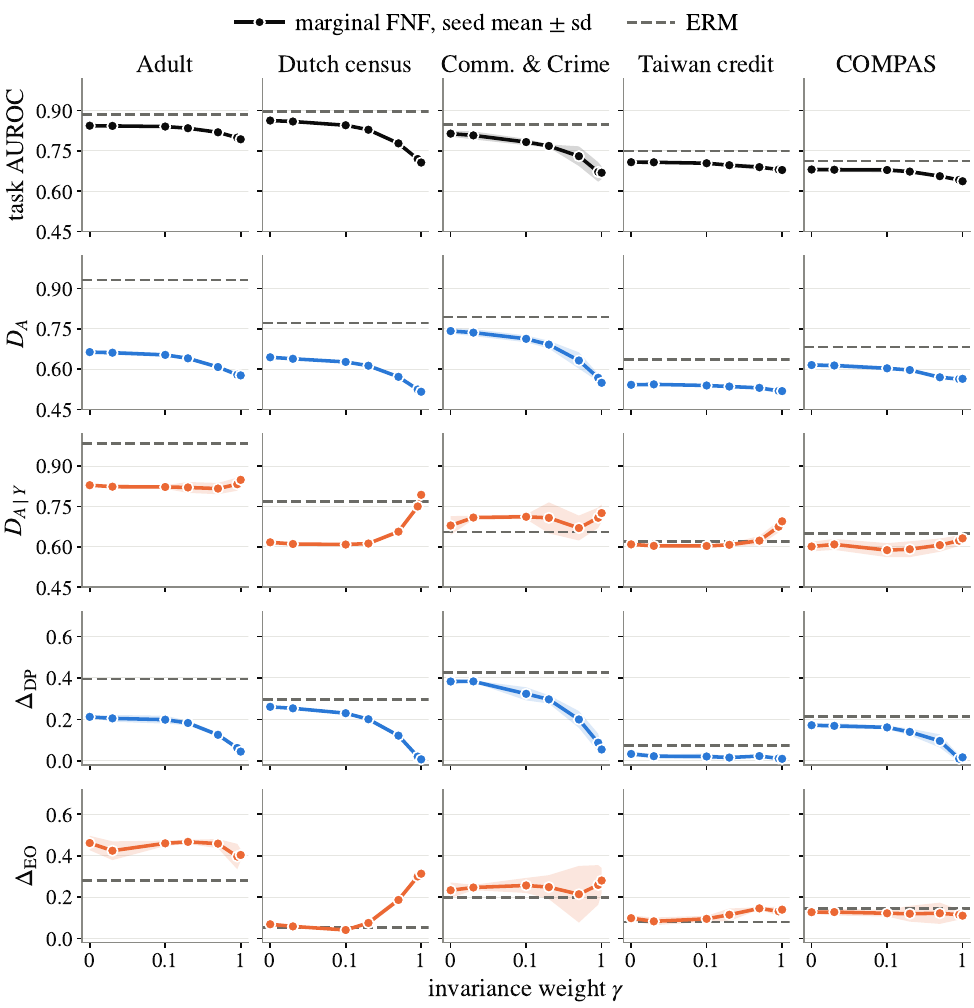}
  \caption{Marginal alignment method as a function of the invariance strength $\gamma$ (\textit{shift-50}, selected feature set, MADE width 50). Reported mean $\pm$ std over seeds.}
  \label{fig:gamma}
\end{figure}

\begin{table*}[!htbp]\centering\scriptsize\setlength{\tabcolsep}{3pt}
\caption{Marginal alignment method at $\gamma=1$ for every feature set and MADE width in the sweep (\textit{shift 50}, mean $\pm$ std over all seeds).}\label{tab:app-hparams}
\resizebox{\textwidth}{!}{\begin{tabular}{llrrrrrrrr}\toprule
Dataset & Feature set & MADE & n fits & AUROC & $D_A$ & $D_{A\mid Y}$ & $\Delta_{\mathrm{DP}}$ & $\Delta_{\mathrm{EO}}$ & exact TV \\\midrule
Adult & paper6 & 50 & 3 & 0.79 $\pm$ 0.00 & 0.58 $\pm$ 0.01 & 0.89 $\pm$ 0.01 & 0.08 $\pm$ 0.00 & 0.38 $\pm$ 0.03 & 0.103 $\pm$ 0.003 \\
Adult & paper6 & 100 & 3 & 0.79 $\pm$ 0.00 & 0.59 $\pm$ 0.01 & 0.87 $\pm$ 0.02 & 0.09 $\pm$ 0.01 & 0.36 $\pm$ 0.04 & 0.102 $\pm$ 0.005 \\
Adult & paper6\_age (selected) & 50 & 3 & 0.79 $\pm$ 0.00 & 0.58 $\pm$ 0.01 & 0.85 $\pm$ 0.02 & 0.04 $\pm$ 0.00 & 0.40 $\pm$ 0.03 & 0.096 $\pm$ 0.007 \\
Adult & paper6\_age & 100 & 3 & 0.80 $\pm$ 0.00 & 0.58 $\pm$ 0.00 & 0.84 $\pm$ 0.02 & 0.05 $\pm$ 0.01 & 0.38 $\pm$ 0.02 & 0.093 $\pm$ 0.024 \\
Dutch census & core6 & 50 & 3 & 0.72 $\pm$ 0.00 & 0.52 $\pm$ 0.01 & 0.81 $\pm$ 0.01 & 0.01 $\pm$ 0.01 & 0.27 $\pm$ 0.02 & 0.022 $\pm$ 0.003 \\
Dutch census & core6 & 100 & 3 & 0.72 $\pm$ 0.01 & 0.52 $\pm$ 0.00 & 0.81 $\pm$ 0.00 & 0.01 $\pm$ 0.01 & 0.30 $\pm$ 0.01 & 0.022 $\pm$ 0.003 \\
Dutch census & core6\_hh (selected) & 50 & 3 & 0.71 $\pm$ 0.00 & 0.52 $\pm$ 0.00 & 0.79 $\pm$ 0.01 & 0.01 $\pm$ 0.01 & 0.31 $\pm$ 0.03 & 0.023 $\pm$ 0.002 \\
Dutch census & core6\_hh & 100 & 3 & 0.71 $\pm$ 0.00 & 0.51 $\pm$ 0.01 & 0.78 $\pm$ 0.00 & 0.01 $\pm$ 0.01 & 0.32 $\pm$ 0.02 & 0.026 $\pm$ 0.004 \\
Comm. \& Crime & socio4 (selected) & 50 & 3 & 0.67 $\pm$ 0.02 & 0.55 $\pm$ 0.02 & 0.73 $\pm$ 0.03 & 0.05 $\pm$ 0.01 & 0.28 $\pm$ 0.06 & 0.070 $\pm$ 0.020 \\
Comm. \& Crime & socio4 & 100 & 3 & 0.64 $\pm$ 0.03 & 0.56 $\pm$ 0.03 & 0.76 $\pm$ 0.02 & 0.04 $\pm$ 0.03 & 0.28 $\pm$ 0.05 & 0.069 $\pm$ 0.015 \\
Comm. \& Crime & socio6 & 50 & 3 & 0.65 $\pm$ 0.02 & 0.56 $\pm$ 0.03 & 0.74 $\pm$ 0.01 & 0.09 $\pm$ 0.03 & 0.30 $\pm$ 0.09 & 0.085 $\pm$ 0.013 \\
Comm. \& Crime & socio6 & 100 & 3 & 0.68 $\pm$ 0.02 & 0.58 $\pm$ 0.02 & 0.73 $\pm$ 0.04 & 0.10 $\pm$ 0.07 & 0.22 $\pm$ 0.02 & 0.063 $\pm$ 0.011 \\
Taiwan credit & core5 (selected) & 50 & 3 & 0.68 $\pm$ 0.01 & 0.52 $\pm$ 0.01 & 0.69 $\pm$ 0.02 & 0.01 $\pm$ 0.00 & 0.14 $\pm$ 0.03 & 0.038 $\pm$ 0.005 \\
Taiwan credit & core5 & 100 & 3 & 0.69 $\pm$ 0.00 & 0.53 $\pm$ 0.01 & 0.69 $\pm$ 0.01 & 0.01 $\pm$ 0.01 & 0.12 $\pm$ 0.03 & 0.039 $\pm$ 0.005 \\
Taiwan credit & extended8 & 50 & 3 & 0.65 $\pm$ 0.01 & 0.53 $\pm$ 0.01 & 0.70 $\pm$ 0.01 & 0.01 $\pm$ 0.01 & 0.22 $\pm$ 0.04 & 0.022 $\pm$ 0.002 \\
Taiwan credit & extended8 & 100 & 3 & 0.65 $\pm$ 0.02 & 0.58 $\pm$ 0.01 & 0.71 $\pm$ 0.03 & 0.01 $\pm$ 0.01 & 0.25 $\pm$ 0.03 & 0.025 $\pm$ 0.005 \\
COMPAS & paper7 (selected) & 50 & 3 & 0.64 $\pm$ 0.00 & 0.56 $\pm$ 0.01 & 0.63 $\pm$ 0.02 & 0.02 $\pm$ 0.01 & 0.11 $\pm$ 0.02 & 0.125 $\pm$ 0.007 \\
COMPAS & paper7 & 100 & 3 & 0.63 $\pm$ 0.02 & 0.57 $\pm$ 0.01 & 0.64 $\pm$ 0.04 & 0.02 $\pm$ 0.03 & 0.10 $\pm$ 0.03 & 0.116 $\pm$ 0.016 \\
COMPAS & paper7\_nosex & 50 & 3 & 0.63 $\pm$ 0.00 & 0.60 $\pm$ 0.01 & 0.65 $\pm$ 0.04 & 0.04 $\pm$ 0.02 & 0.10 $\pm$ 0.02 & 0.143 $\pm$ 0.018 \\
COMPAS & paper7\_nosex & 100 & 3 & 0.63 $\pm$ 0.02 & 0.60 $\pm$ 0.01 & 0.65 $\pm$ 0.01 & 0.04 $\pm$ 0.03 & 0.07 $\pm$ 0.02 & 0.131 $\pm$ 0.007 \\
\bottomrule\end{tabular}}\end{table*}

\begin{table*}[!htbp]\centering\scriptsize\setlength{\tabcolsep}{3pt}
\caption{Optimal Youden threshold $\tau$ (for validation) vs. a default $0.5$ (\textit{shift 50}, reported mean $\pm$ std over all seeds).}\label{tab:app-threshold}
\resizebox{\textwidth}{!}{\begin{tabular}{llrrrrrrr}\toprule
Dataset & Method & Youden $\tau$ & Acc ($\tau$) & $\Delta_{\mathrm{DP}}$ ($\tau$) & $\Delta_{\mathrm{EO}}$ ($\tau$) & Acc (0.5) & $\Delta_{\mathrm{DP}}$ (0.5) & $\Delta_{\mathrm{EO}}$ (0.5) \\\midrule
Adult & ERM & 0.22 $\pm$ 0.02 & 0.78 $\pm$ 0.02 & 0.40 $\pm$ 0.03 & 0.28 $\pm$ 0.02 & 0.84 $\pm$ 0.00 & 0.24 $\pm$ 0.01 & 0.39 $\pm$ 0.03 \\
Adult & $Z\perp A$ & 0.17 $\pm$ 0.02 & 0.69 $\pm$ 0.02 & 0.04 $\pm$ 0.00 & 0.40 $\pm$ 0.03 & 0.78 $\pm$ 0.00 & 0.01 $\pm$ 0.01 & 0.06 $\pm$ 0.03 \\
Adult & $Z\perp A\mid Y$ & 0.17 $\pm$ 0.00 & 0.75 $\pm$ 0.02 & 0.17 $\pm$ 0.01 & 0.08 $\pm$ 0.04 & 0.82 $\pm$ 0.00 & 0.11 $\pm$ 0.01 & 0.06 $\pm$ 0.01 \\
Dutch census & ERM & 0.43 $\pm$ 0.02 & 0.82 $\pm$ 0.00 & 0.30 $\pm$ 0.00 & 0.05 $\pm$ 0.00 & 0.83 $\pm$ 0.00 & 0.29 $\pm$ 0.00 & 0.06 $\pm$ 0.00 \\
Dutch census & $Z\perp A$ & 0.40 $\pm$ 0.04 & 0.64 $\pm$ 0.00 & 0.01 $\pm$ 0.01 & 0.31 $\pm$ 0.03 & 0.65 $\pm$ 0.01 & 0.01 $\pm$ 0.01 & 0.22 $\pm$ 0.03 \\
Dutch census & $Z\perp A\mid Y$ & 0.41 $\pm$ 0.04 & 0.81 $\pm$ 0.00 & 0.27 $\pm$ 0.01 & 0.03 $\pm$ 0.01 & 0.81 $\pm$ 0.00 & 0.28 $\pm$ 0.01 & 0.02 $\pm$ 0.01 \\
Comm. \& Crime & ERM & 0.45 $\pm$ 0.03 & 0.77 $\pm$ 0.01 & 0.43 $\pm$ 0.02 & 0.20 $\pm$ 0.02 & 0.75 $\pm$ 0.01 & 0.39 $\pm$ 0.01 & 0.27 $\pm$ 0.00 \\
Comm. \& Crime & $Z\perp A$ & 0.44 $\pm$ 0.03 & 0.62 $\pm$ 0.02 & 0.05 $\pm$ 0.01 & 0.28 $\pm$ 0.06 & 0.61 $\pm$ 0.02 & 0.04 $\pm$ 0.03 & 0.19 $\pm$ 0.04 \\
Comm. \& Crime & $Z\perp A\mid Y$ & 0.50 $\pm$ 0.02 & 0.75 $\pm$ 0.03 & 0.40 $\pm$ 0.07 & 0.28 $\pm$ 0.16 & 0.76 $\pm$ 0.03 & 0.40 $\pm$ 0.07 & 0.25 $\pm$ 0.18 \\
Taiwan credit & ERM & 0.18 $\pm$ 0.02 & 0.76 $\pm$ 0.01 & 0.07 $\pm$ 0.00 & 0.08 $\pm$ 0.00 & 0.85 $\pm$ 0.00 & 0.05 $\pm$ 0.00 & 0.03 $\pm$ 0.00 \\
Taiwan credit & $Z\perp A$ & 0.20 $\pm$ 0.03 & 0.72 $\pm$ 0.03 & 0.01 $\pm$ 0.00 & 0.14 $\pm$ 0.03 & 0.82 $\pm$ 0.00 & 0.00 $\pm$ 0.00 & 0.02 $\pm$ 0.01 \\
Taiwan credit & $Z\perp A\mid Y$ & 0.17 $\pm$ 0.03 & 0.73 $\pm$ 0.05 & 0.04 $\pm$ 0.01 & 0.01 $\pm$ 0.01 & 0.85 $\pm$ 0.00 & 0.03 $\pm$ 0.01 & 0.02 $\pm$ 0.02 \\
COMPAS & ERM & 0.39 $\pm$ 0.04 & 0.64 $\pm$ 0.01 & 0.21 $\pm$ 0.01 & 0.15 $\pm$ 0.01 & 0.68 $\pm$ 0.00 & 0.20 $\pm$ 0.01 & 0.14 $\pm$ 0.01 \\
COMPAS & $Z\perp A$ & 0.41 $\pm$ 0.03 & 0.62 $\pm$ 0.01 & 0.02 $\pm$ 0.01 & 0.11 $\pm$ 0.02 & 0.62 $\pm$ 0.00 & 0.02 $\pm$ 0.02 & 0.05 $\pm$ 0.03 \\
COMPAS & $Z\perp A\mid Y$ & 0.39 $\pm$ 0.07 & 0.66 $\pm$ 0.01 & 0.10 $\pm$ 0.02 & 0.04 $\pm$ 0.02 & 0.65 $\pm$ 0.00 & 0.08 $\pm$ 0.00 & 0.05 $\pm$ 0.01 \\
\bottomrule\end{tabular}}\end{table*}

\subsection{Additional Details: Datasets and per-seed results}
\label{app:seeds}
Dataset details, group definitions, and training-split prevalences per condition are reported in Table~\ref{tab:app-datasets} for accountability.

\begin{table*}[!htbp]
\centering
\small
\caption{\textbf{Dataset statistics and prevalence shift configurations.} 
Binary groups ($A{=}0$ vs.\ $A{=}1$) and prediction targets ($Y{=}1$) are: 
\textit{Adult}: female vs.\ male, income $>$50K; 
\textit{Dutch}: female vs.\ male, high-level occupation; 
\textit{Crime}: Black population share below vs.\ at or above median, violent crime rate above median; 
\textit{Taiwan credit}: female vs.\ male, default next month; 
\textit{COMPAS}: Caucasian vs.\ African-American, two-year recidivism. 
Rightmost columns report the training-split base rates ($\pi_0$ / $\pi_1$) across synthetic shift regimes. 
Selected feature sets: Adult \texttt{paper6\_age} (workclass, education, marital status, occupation, relationship, race, binned age), Dutch \texttt{core6\_hh} (age, education, economic status, current activity, marital status, household position, household size), Crime \texttt{socio4} (median income, poverty, unemployment, two-parent families; quantile-binned), Credit \texttt{core5} (education, marriage, binned age, last repayment status, binned limit), COMPAS \texttt{paper7} (sex, age category, charge degree, binned priors, three juvenile counts).}
\label{tab:app-datasets}
\setlength{\tabcolsep}{4pt}
\begin{tabularx}{\textwidth}{@{} l X c cccc @{}}
\toprule
& & & \multicolumn{4}{c}{\textbf{Training Prevalences} ($\pi_0$ / $\pi_1$)} \\
\cmidrule(lr){4-7}
\textbf{Dataset} & \textbf{Samples} & \textbf{Natural} $\pi_0 / \pi_1$ & \textbf{Equalized} & \textbf{Natural} & \textbf{Shift-50} & \textbf{Shift-75} \\
\midrule
Adult income & 48,842 & 0.11 / 0.30 & 0.11 / 0.11 & 0.11 / 0.31 & 0.06 / 0.31 & 0.03 / 0.31 \\
             & \scriptsize (29,312 / 4,880 / 4,885 / 9,765) & & & & & \\
\addlinespace[3pt]
Dutch census & 60,420 & 0.33 / 0.63 & 0.33 / 0.33 & 0.33 / 0.62 & 0.20 / 0.62 & 0.11 / 0.62 \\
             & \scriptsize (36,252 / 6,042 / 6,042 / 12,084) & & & & & \\
\addlinespace[3pt]
Comm. \& Crime & 1,994 & 0.27 / 0.71 & 0.26 / 0.26 & 0.26 / 0.69 & 0.15 / 0.69 & 0.08 / 0.69 \\
             & \scriptsize (1,196 / 199 / 199 / 400) & & & & & \\
\addlinespace[3pt]
Taiwan credit & 30,000 & 0.21 / 0.24 & 0.21 / 0.21 & 0.21 / 0.25 & 0.12 / 0.25 & 0.06 / 0.25 \\
             & \scriptsize (18,000 / 3,000 / 3,000 / 6,000) & & & & & \\
\addlinespace[3pt]
COMPAS       & 5,278 & 0.39 / 0.52 & 0.39 / 0.39 & 0.39 / 0.54 & 0.24 / 0.54 & 0.14 / 0.54 \\
             & \scriptsize (3,166 / 528 / 528 / 1,056) & & & & & \\
\bottomrule
\end{tabularx}
\end{table*}

\begin{table*}[!htbp]
\centering
\scriptsize
\setlength{\tabcolsep}{3pt}
\caption{Detailed results for the \textit{shift-50} prevalence condition per selected seeds.}
\label{tab:app-per-seed}
\resizebox{\textwidth}{!}{
\begin{tabular}{lllrrrrrrr}
\toprule
Dataset & Method & Seed & AUROC & $D_A$ & $D_{A\mid Y}$ & $\Delta_{\mathrm{DP}}$ & $\Delta_{\mathrm{EO}}$ & $\Delta_{\mathrm{PPV}}$ & $\Delta_{\mathrm{bias}}$ \\
\midrule
Adult & ERM & 42 & 0.89 & 0.93 & 0.99 & 0.42 & 0.30 & 0.17 & 0.00 \\
Adult & ERM & 123 & 0.89 & 0.93 & 0.99 & 0.36 & 0.27 & 0.19 & 0.02 \\
Adult & ERM & 456 & 0.89 & 0.93 & 0.99 & 0.40 & 0.27 & 0.18 & 0.02 \\
Adult & M-FNF ($Z\perp A$) & 42 & 0.79 & 0.58 & 0.86 & 0.04 & 0.37 & 0.48 & 0.23 \\
Adult & M-FNF ($Z\perp A$) & 123 & 0.79 & 0.57 & 0.86 & 0.05 & 0.43 & 0.47 & 0.23 \\
Adult & M-FNF ($Z\perp A$) & 456 & 0.79 & 0.58 & 0.83 & 0.04 & 0.41 & 0.51 & 0.23 \\
Adult & C-FNF ($Z\perp A\mid Y$) & 42 & 0.84 & 0.66 & 0.83 & 0.21 & 0.41 & 0.47 & 0.11 \\
Adult & C-FNF ($Z\perp A\mid Y$) & 123 & 0.84 & 0.65 & 0.78 & 0.20 & 0.41 & 0.48 & 0.11 \\
Adult & C-FNF ($Z\perp A\mid Y$) & 456 & 0.84 & 0.66 & 0.81 & 0.21 & 0.48 & 0.50 & 0.11 \\
Dutch census & ERM & 42 & 0.90 & 0.77 & 0.76 & 0.30 & 0.05 & 0.26 & 0.22 \\
Dutch census & ERM & 123 & 0.90 & 0.77 & 0.77 & 0.30 & 0.05 & 0.29 & 0.22 \\
Dutch census & ERM & 456 & 0.90 & 0.77 & 0.77 & 0.30 & 0.05 & 0.27 & 0.21 \\
Dutch census & M-FNF ($Z\perp A$) & 42 & 0.71 & 0.51 & 0.79 & 0.01 & 0.28 & 0.63 & 0.44 \\
Dutch census & M-FNF ($Z\perp A$) & 123 & 0.70 & 0.51 & 0.80 & 0.00 & 0.32 & 0.63 & 0.44 \\
Dutch census & M-FNF ($Z\perp A$) & 456 & 0.71 & 0.52 & 0.79 & 0.01 & 0.33 & 0.63 & 0.45 \\
Dutch census & C-FNF ($Z\perp A\mid Y$) & 42 & 0.87 & 0.64 & 0.61 & 0.26 & 0.07 & 0.29 & 0.27 \\
Dutch census & C-FNF ($Z\perp A\mid Y$) & 123 & 0.86 & 0.64 & 0.61 & 0.26 & 0.06 & 0.29 & 0.27 \\
Dutch census & C-FNF ($Z\perp A\mid Y$) & 456 & 0.86 & 0.64 & 0.62 & 0.26 & 0.06 & 0.29 & 0.27 \\
Comm. \& Crime & ERM & 42 & 0.85 & 0.79 & 0.66 & 0.43 & 0.18 & 0.50 & 0.23 \\
Comm. \& Crime & ERM & 123 & 0.85 & 0.79 & 0.66 & 0.41 & 0.23 & 0.52 & 0.23 \\
Comm. \& Crime & ERM & 456 & 0.85 & 0.79 & 0.66 & 0.44 & 0.19 & 0.50 & 0.23 \\
Comm. \& Crime & M-FNF ($Z\perp A$) & 42 & 0.67 & 0.53 & 0.76 & 0.05 & 0.35 & 0.74 & 0.55 \\
Comm. \& Crime & M-FNF ($Z\perp A$) & 123 & 0.65 & 0.57 & 0.72 & 0.06 & 0.27 & 0.73 & 0.56 \\
Comm. \& Crime & M-FNF ($Z\perp A$) & 456 & 0.69 & 0.54 & 0.70 & 0.05 & 0.23 & 0.66 & 0.55 \\
Comm. \& Crime & C-FNF ($Z\perp A\mid Y$) & 42 & 0.81 & 0.71 & 0.74 & 0.36 & 0.18 & 0.54 & 0.32 \\
Comm. \& Crime & C-FNF ($Z\perp A\mid Y$) & 123 & 0.82 & 0.73 & 0.73 & 0.40 & 0.28 & 0.55 & 0.31 \\
Comm. \& Crime & C-FNF ($Z\perp A\mid Y$) & 456 & 0.82 & 0.74 & 0.56 & 0.39 & 0.31 & 0.56 & 0.30 \\
Taiwan credit & ERM & 42 & 0.75 & 0.63 & 0.62 & 0.07 & 0.09 & 0.20 & 0.08 \\
Taiwan credit & ERM & 123 & 0.75 & 0.63 & 0.62 & 0.08 & 0.08 & 0.19 & 0.08 \\
Taiwan credit & ERM & 456 & 0.75 & 0.64 & 0.62 & 0.07 & 0.08 & 0.20 & 0.08 \\
Taiwan credit & M-FNF ($Z\perp A$) & 42 & 0.67 & 0.52 & 0.67 & 0.01 & 0.16 & 0.29 & 0.11 \\
Taiwan credit & M-FNF ($Z\perp A$) & 123 & 0.68 & 0.52 & 0.72 & 0.01 & 0.11 & 0.24 & 0.11 \\
Taiwan credit & M-FNF ($Z\perp A$) & 456 & 0.69 & 0.51 & 0.69 & 0.01 & 0.15 & 0.31 & 0.11 \\
Taiwan credit & C-FNF ($Z\perp A\mid Y$) & 42 & 0.71 & 0.54 & 0.60 & 0.03 & 0.07 & 0.26 & 0.09 \\
Taiwan credit & C-FNF ($Z\perp A\mid Y$) & 123 & 0.70 & 0.54 & 0.62 & 0.04 & 0.08 & 0.25 & 0.09 \\
Taiwan credit & C-FNF ($Z\perp A\mid Y$) & 456 & 0.72 & 0.54 & 0.57 & 0.03 & 0.08 & 0.25 & 0.10 \\
COMPAS & ERM & 42 & 0.71 & 0.68 & 0.65 & 0.20 & 0.13 & 0.23 & 0.13 \\
COMPAS & ERM & 123 & 0.72 & 0.69 & 0.65 & 0.22 & 0.16 & 0.23 & 0.13 \\
COMPAS & ERM & 456 & 0.71 & 0.68 & 0.65 & 0.22 & 0.15 & 0.24 & 0.14 \\
COMPAS & M-FNF ($Z\perp A$) & 42 & 0.63 & 0.56 & 0.65 & 0.02 & 0.13 & 0.38 & 0.24 \\
COMPAS & M-FNF ($Z\perp A$) & 123 & 0.64 & 0.58 & 0.63 & 0.01 & 0.10 & 0.26 & 0.25 \\
COMPAS & M-FNF ($Z\perp A$) & 456 & 0.64 & 0.56 & 0.61 & 0.02 & 0.09 & 0.35 & 0.24 \\
COMPAS & C-FNF ($Z\perp A\mid Y$) & 42 & 0.66 & 0.61 & 0.62 & 0.17 & 0.12 & 0.20 & 0.18 \\
COMPAS & C-FNF ($Z\perp A\mid Y$) & 123 & 0.69 & 0.63 & 0.62 & 0.18 & 0.12 & 0.21 & 0.17 \\
COMPAS & C-FNF ($Z\perp A\mid Y$) & 456 & 0.68 & 0.63 & 0.62 & 0.18 & 0.12 & 0.21 & 0.17 \\
\bottomrule
\end{tabular}
}
\end{table*}

Table~\ref{tab:app-per-seed} reports every seed for the four methods at the \textit{shift-50} prevalence condition in Table~\ref{tab:fnf-results}. The training seed variation is small relative to the between-method differences on all datasets except Communities \& Crime (400 test rows, 38--151 positives in $A{=}0$)/ It is the only dataset where the training seed variation of $\Delta_{\mathrm{EO}}$ and $\Delta_{\mathrm{PPV}}$ is comparable to the method effect.

\clearpage

\subsection{Additional Details: Calibration by groups}
\label{app:calibration}

Table~\ref{tab:app-calibration} reports the calibration by groups results for the selected models under \textit{shift-50} prevalence condition. Under marginal alignment: the low-prevalence group $A{=}0$ becomes strongly 'over-confident' (Adult $b_0$: $0.96\to0.06$; Dutch $0.98\to0.51$) and the high-prevalence group 'under-confident' ($b_1=1.34$--$1.84$), the intercepts can be seen in opposite directions ($a_0<0<a_1$), and $\Delta_{\mathrm{cal}}$ increases for all datasets, as anticipated by $\mathrm{P}_2$. Similar dependency on training prevalence is also discussed by~\citet{kina2026prevalencecalibrationshortcutmitigation}~for calibration by groups: a model fitted to a $P(Y\mid A)$ carries that prevalence, and the mismatch may appears as group-specific miscalibration under $Q$. . Class-conditional alignment brings back both slopes to what they were for $\mathrm{ERM}$  but leaves the intercept gap at par with or above $\mathrm{ERM}$ ($\mathrm{P}_4$). Under prevalence shift at deployment (Fig.~\ref{fig:deployment}), $\Delta_{\mathrm{bias}}$ increases for every method, including at $\mathrm{ERM}$, because a score calibrated to the training prevalence $\pi_a$ is off by $\mathrm{logit}\,\pi^Q_a-\mathrm{logit}\,\pi_a$ in the large once the test prevalence is $\pi^Q_a$; when the shift is group-specific, this gap differs between groups. Per-run values for all $2{,}040$ evaluations are summarized in the table.

\begin{table*}[!htbp]
\centering
\scriptsize
\setlength{\tabcolsep}{2.5pt}
\caption{\textbf{Calibration by groups at \textit{shift 50}:} 15-bin ECE, Cox recalibration slope~\citep{cox1958two} $b_g$ and fixed-slope intercept $a_g$ (logits), signed mean bias $\bar y_g-\bar p_g$, and the two between-group gaps $\Delta_{\mathrm{bias}}$ and $\Delta_{\mathrm{cal}}$. Methods: ERM, \textbf{M} (marginal alignment, $Z\perp A$), and \textbf{C} (class-conditional alignment, $Z\perp A\mid Y$). Datasets: Adult, Dutch census (Dutch), Communities \& Crime (Crime), Taiwan credit (Credit), and COMPAS. Reported as mean $\pm$ std over all three seeds.}
\label{tab:app-calibration}
\resizebox{\textwidth}{!}{%
\begin{tabular}{ll rrrrrrrrrr}
\toprule
\textbf{Dataset} & \textbf{Method} & $\mathbf{ECE}_0$ & $\mathbf{ECE}_1$ & $\boldsymbol{b}_0$ & $\boldsymbol{b}_1$ & $\boldsymbol{a}_0$ & $\boldsymbol{a}_1$ & $\boldsymbol{\bar y_0-\bar p_0}$ & $\boldsymbol{\bar y_1-\bar p_1}$ & $\boldsymbol{\Delta_{\mathrm{bias}}}$ & $\boldsymbol{\Delta_{\mathrm{cal}}}$ \\
\midrule
\multirow{3}{*}{Adult}
      & ERM & 0.02 $\pm$ 0.00 & 0.02 $\pm$ 0.00 & 0.96 $\pm$ 0.03 & 0.85 $\pm$ 0.01 & $-0.29$ $\pm$ 0.08 & 0.03 $\pm$ 0.12 & $-0.013$ $\pm$ 0.004 & 0.004 $\pm$ 0.016 & 0.02 $\pm$ 0.01 & 0.32 $\pm$ 0.06 \\
      & M   & 0.13 $\pm$ 0.01 & 0.11 $\pm$ 0.01 & 0.06 $\pm$ 0.04 & 1.34 $\pm$ 0.05 & $-1.36$ $\pm$ 0.06 & 0.75 $\pm$ 0.06 & $-0.116$ $\pm$ 0.007 & 0.112 $\pm$ 0.009 & 0.23 $\pm$ 0.00 & 2.12 $\pm$ 0.01 \\
      & C   & 0.08 $\pm$ 0.01 & 0.08 $\pm$ 0.01 & 0.79 $\pm$ 0.03 & 0.88 $\pm$ 0.02 & $-1.20$ $\pm$ 0.14 & 0.60 $\pm$ 0.09 & $-0.077$ $\pm$ 0.013 & 0.075 $\pm$ 0.010 & 0.15 $\pm$ 0.00 & 1.80 $\pm$ 0.07 \\
\midrule
\multirow{3}{*}{Dutch}
      & ERM & 0.12 $\pm$ 0.00 & 0.10 $\pm$ 0.00 & 1.01 $\pm$ 0.02 & 0.90 $\pm$ 0.01 & $-1.24$ $\pm$ 0.05 & 0.75 $\pm$ 0.02 & $-0.120$ $\pm$ 0.004 & 0.095 $\pm$ 0.003 & 0.22 $\pm$ 0.00 & 2.00 $\pm$ 0.03 \\
      & M   & 0.23 $\pm$ 0.00 & 0.22 $\pm$ 0.01 & 0.51 $\pm$ 0.08 & 1.65 $\pm$ 0.06 & $-1.25$ $\pm$ 0.01 & 1.07 $\pm$ 0.04 & $-0.224$ $\pm$ 0.003 & 0.219 $\pm$ 0.007 & 0.44 $\pm$ 0.00 & 2.32 $\pm$ 0.03 \\
      & C   & 0.12 $\pm$ 0.01 & 0.11 $\pm$ 0.00 & 0.98 $\pm$ 0.03 & 0.91 $\pm$ 0.02 & $-1.16$ $\pm$ 0.05 & 0.89 $\pm$ 0.04 & $-0.121$ $\pm$ 0.006 & 0.114 $\pm$ 0.005 & 0.23 $\pm$ 0.01 & 2.05 $\pm$ 0.07 \\
\midrule
\multirow{3}{*}{Crime}
      & ERM & 0.18 $\pm$ 0.02 & 0.12 $\pm$ 0.00 & 0.45 $\pm$ 0.01 & 0.74 $\pm$ 0.01 & $-1.28$ $\pm$ 0.04 & 0.82 $\pm$ 0.01 & $-0.126$ $\pm$ 0.006 & 0.104 $\pm$ 0.002 & 0.23 $\pm$ 0.00 & 2.10 $\pm$ 0.03 \\
      & M   & 0.28 $\pm$ 0.02 & 0.29 $\pm$ 0.00 & 0.43 $\pm$ 0.53 & 1.48 $\pm$ 0.18 & $-1.46$ $\pm$ 0.01 & 1.41 $\pm$ 0.02 & $-0.263$ $\pm$ 0.003 & 0.290 $\pm$ 0.005 & 0.55 $\pm$ 0.01 & 2.87 $\pm$ 0.03 \\
      & C   & 0.17 $\pm$ 0.04 & 0.16 $\pm$ 0.01 & 0.71 $\pm$ 0.26 & 0.74 $\pm$ 0.01 & $-1.21$ $\pm$ 0.24 & 1.13 $\pm$ 0.01 & $-0.137$ $\pm$ 0.026 & 0.153 $\pm$ 0.007 & 0.29 $\pm$ 0.03 & 2.34 $\pm$ 0.24 \\
\midrule
\multirow{3}{*}{Credit}
      & ERM & 0.03 $\pm$ 0.00 & 0.05 $\pm$ 0.00 & 0.96 $\pm$ 0.01 & 0.95 $\pm$ 0.00 & $-0.34$ $\pm$ 0.02 & 0.35 $\pm$ 0.02 & $-0.034$ $\pm$ 0.003 & 0.046 $\pm$ 0.002 & 0.08 $\pm$ 0.00 & 0.68 $\pm$ 0.01 \\
      & M   & 0.06 $\pm$ 0.00 & 0.05 $\pm$ 0.00 & 0.50 $\pm$ 0.05 & 1.22 $\pm$ 0.01 & $-0.50$ $\pm$ 0.02 & 0.36 $\pm$ 0.02 & $-0.058$ $\pm$ 0.003 & 0.053 $\pm$ 0.003 & 0.11 $\pm$ 0.00 & 0.86 $\pm$ 0.02 \\
      & C   & 0.05 $\pm$ 0.00 & 0.05 $\pm$ 0.00 & 0.94 $\pm$ 0.02 & 0.89 $\pm$ 0.01 & $-0.44$ $\pm$ 0.02 & 0.35 $\pm$ 0.01 & $-0.045$ $\pm$ 0.002 & 0.046 $\pm$ 0.002 & 0.09 $\pm$ 0.00 & 0.79 $\pm$ 0.02 \\
\midrule
\multirow{3}{*}{COMPAS}
      & ERM & 0.13 $\pm$ 0.00 & 0.06 $\pm$ 0.01 & 0.63 $\pm$ 0.01 & 0.75 $\pm$ 0.03 & $-0.69$ $\pm$ 0.02 & 0.09 $\pm$ 0.02 & $-0.116$ $\pm$ 0.003 & 0.018 $\pm$ 0.004 & 0.13 $\pm$ 0.00 & 0.78 $\pm$ 0.01 \\
      & M   & 0.18 $\pm$ 0.01 & 0.10 $\pm$ 0.01 & 0.52 $\pm$ 0.14 & 0.86 $\pm$ 0.07 & $-0.81$ $\pm$ 0.02 & 0.38 $\pm$ 0.01 & $-0.157$ $\pm$ 0.004 & 0.086 $\pm$ 0.002 & 0.24 $\pm$ 0.00 & 1.19 $\pm$ 0.01 \\
      & C   & 0.13 $\pm$ 0.01 & 0.11 $\pm$ 0.01 & 0.78 $\pm$ 0.01 & 0.62 $\pm$ 0.02 & $-0.72$ $\pm$ 0.07 & 0.37 $\pm$ 0.04 & $-0.125$ $\pm$ 0.013 & 0.073 $\pm$ 0.006 & 0.20 $\pm$ 0.01 & 1.09 $\pm$ 0.05 \\
\bottomrule
\end{tabular}%
}
\end{table*}

\subsection{Additional Details: Score-level alignment}
\label{app:score-level}

We constrain the score distribution, whereas alignment methods constrain the representation. We look for (i) whether aligning the score $R$ directly, rather than $Z$, changes the outcome, and (ii) whether a weaker constraint that matches only the class-conditional \textit{mean} of the score suffices for separation. We evaluate this on Dutch at \textit{shift-50} with the ERM model (Table~\ref{tab:app-score-level}).

\paragraph{Score-level alignment.} It reproduce the representation-level results: marginal alignment gives $\Delta_{\mathrm{DP}}=0.006$ with $\Delta_{\mathrm{EO}}=0.26$ and $\Delta_{\mathrm{bias}}=0.44$ (FNF: $0.007$, $0.31$, $0.44$), and within-class alignment gives $\Delta_{\mathrm{EO}}=0.024$ (C-FNF: $0.028$) with $\Delta_{\mathrm{DP}}$ and $\Delta_{\mathrm{bias}}$ at the ERM level. Aligning one scalar costs less utility than aligning $Z$ (AUROC $0.81$ versus $0.71$ for the marginal target), but the fairness conclusions remain unchanged, as expected for statements about $p(r\mid a)$ and $p(r\mid y,a)$.

\paragraph{Conditional mean matching.} We retrain the ERM head with the penalty $\lambda\sum_{y}\big(\mathbb{E}[R\mid Y{=}y,A{=}0]-\mathbb{E}[R\mid Y{=}y,A{=}1]\big)^2$ estimated per mini-batch, $\lambda\in\{1,10,100\}$. The class-conditional mean gaps fall from $+0.041$ ($y{=}0$) and $-0.028$ ($y{=}1$) under ERM to $0.000$ on train and test for $\lambda\ge100$ (also at $\lambda=10^3,10^4$). It does not deliver separation: $\Delta_{\mathrm{EO}}$ stays at $0.06$--$0.08$ (ERM $0.055$) and $D_{A\mid Y}$ does not fall, because the group-specific shape of $p(r\mid y,a)$ beyond its mean is untouched. It does move calibration: $\Delta_{\mathrm{bias}}$ rises to $0.41$.

\begin{table*}[t]
\centering
\small
\setlength{\tabcolsep}{0pt}
\renewcommand{\arraystretch}{1.15}
\caption{Score-level alignment and class-conditional mean matching on Dutch, \textit{shift-50}, test split, mean over three seeds.}
\label{tab:app-score-level}
\begin{tabular*}{\textwidth}{@{\extracolsep{\fill}} l ccccccccc @{}}
\toprule
\textbf{Method} 
& \textbf{AUROC} 
& $\boldsymbol{D_A}$ 
& $\boldsymbol{D_{A\mid Y}}$ 
& $\boldsymbol{\Delta_{\mathrm{DP}}}$ 
& $\boldsymbol{\Delta_{\mathrm{EO}}}$ 
& $\boldsymbol{\Delta_{\mathrm{PPV}}}$ 
& $\boldsymbol{\Delta_{\mathrm{NPV}}}$ 
& $\boldsymbol{\Delta_{\mathrm{bias}}}$ 
& $\boldsymbol{\Delta_{\mathrm{cal}}}$ \\
\midrule
ERM                                                       & 0.896 & 0.70       & 0.64       & 0.296 & 0.055 & 0.27 & 0.28 & 0.22 & 2.00 \\
Score $R\perp A$                  & 0.813 & 0.54       & 0.77       & 0.006 & 0.258 & 0.55 & 0.35 & 0.44 & 3.68 \\
Score $R\perp A\mid Y$ (within-class map)                 & 0.881 & 0.67       & 0.59       & 0.257 & 0.024 & 0.36 & 0.28 & 0.24 & 2.11 \\
Mean matching, $\lambda=1$                                & 0.896 & 0.70       & 0.63       & 0.291 & 0.051 & 0.29 & 0.28 & 0.22 & 1.95 \\
Mean matching, $\lambda=10$                               & 0.889 & 0.69       & 0.66       & 0.270 & 0.082 & 0.27 & 0.30 & 0.28 & 1.79 \\
Mean matching, $\lambda=100$                              & 0.883 & 0.69       & 0.66       & 0.266 & 0.081 & 0.30 & 0.30 & 0.41 & 1.87 \\
\midrule
FNF $Z\perp A$ ($\gamma=1$), repr. level         & 0.707 & 0.52$^{*}$ & 0.79$^{*}$ & 0.007 & 0.313 & 0.63 & 0.21 & 0.44 & 2.32 \\
C-FNF $Z\perp A\mid Y$ ($\gamma=1$), repr. level & 0.888 & 0.68$^{*}$ & 0.56$^{*}$ & 0.275 & 0.028 & 0.34 & 0.27 & 0.23 & 2.05 \\
\bottomrule
\end{tabular*}
\end{table*}

\clearpage

\section{Appendix 2: CheXpert}
\label{app:chexpert}

\subsection{Additional Details: Dataset details and cohorts}

We perform a classification task on Edema and Cardiomegaly from frontal CheXpert radiographs using recorded sex ($A=0$: female; $A=1$: male). Images with unknown sex are excluded;
Uncertain pathology labels are positive, and blank labels are negative. The \emph{shift50} intervention removes half of female-positive images separately within each partition. Both training and evaluation use the shifted distribution; this experiment does not test a new deployment shift. Table~\ref{tab:cx-cohort} reports the retained cohorts.

\begin{table}[!htbp]
\centering
\small
\setlength{\tabcolsep}{3pt}
\caption{\textit{shift50} cohorts. Patient partitions are fixed across seeds; image-level deletion changes the retained patient counts.}
\label{tab:cx-cohort}
\begin{tabular}{@{}llrrrr@{}}
\toprule
Task & Partition & Images & Patients & $\pi_0$ & $\pi_1$ \\
\midrule
Edema & Train & 107,688 & 37,580 & 0.198 & 0.313 \\
 & Checkpoint & 9,050 & 3,134 & 0.195 & 0.325 \\
 & Development & 8,872 & 3,138 & 0.202 & 0.333 \\
 & Probe val. & 17,186 & 6,271 & 0.206 & 0.310 \\
 & Test & 35,078 & 12,551 & 0.205 & 0.309 \\ \midrule
Cardiomegaly & Train & 111,980 & 38,082 & 0.081 & 0.159 \\
 & Checkpoint & 9,363 & 3,175 & 0.093 & 0.161 \\
 & Development & 9,225 & 3,168 & 0.082 & 0.173 \\
 & Probe val. & 17,896 & 6,351 & 0.085 & 0.168 \\
 & Test & 36,525 & 12,693 & 0.085 & 0.160 \\
\bottomrule
\end{tabular}
\end{table}

Similar to the tabular setting, we use three training seeds (42, 123, 456). ImageNet-init DenseNet121 models with $256\times256$ RGB input and standard normalization. We use BCE, AdamW (lr $10^{-4}$, decay of $10^{-5}$), a batch size of 32, and at most 30 epochs with early stopping. It uses a patience of 7 epochs after the first 10 epochs and minimizes validation BCE.

We adapt continuous Fair Normalizing Flows~\citep{balunović2022fairnormalizingflows}. Marginal training uses the analogous marginal KL divided by $d$. We use $\gamma\in\{.05,.50,.95\}$, 60 epochs, batch size 256, AdamW (learning rate $.001$, weight decay $.0001$), gradient clipping at five, and a ten-epoch ramp. Each KL direction uses 128 samples from its fitted prior. Checkpoint selection (passing criteria) minimizes validation MMD (worst-label conditional MMD
for conditional FNF), subject to retaining at least 90\% of ERM validation AUROC. All 36 FNF fits completed without fallback; all fitted GMMs converged. The existing MMD and adversarial-adapter experiments are ancillary comparisons; the marginal-versus-conditional analysis here holds the FNF architecture fixed. Results are detailed in~Tables~\ref{tab:cx-fairness} and \ref{tab:cx-utility}.

\paragraph{Evaluation.} 
The frozen representation $Z$ and scalar score $R$ are audited separately,
both marginally and within each outcome class. Probes use up to 40,000 training examples; the RBF-SVM uses at most 2,000 per sex. Hyperparameters and MLP early stopping use only probe-validation
balanced BCE. Decision metrics use a common, sex-independent threshold selected by validation Youden's index. Calibration uses 15 equal-width ECE bins and logistic recalibration.

\begin{figure}[!htbp]
  \centering
  \begin{minipage}[c]{0.58\linewidth}
    \centering
    \includegraphics[width=\linewidth]{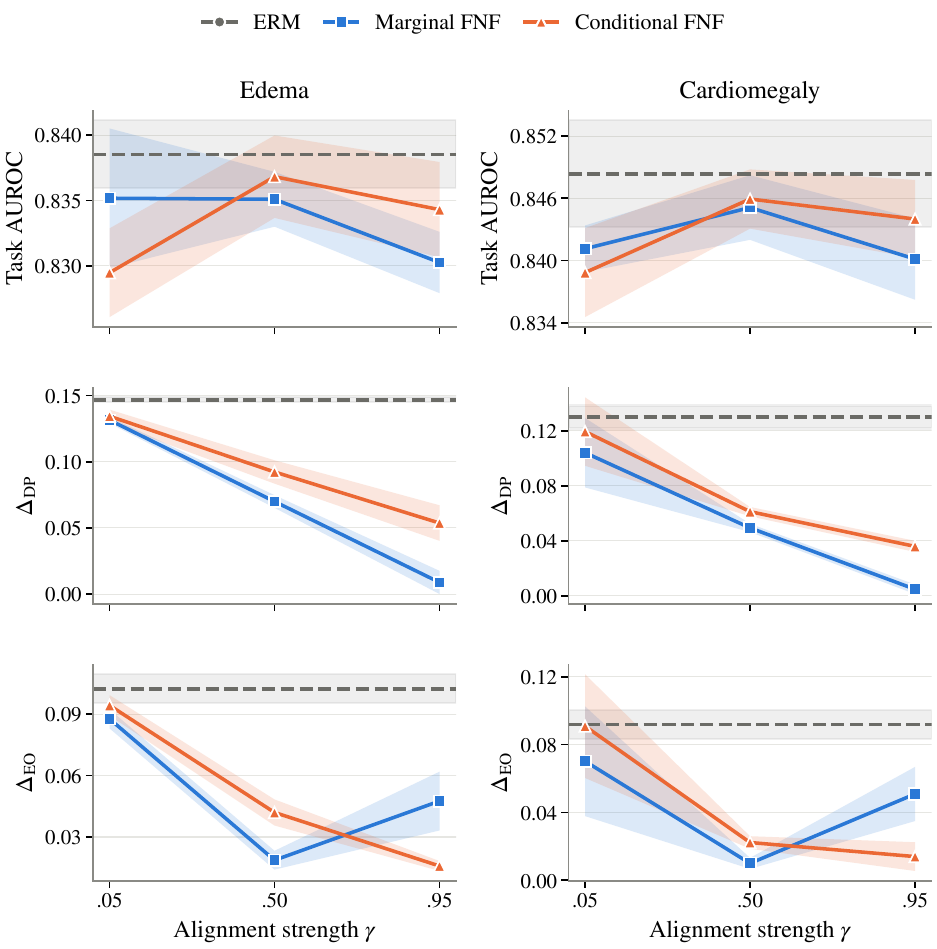}
  \end{minipage}%
  \hfill
  \begin{minipage}[c]{0.38\linewidth}
    \small
    \caption{\textbf{Invariance strength changes the fairness trade-off.} Lines and bands show means $\pm$ std. Conditional alignment lowers $\mathrm{EO}$ gaps at strong alignment, while marginal alignment produces smaller $\mathrm{DP}$ gaps.}
    \label{fig:cx-gamma}
  \end{minipage}
\end{figure}

\begin{table}[!htbp]
\centering
\small
\setlength{\tabcolsep}{3pt}
\caption{\textbf{Fairness and calibration for every invariance method strength} (mean $\pm$ seed reported). $\Delta_{\rm bias}$ is the absolute between-sex difference in mean residual $Y-R$; $\Delta_{\rm cal}$ is the absolute difference in fixed-slope calibration intercepts.}
\label{tab:cx-fairness}
\begin{tabular}{@{}llrrrrr@{}}
\toprule
Task & Method & $\Delta_{\rm DP}$ & $\Delta_{\rm EO}$ & Max ECE & $\Delta_{\rm bias}$ & $\Delta_{\rm cal}$ \\
\midrule
Edema & ERM & $0.147\pm0.002$ & $0.102\pm0.007$ & $0.020\pm0.009$ & $0.013\pm0.006$ & $0.092\pm0.037$ \\
 & M .05 & $0.132\pm0.002$ & $0.088\pm0.005$ & $0.037\pm0.003$ & $0.018\pm0.005$ & $0.160\pm0.043$ \\
 & M .50 & $0.070\pm0.005$ & $0.019\pm0.005$ & $0.037\pm0.002$ & $0.062\pm0.004$ & $0.504\pm0.012$ \\
 & M .95 & $0.009\pm0.009$ & $0.048\pm0.014$ & $0.069\pm0.006$ & $0.101\pm0.006$ & $0.797\pm0.071$ \\
 & C .05 & $0.134\pm0.005$ & $0.094\pm0.005$ & $0.034\pm0.009$ & $0.020\pm0.009$ & $0.165\pm0.064$ \\
 & C .50 & $0.092\pm0.009$ & $0.042\pm0.006$ & $0.039\pm0.004$ & $0.046\pm0.008$ & $0.387\pm0.055$ \\
 & C .95 & $0.054\pm0.013$ & $0.016\pm0.002$ & $0.051\pm0.006$ & $0.072\pm0.007$ & $0.582\pm0.029$ \\
Cardiomegaly & ERM & $0.130\pm0.008$ & $0.092\pm0.009$ & $0.022\pm0.003$ & $0.013\pm0.008$ & $0.188\pm0.102$ \\
 & M .05 & $0.104\pm0.025$ & $0.070\pm0.032$ & $0.030\pm0.006$ & $0.019\pm0.006$ & $0.266\pm0.087$ \\
 & M .50 & $0.049\pm0.003$ & $0.010\pm0.003$ & $0.031\pm0.006$ & $0.044\pm0.002$ & $0.635\pm0.032$ \\
 & M .95 & $0.005\pm0.003$ & $0.051\pm0.016$ & $0.050\pm0.002$ & $0.070\pm0.001$ & $0.962\pm0.018$ \\
 & C .05 & $0.119\pm0.025$ & $0.091\pm0.031$ & $0.028\pm0.008$ & $0.016\pm0.008$ & $0.246\pm0.110$ \\
 & C .50 & $0.061\pm0.003$ & $0.022\pm0.004$ & $0.027\pm0.001$ & $0.041\pm0.002$ & $0.592\pm0.018$ \\
 & C .95 & $0.036\pm0.004$ & $0.014\pm0.008$ & $0.035\pm0.001$ & $0.052\pm0.001$ & $0.737\pm0.008$ \\
\bottomrule
\end{tabular}
\end{table}

\begin{table}[!htbp]
\centering
\small
\setlength{\tabcolsep}{3pt}
\caption{\textbf{All settings:} utility and strongest demographic probe AUROC (mean $\pm$ sample std over all training seeds). M/C denote marginal/conditional objectives, followed by $\gamma$.}
\label{tab:cx-utility}
\begin{tabular}{@{}llrrrrr@{}}
\toprule
Task & Method & AUROC & AUPRC & $D_A(Z)$ & $D_{A\mid Y}(Z)$ & $D_{A\mid Y}(R)$ \\
\midrule
Edema & ERM & $0.839\pm0.003$ & $0.646\pm0.007$ & $0.978\pm0.002$ & $0.978\pm0.002$ & $0.613\pm0.019$ \\
 & M .05 & $0.835\pm0.005$ & $0.641\pm0.011$ & $0.693\pm0.022$ & $0.678\pm0.022$ & $0.583\pm0.011$ \\
 & M .50 & $0.835\pm0.002$ & $0.640\pm0.006$ & $0.631\pm0.034$ & $0.624\pm0.029$ & $0.519\pm0.008$ \\
 & M .95 & $0.830\pm0.002$ & $0.630\pm0.005$ & $0.623\pm0.042$ & $0.629\pm0.043$ & $0.555\pm0.008$ \\
 & C .05 & $0.829\pm0.003$ & $0.632\pm0.008$ & $0.675\pm0.013$ & $0.664\pm0.024$ & $0.582\pm0.010$ \\
 & C .50 & $0.837\pm0.003$ & $0.642\pm0.007$ & $0.613\pm0.004$ & $0.599\pm0.013$ & $0.537\pm0.010$ \\
 & C .95 & $0.834\pm0.004$ & $0.636\pm0.008$ & $0.596\pm0.016$ & $0.590\pm0.023$ & $0.513\pm0.003$ \\
Cardiomegaly & ERM & $0.848\pm0.005$ & $0.531\pm0.010$ & $0.978\pm0.003$ & $0.978\pm0.002$ & $0.602\pm0.009$ \\
 & M .05 & $0.841\pm0.002$ & $0.514\pm0.009$ & $0.703\pm0.019$ & $0.696\pm0.023$ & $0.607\pm0.016$ \\
 & M .50 & $0.845\pm0.003$ & $0.524\pm0.006$ & $0.650\pm0.021$ & $0.636\pm0.025$ & $0.551\pm0.011$ \\
 & M .95 & $0.840\pm0.004$ & $0.507\pm0.011$ & $0.644\pm0.028$ & $0.646\pm0.023$ & $0.533\pm0.008$ \\
 & C .05 & $0.839\pm0.004$ & $0.507\pm0.014$ & $0.673\pm0.012$ & $0.663\pm0.012$ & $0.592\pm0.019$ \\
 & C .50 & $0.846\pm0.003$ & $0.526\pm0.008$ & $0.630\pm0.005$ & $0.623\pm0.004$ & $0.537\pm0.002$ \\
 & C .95 & $0.844\pm0.004$ & $0.521\pm0.008$ & $0.617\pm0.002$ & $0.616\pm0.004$ & $0.515\pm0.001$ \\
\bottomrule
\end{tabular}
\end{table}

\subsection{Additional Details: Probe Results}

Strong conditional alignment did reduce conditional representation probe AUROC from $\approx$ $.978$ under $\mathrm{ERM}$ to $.590$ (Edema) and $.616$ (Cardiomegaly), but none of the 18 conditional representations pass the marginal or conditional near-chance $\mathrm{AUROC}$ criterion (which is, in fact is unrealistic for medical images. Figure~\ref{fig:cx-probes} shows why a single probe family is insufficient. Score-level results are stronger: conditional FNF at $\gamma=.95$ passes all three test criteria for Edema in 3/3 seeds. Neither task passes the conditional $\mathrm{AUROC}$ criterion (Table~\ref{tab:cx-gate}), but are sufficient to support our theoretical findings.

\begin{figure*}[!htbp]
  \centering
  \begin{minipage}[c]{0.60\textwidth}
    \centering
    \includegraphics[width=\linewidth]{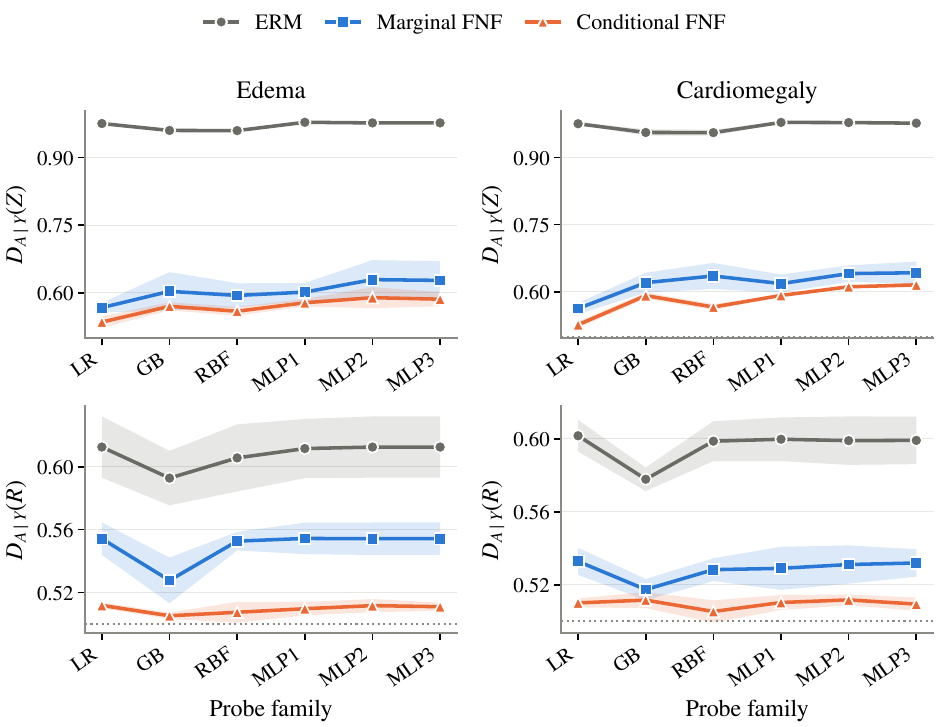}
  \end{minipage}%
  \hfill
  \begin{minipage}[c]{0.37\textwidth}
    \caption{\textbf{Conditional probe-family audit at $\gamma=.95$.} Top: representations; bottom: scores.}
    \label{fig:cx-probes}
  \end{minipage}
\end{figure*}

\begin{table}[!htbp]
\centering
\small
\setlength{\tabcolsep}{3pt}
\caption{\textbf{Conditional score audit.} Upper bounds use the within-scope Bonferroni correction; passing upper bound $\leq 0.55$. All three Edema test runs also pass the balanced-accuracy and log-loss interval checks.}
\label{tab:cx-gate}
\begin{tabular}{@{}llrrrr@{}}
\toprule
Task & Seed & Test $D_{A\mid Y}(R)$ & Test upper & Dev. upper & Test pass \\
\midrule
Edema & 42 & 0.515 & 0.540 & 0.591 & Yes \\
 & 123 & 0.510 & 0.541 & 0.563 & Yes \\
 & 456 & 0.514 & 0.545 & 0.572 & Yes \\
Cardiomegaly & 42 & 0.514 & 0.557 & 0.637 & No \\
 & 123 & 0.516 & 0.559 & 0.618 & No \\
 & 456 & 0.515 & 0.558 & 0.606 & No \\
\bottomrule
\end{tabular}
\end{table}

\subsection{Additional Details: Calibration and operating-point analysis}

At $\gamma=.95$, conditional FNF improves mean $\mathrm{AUROC}$, $\mathrm{EO}$, and worst-group $\mathrm{ECE}$ relative to marginal FNF on both tasks, but increases $\mathrm{DP}$ gaps. Detailed in Table~\ref{tab:cx-calibration}; Fig.~\ref{fig:cx-calibration}. The $\mathrm{EO}$ comparison depends on the operating point (Fig.~\ref{fig:cx-threshold}). All results support our propositions, though extreme thresholds can shrink gaps by
predicting the same label for everyone. Stronger alignment is therefore not preferable, and reduced gaps alone do not establish improved predictions.

\begin{figure*}[!htbp]
  \centering
  \begin{minipage}[!htbp]{0.48\textwidth}
    \vspace{0pt} 
    \centering
    \includegraphics[width=\linewidth]{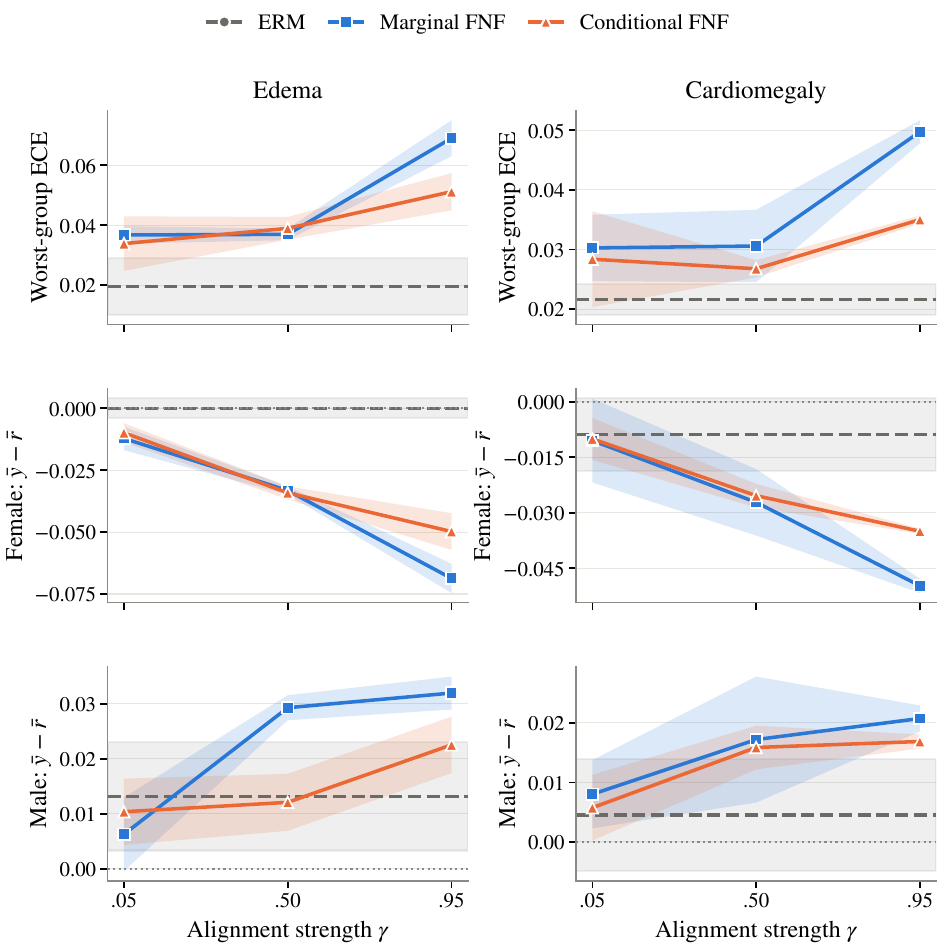}
    \caption{\textbf{Calibration costs vs. improved $\mathrm{EO}$.} Means $\pm$ std across strengths.}
    \label{fig:cx-calibration}
  \end{minipage}%
  \hfill
  \begin{minipage}[!htbp]{0.48\textwidth}
    \vspace{0pt} 
    \centering
    \includegraphics[width=\linewidth]{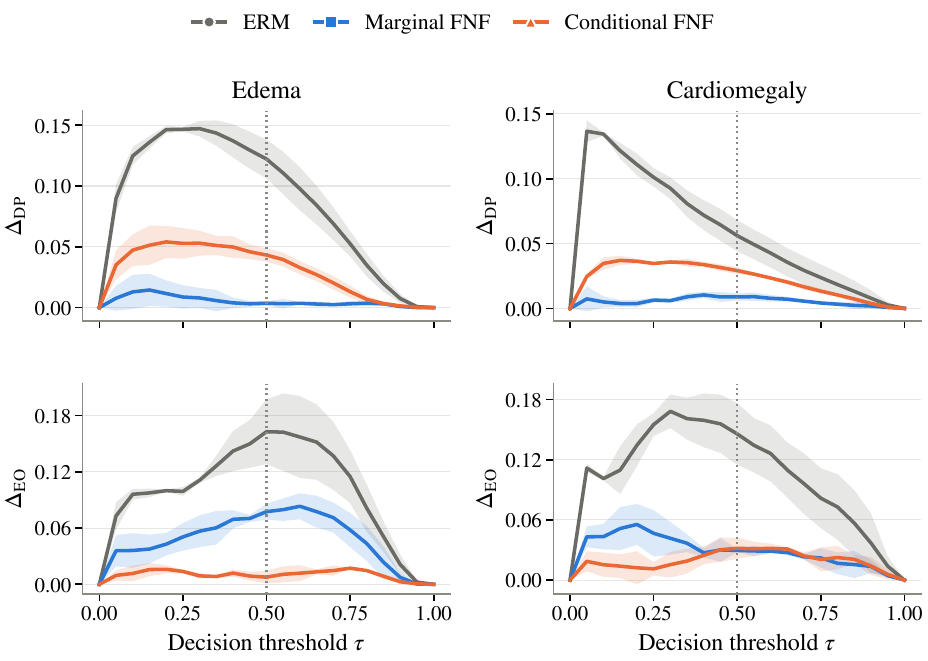}
    \caption{\textbf{Fairness gaps dependency on threshold.} $\mathrm{ERM}$ and FNF at $\gamma=.95$, evaluated on the same test cohorts. Bands show training seed std; the vertical dotted line marks the default $.5$ threshold.}
    \label{fig:cx-threshold}
  \end{minipage}
\end{figure*}

\begin{table}[!htbp]
\centering
\small
\setlength{\tabcolsep}{3pt}
\caption{Sex-calibration at $\gamma=.95$ (mean $\pm$ std reported). $\alpha_a$ is fitted separately with slope fixed at one. Negative intercepts and negative mean residuals $\implies$ overprediction. Ideal slope is 1; ideal intercept and mean residual are 0.}
\label{tab:cx-calibration}
\begin{tabular}{@{}llcrrrr@{}}
\toprule
Task & Method & $A$ & ECE & Slope $b_a$ & Intercept $\alpha_a$ & $\bar y_a-\bar r_a$ \\
\midrule
Edema & ERM & 0 & $0.010\pm0.005$ & $0.978\pm0.054$ & $0.001\pm0.034$ & $-0.000\pm0.004$ \\
 &  & 1 & $0.020\pm0.009$ & $0.926\pm0.044$ & $0.093\pm0.071$ & $0.013\pm0.010$ \\
 & M .95 & 0 & $0.069\pm0.006$ & $0.875\pm0.037$ & $-0.561\pm0.057$ & $-0.069\pm0.006$ \\
 &  & 1 & $0.036\pm0.004$ & $0.858\pm0.056$ & $0.236\pm0.030$ & $0.032\pm0.003$ \\
 & C .95 & 0 & $0.051\pm0.006$ & $0.867\pm0.079$ & $-0.416\pm0.036$ & $-0.050\pm0.007$ \\
 &  & 1 & $0.029\pm0.007$ & $0.861\pm0.076$ & $0.166\pm0.043$ & $0.023\pm0.005$ \\
Cardiomegaly & ERM & 0 & $0.015\pm0.004$ & $0.923\pm0.083$ & $-0.136\pm0.142$ & $-0.009\pm0.010$ \\
 &  & 1 & $0.018\pm0.008$ & $0.870\pm0.102$ & $0.052\pm0.111$ & $0.005\pm0.009$ \\
 & M .95 & 0 & $0.050\pm0.002$ & $0.851\pm0.030$ & $-0.718\pm0.031$ & $-0.050\pm0.002$ \\
 &  & 1 & $0.023\pm0.001$ & $0.864\pm0.020$ & $0.244\pm0.025$ & $0.021\pm0.002$ \\
 & C .95 & 0 & $0.035\pm0.001$ & $0.857\pm0.036$ & $-0.534\pm0.008$ & $-0.035\pm0.001$ \\
 &  & 1 & $0.024\pm0.003$ & $0.818\pm0.018$ & $0.204\pm0.016$ & $0.017\pm0.001$ \\
\bottomrule
\end{tabular}
\end{table}

\clearpage

\section{Appendix 3: Fairness Metrics Do Not Establish Invariance}

\paragraph{Failure to achieve demographic erasure using popular methods.}
We initially used more popular methods, specifically DANN~\citep{ganin2016domain} and class-conditional DANN~\cite{long2018conditional}. Neither achieved its objective, i.e., to erase demographic information in our CheXpert experiments. Across eight adversary strengths and five seeds, probes were still able to recover demographic information from learned representations. Fairness gaps can sometimes decrease (and even be zero) because diagnostic performance also decreases as a function of demographic erasure. We leave it to future work to investigate why certain methods fail to erase demographic signals even with strong constraints while others succeed, and why it is \textit{easy} on some datasets and not on others. This also shows that improved fairness metrics alone do not establish invariance. Our broader exploratory comparison revealed substantial variation across methods and hyperparameters: some reduced demographic decodability while retaining useful prediction, whereas others retained demographic information or substantially reduced predictive performance. Moreover, nonlinear probes recovered information that linear probes missed, highlighting the importance of evaluating erasure with multiple probe families (also discussed in Appendix 2). 

Beyond DANN and CDANN, we also explored MMD~\citep{gretton2006kernel} and HSIC~\citep{perez2017fair} penalties during training and their alternative variants, such as adversarial ensemble variational bottlenecks with MMD~\citep{gretton2006kernel}, learning Fair Representations (LFR)~\citep{zemel2013learning}, Gaussian moment transport, LEACE~\citep{belrose2023leace}, LEOPARD~\citep{ravfogel2022adversarial}, Obliviate~\citep{shakibania2026obliviate}, random Fourier features~\citep{rahimi2007random} followed by LEACE. Across 51 configurations in the single-seed Edema study, LFR, FNF, Gaussian transport, LEACE, LEOPARD, and RFF–LEACE did not meet our expectation to reduce demographic decodability. As the probe-comparison plot illustrates in Fig~\ref{fig:erasure-probe-comparison}, low linear decodability sometimes hides nonlinear leakage. Some methods like MMD and HSIC did achieve near-chance decodability but with reduced task utility (classification AUROC). Similarly, Obliviator also costs predictive performance; variational bottlenecks and gender-aware channels retained decodable information. An analysis is detailed in Table~\ref{tab:chexpert-erasure-exploration}. Note that this analysis does not reflect a limitation or general failure of the methods. 

\begin{table*}[t]
\centering
\caption{\textbf{Evaluation across pathologies and objectives.} Task AUROC, marginal and class-conditional sex probe AUROCs, demographic parity ($\mathrm{DP}$) gap, and equalized odds ($\mathrm{EO}$) gap. Results are reported as mean $\pm$ std.}
\label{tab:pathology_results}
\noindent\resizebox{\textwidth}{!}{%
\begin{tabular}{@{}llcccccc@{}}
\toprule
& & & \multicolumn{3}{c}{\textbf{Sex Probe AUROC} $\downarrow$} & \multicolumn{2}{c}{\textbf{Fairness Gaps} $\downarrow$} \\
\cmidrule(lr){4-6} \cmidrule(lr){7-8}
\textbf{Pathology} & \textbf{Method} & \textbf{Task AUROC} $\uparrow$ & \textbf{Marginal} & $\boldsymbol{Y=0}$ & $\boldsymbol{Y=1}$ & \textbf{DP Gap} & \textbf{EO Gap} \\
\midrule
\multirow{2}{*}{Edema} 
  & DANN  & 0.842 $\pm$ 0.003 & 0.950 $\pm$ 0.005 & 0.950 $\pm$ 0.006 & 0.938 $\pm$ 0.008 & 0.069 $\pm$ 0.023 & 0.033 $\pm$ 0.025 \\
  & CDANN & 0.845 $\pm$ 0.003 & 0.952 $\pm$ 0.006 & 0.955 $\pm$ 0.007 & 0.940 $\pm$ 0.012 & 0.111 $\pm$ 0.015 & 0.074 $\pm$ 0.025 \\
\midrule
\multirow{2}{*}{Cardiomegaly} 
  & DANN  & 0.842 $\pm$ 0.004 & 0.947 $\pm$ 0.006 & 0.951 $\pm$ 0.006 & 0.928 $\pm$ 0.018 & 0.054 $\pm$ 0.031 & 0.037 $\pm$ 0.015 \\
  & CDANN & 0.843 $\pm$ 0.006 & 0.944 $\pm$ 0.030 & 0.957 $\pm$ 0.008 & 0.924 $\pm$ 0.027 & 0.105 $\pm$ 0.028 & 0.078 $\pm$ 0.042 \\
\bottomrule
\end{tabular}%
}
\end{table*}

\begin{figure*}[p]
    \centering
    \begin{minipage}[b]{\textwidth}
        \centering
        \includegraphics[
            width=\linewidth,
            height=0.20\textheight,
            keepaspectratio
        ]{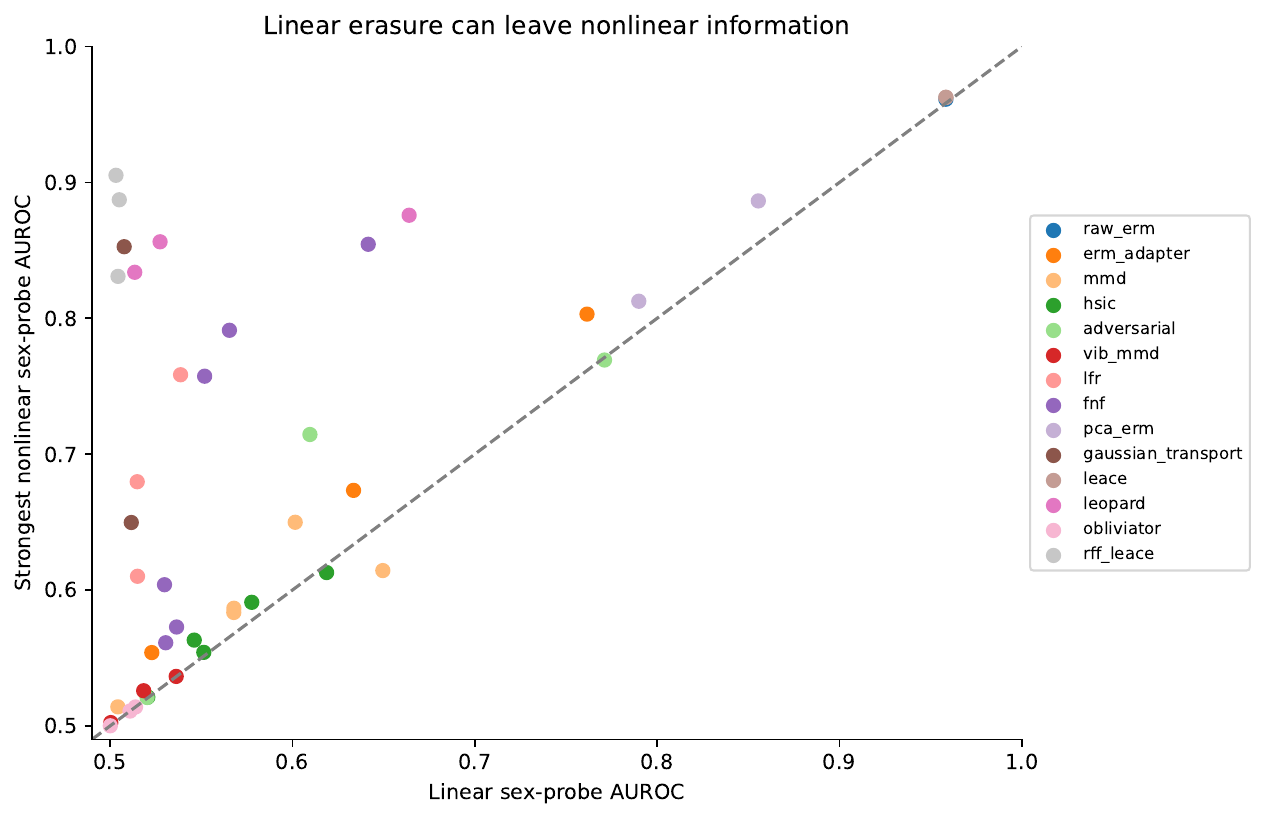}
        \par\vspace{2pt}
        {\small \textbf{(a)} Nonlinear probes reveal demographic information missed by linear probes.}
        \label{fig:erasure-probe-comparison}
    \end{minipage}
    \vspace{8pt}
    \begin{minipage}[b]{\textwidth}
        \centering
        \includegraphics[
            width=\linewidth,
            height=0.20\textheight,
            keepaspectratio
        ]{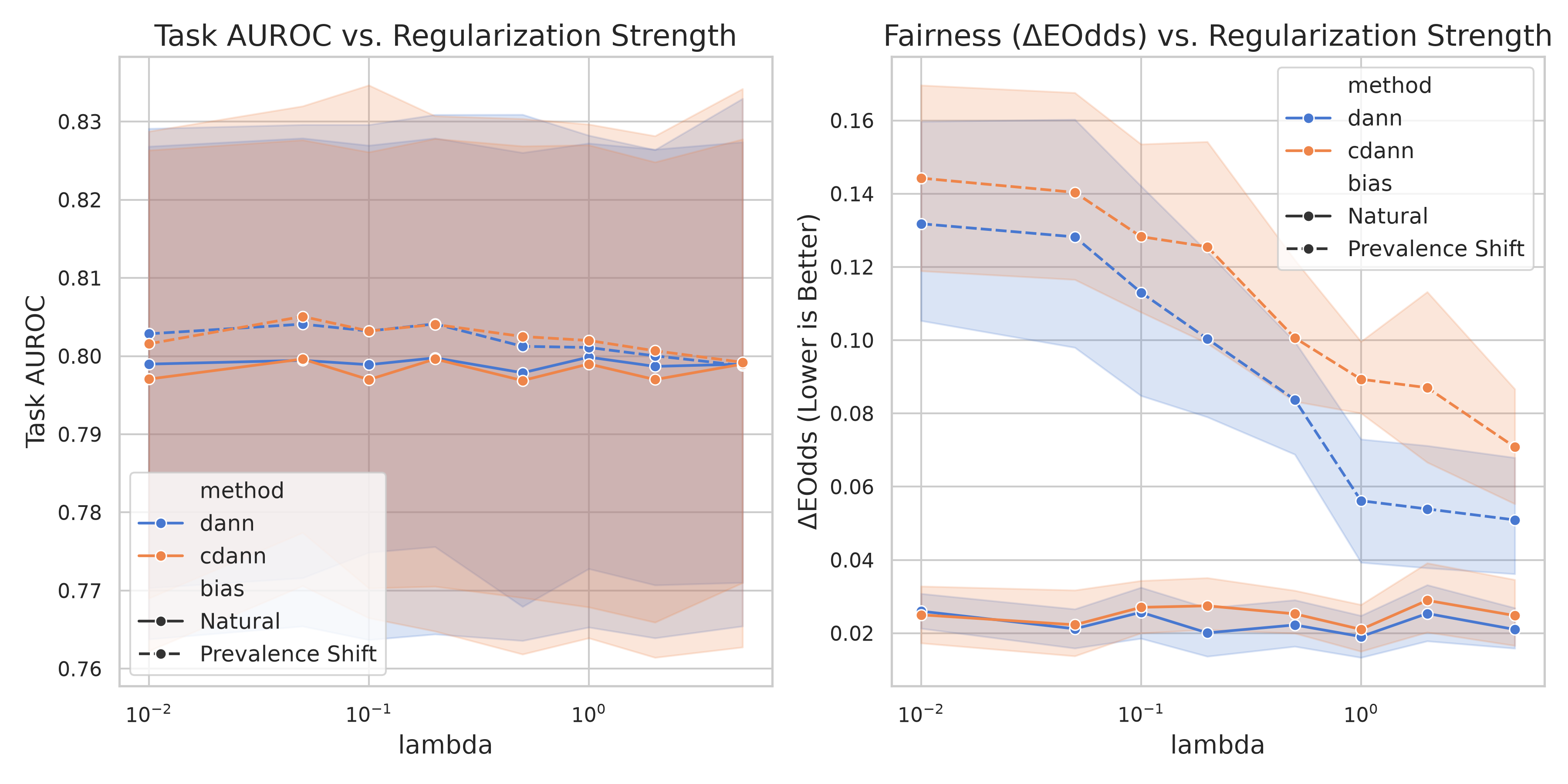}
        \par\vspace{2pt}
        {\small \textbf{(b)} Diagnostic AUROC and $\mathrm{\Delta_{EO}}$ across adversarial strengths.}
        \label{fig:erasure-lambda-sweep}
    \end{minipage}
    \vspace{8pt}
    \begin{minipage}[b]{\textwidth}
        \centering
        \includegraphics[
            width=\linewidth,
            height=0.20\textheight,
            keepaspectratio
        ]{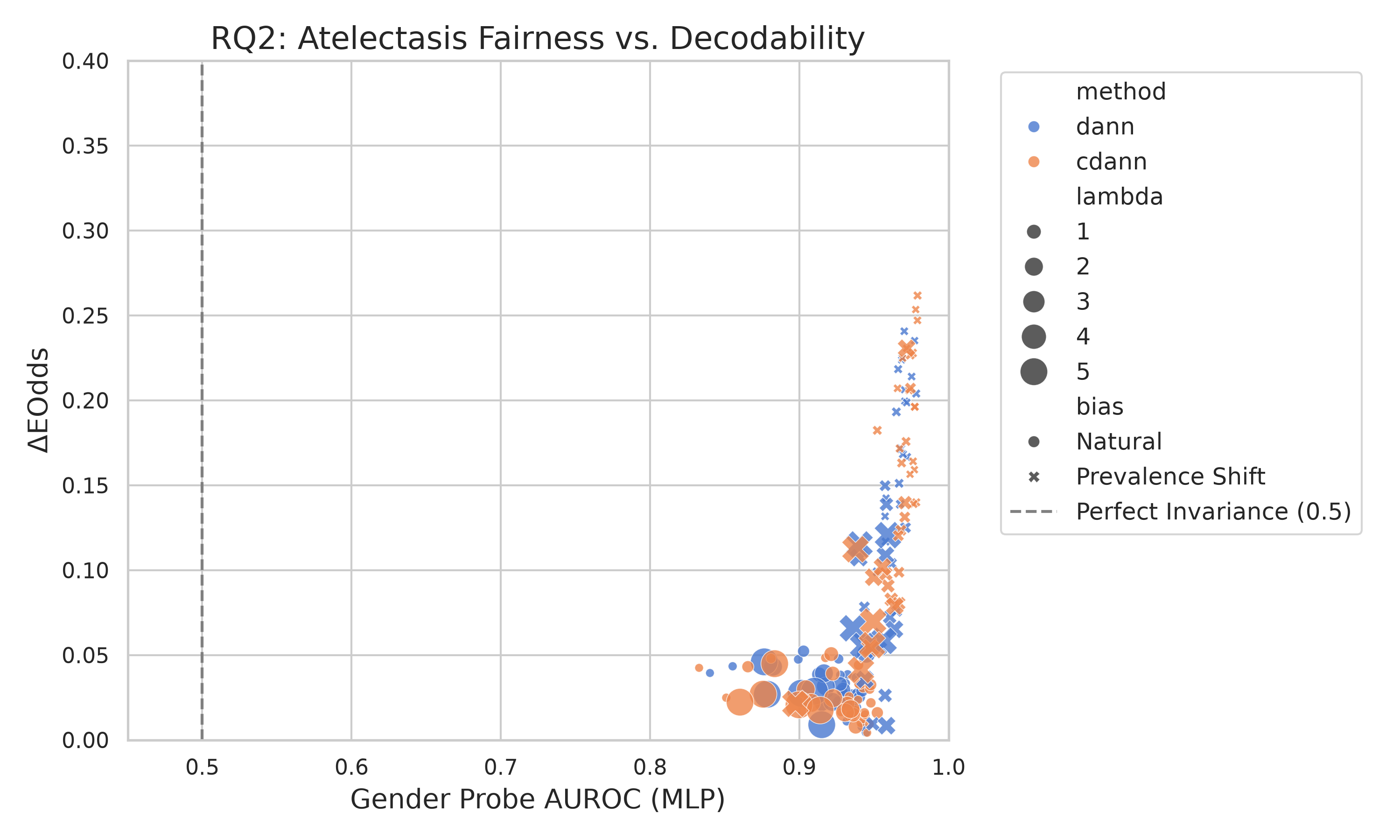}
        \par\vspace{2pt}
        {\small \textbf{(c)} Atelectasis: small fairness gaps coexist with strong demographic decodability.}
        \label{fig:erasure-atelectasis}
    \end{minipage}
    \vspace{6pt}
    \caption{\textbf{Evaluating demographic erasure on CheXpert.} 
        \textbf{(a)} Comparison of linear and nonlinear probes in the single-seed exploratory Edema study. 
        \textbf{(b--c)} Summary of DANN and CDANN experiments across adversarial strengths and seeds. 
        Chance probe AUROC does not establish independence, and marginal probes alone do not assess conditional invariance.}
    \label{fig:demographic-erasure-diagnostics}
\end{figure*}

\begin{table*}[t]
\centering
\small
\setlength{\tabcolsep}{5pt}
\renewcommand{\arraystretch}{1.2}
\begin{tabularx}{\textwidth}{@{} >{\raggedright\arraybackslash}p{2.6cm} c c c c >{\raggedright\arraybackslash}X @{}}
\toprule
\textbf{Method} & \textbf{Configs.} & \textbf{Task AUROC} $\uparrow$ & \textbf{Probe AUROC} $\downarrow$ & \textbf{Pass} & \textbf{Main Finding / Limitation} \\
\midrule
\multicolumn{6}{@{}l}{\textit{\textbf{Baselines and Controls}}} \\
ERM & 1 & 0.840 & 0.961 & 0/1 & Strong demographic leakage. \\
ERM adapters & 3 & 0.833 & 0.554--0.803 & 0/3 & Compression reduced leakage but did not meet the criterion. \\
PCA + ERM & 2 & 0.838 & 0.813--0.886 & 0/2 & Substantial residual leakage. \\
Constant output & 1 & 0.500 & 0.502 & 1/1 & Trivial removal through complete task collapse. \\
\midrule
\multicolumn{6}{@{}l}{\textit{\textbf{Learned Representations}}} \\
MMD & 5 & 0.802--0.819 & 0.514--0.650 & 1/5 & The 2-D setting passed development and test auditing. \\
HSIC & 5 & 0.774--0.818 & 0.521--0.619 & 1/5 & The 2-D setting passed, with reduced task utility. \\
Adversarial ensemble & 3 & 0.812--0.828 & 0.521--0.771 & 1/3 & Strongest setting passed test but missed development criterion. \\
Variational bottleneck + MMD & 3 & 0.804--0.806 & 0.502--0.536 & 3/3 & Single-sample passes; mean and repeated-release audits failed. \\
LFR & 3 & 0.834--0.836 & 0.610--0.758 & 0/3 & Retained demographic information. \\
NF & 6 & 0.771--0.838 & 0.561--0.854 & 0/6 & Reduced leakage but did not meet strict criterion. \\
\midrule
\multicolumn{6}{@{}l}{\textit{\textbf{Post-Hoc Representation Transformations}}} \\
Gaussian moment transport & 2 & 0.832--0.839 & 0.650--0.853 & 0/2 & Moment matching left demographic information recoverable. \\
LEACE & 2 & 0.836--0.838 & 0.962--0.963 & 0/2 & Strong residual leakage in the tested implementation. \\
LEOPARD & 3 & 0.829--0.834 & 0.834--0.876 & 0/3 & Nonlinear probes recovered substantial information. \\
RFF + LEACE & 3 & 0.818--0.836 & 0.831--0.905 & 0/3 & Near-chance linear probes concealed nonlinear leakage. \\
Obliviator & 3 & 0.749--0.759 & 0.500--0.514 & 3/3 & Only one setting met AUROC-loss budget; AUPRC fell substantially. \\
\midrule
\multicolumn{6}{@{}l}{\textit{\textbf{Score-Only Mechanisms}}} \\
Sex-blind score channel & 3 & 0.825--0.829 & 0.510--0.512 & 3/3 & Also passed repeated-release auditing; backbone unchanged. \\
Sex-aware score channel & 3 & 0.836--0.837 & 0.513--0.517 & 3/3 & Repeated releases exposed demographic information. \\
\bottomrule
\end{tabularx}
\caption{\textbf{Exploratory demographic-erasure experiments on CheXpert Edema} at natural prevalence (one training seed; 51 configurations). Results do not reflect general limitations of the evaluated methods.}
\label{tab:chexpert-erasure-exploration}
\end{table*}

\subsection{Additional Details: Clarifications}

\paragraph{Regarding using sensitive attributes $A$.} In Section.~\ref{sec:problem_settings}. Acquisition variables such as scanner, site, or protocol are nuisance factors that can affect $Z$, which is unwanted, and are handled by domain-invariant methods. We use $A$ to denote demographic attributes in this work that apply to people; usage and interpretation may differ if we consider nuisance variables.

\end{document}